\RequirePackage{fix-cm}
\documentclass{article}

\PassOptionsToPackage{sort}{natbib}
\usepackage{iclr2027_conference,times}

\usepackage[utf8]{inputenc}
\usepackage[T1]{fontenc}
\usepackage[american]{babel}
\usepackage{microtype}
\usepackage{graphicx}
\usepackage{subcaption}
\usepackage{booktabs}
\usepackage{array}
\usepackage{xcolor}
\definecolor{mydarkblue}{RGB}{0,0,139}
\usepackage{wrapfig}
\usepackage{float}
\usepackage{caption}
\usepackage{siunitx}
\usepackage{hyperref}
\usepackage{xurl}
\usepackage{amsmath}
\usepackage{amsfonts}
\usepackage{amssymb}
\usepackage{mathtools}
\usepackage{amsthm}
\usepackage[ruled,vlined]{algorithm2e}
\SetKwInput{KwInit}{Init}

\usepackage[noend]{algpseudocode}
\usepackage{titletoc}
\usepackage[capitalize,noabbrev]{cleveref}

\hypersetup{
  colorlinks=true,
  linkcolor=mydarkblue,
  citecolor=mydarkblue,
  urlcolor=mydarkblue
}

\crefalias{equation}{equation}
\crefname{section}{\S}{\S\S}
\Crefname{section}{\S}{\S\S}
\crefname{equation}{Eq.}{Eqs.}
\Crefname{equation}{Equation}{Equations}

\theoremstyle{plain}
\newtheorem{theorem}{Theorem}[section]
\newtheorem{proposition}[theorem]{Proposition}

\theoremstyle{definition}

\theoremstyle{remark}

\newcolumntype{L}[1]{>{\raggedright\arraybackslash}p{#1}}

\title{Scalable Diffusion SBI for Compositional\\
Inference under Simulator Misspecification}

\author{%
  Vincent D. Zaballa \& Elliot E. Hui\\
  University of California, Irvine\\
  \texttt{vzaballa@uci.edu}
}
\iclrfinalcopy
\begin{document}
\maketitle
\lhead{Preprint}

\begin{abstract}
  Simulation-based inference is challenging when many heterogeneous observations must be composed, hierarchical latent structure must be preserved, and the simulator is misspecified relative to observed data. We develop sampling and fine-tuning methods for diffusion-based inference in design-conditional settings, where the same simulator is queried across different experimental conditions $\xi$. We extend compositional score-based inference with a continuous-time diffusion coefficient that accounts for the number of observations, avoiding Jacobian and auxiliary-covariance corrections. We introduce Hierarchical Blockwise Diffusion Sampling (HBDS), which infers shared parameters and group-specific latent states using a single pretrained model, with the hierarchy specified only at sampling time. Together, these methods support variable observation sets and groupings without retraining. To address misspecification, we introduce path-regularized fine-tuning that adapts the learned likelihood to observations and transfers corrections to posterior inference. Using Girsanov's theorem, we quantify path divergence between pretrained and fine-tuned models across experimental designs and interpret it alongside predictive errors to distinguish candidate misspecification correction from unnecessary adaptation. We evaluate compositional sampling on exact-score Gaussian and Simple Likelihood, Complex Posterior benchmarks, HBDS with analytic and learned scores on a controlled hierarchical model, and fine-tuning and localization on a separate analytic model with known design-dependent discrepancy. Finally, we apply the framework to 940 measurements across four cell lines in a mechanistic Bone Morphogenetic Protein signaling model, where fine-tuning improves posterior-predictive accuracy relative to the pretrained model and shifts posterior marginals toward the least-squares reference while retaining spread.
\end{abstract}

\section{Introduction}
\label{sec:introduction}

Mechanistic simulators compress scientific knowledge into predictive models whose latent parameters carry physical or biological meaning. Simulation-based inference (SBI) \citep{Cranmer2020} can amortize Bayesian inference in these models even when their likelihoods cannot be evaluated. In reality, however, inference rarely consists of conditioning on a single observation from a perfectly specified simulator. Data may contain many heterogeneous observations, shared and group-specific latent variables, and systematic discrepancies between the simulator and the observed system.

These challenges have motivated several complementary directions in SBI. Compositional methods aggregate multiple observations in a score-based model without retraining a posterior estimator, beginning with Factorized Neural Posterior Score Estimation (F-NPSE) \citep{geffner23a} and followed by covariance-aware diffusion samplers for tall-data SBI \citep{linhart2026tall}. Meanwhile, recent hierarchical approaches explicitly model shared and local latent structure through hierarchical estimators or training constructions \citep{heinrich2024hierarchical, arruda2026compositional, charles2026tokenised}. Separately, simulator misspecification has motivated a growing literature on detection and mitigation \citep{Cannon2022, boelts2026misspecification}, including explicit discrepancy models \citep{ward2022robust}, misspecification-aware representations and summary statistics \citep{schmitt2022detectingmodelmisspec, huang2023wentwrongclosingsimtoreal}, and distributional alignment or calibration methods \citep{wehenkel2025addressing}. These lines of work address composition, hierarchy, and misspecification, but typically as distinct problems.


We study how these requirements can be combined within a single diffusion-based SBI framework \citep{sharrock2022sequential}. Our central idea is to treat a pretrained amortized conditional model as a reusable inference engine: observation sets can be composed, hierarchical structure introduced at sampling time, and the simulator-trained surrogate adapted to observed data while retaining posterior inference in the original mechanistic parameterization.

Achieving the last of these requires likelihood and posterior score representations to remain coupled as the simulator-trained surrogate is adapted to observed data. We use the Simformer of \citet{gloeckler24a} as the diffusion backbone and introduce a new mechanism for transfer learning \citep{Pan2010ASO} through shared embeddings, allowing likelihood fine-tuning to update representations used for posterior inference (\cref{sec:rep_learn}). The framework is summarized in \cref{fig:bmp_simformer}. Building on this coupled surrogate, our contributions are:

\begin{figure*}[t]
\centering
\includegraphics[width=\textwidth]{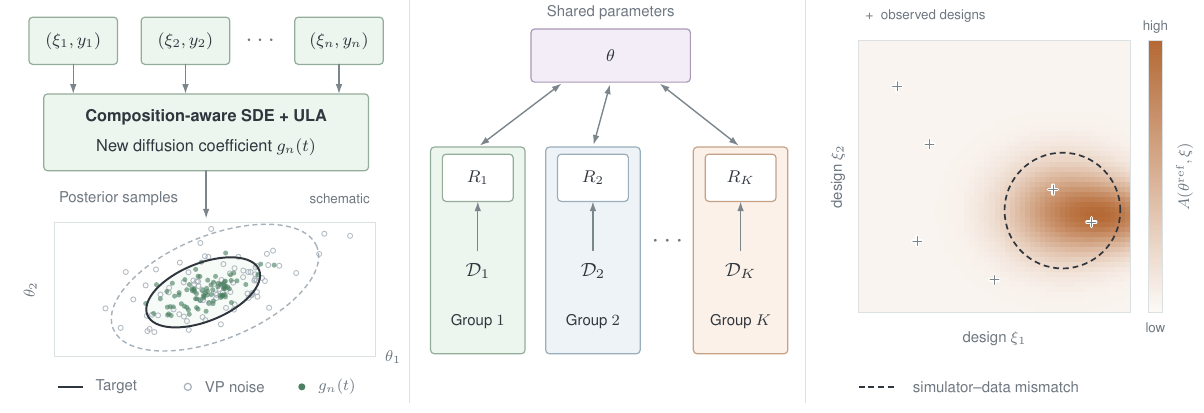}
\vskip -6pt
\caption{Schematic overview of the framework. (\emph{Left}) Compositional sampling combines observations using the derived diffusion coefficient $g_n(t)$ and optional Langevin correctors, without retraining. (\emph{Middle}) HBDS alternates updates of shared parameters $\theta$ and group-specific latent states $R_k$ conditioned on each group's data $\mathcal D_k$. (\emph{Right}) Path divergence $\mathcal{A}(\theta^{\mathrm{ref}},\xi)$ measures the magnitude of likelihood adaptation across experimental designs. Comparing this field with predictive improvement helps distinguish candidate misspecification correction from over-adaptation.}
\label{fig:bmp_simformer}
\vskip -12pt
\end{figure*}

{
\setlength{\leftmargini}{1.2em}

\begin{itemize}
    \setlength{\itemsep}{2pt}
    \setlength{\parskip}{0pt}
    \setlength{\parsep}{0pt}
    \setlength{\topsep}{4pt}

    \item \textit{Composition-aware diffusion sampling.} We derive an $n$-dependent continuous-time diffusion coefficient for compositional SBI, extending F-NPSE to efficient SDE sampling while avoiding the Jacobian and auxiliary-covariance costs of JAC and GAUSS.

    \item \textit{Hierarchical inference at sampling time.} We introduce Hierarchical Blockwise Diffusion Sampling (HBDS), reusing a single joint diffusion model to infer shared and group-specific latents with hierarchy introduced only at sampling time.

    \item \textit{Path-space adaptation and localization.} We fine-tune the surrogate likelihood to observations while transferring the correction to posterior inference, and use Girsanov's theorem to isolate \textit{where} model misspecification may occur.

    \item \textit{Large-scale systems biology inference.} We apply the full framework to a Bone Morphogenetic Protein (BMP) signaling simulator with 60 shared parameters, five receptor quantities per cell line, and 940 observations, showing how composition, hierarchy, and misspecification-aware adaptation jointly make SBI practical on a nontrivial scientific model.

\end{itemize}
}

\section{Background}
\label{sec:background}

\subsection{Design-dependent simulation-based inference}
\label{sec:SBI_background}
SBI considers a simulator from which observations can be sampled conditional on latent parameters and, when relevant, experimental conditions.  In that setting we write \(y \sim p_{\text{eval}}(y \mid \theta)\), where \(p_{\text{eval}}(y,\theta)=p_{\text{eval}}(y \mid \theta)\,p(\theta)\) is an \emph{expert simulator}, distinct from the unknown true data-generating process \(p^{\ast}(y,\theta)\). We adopt the $p_{\text{eval}}(y \mid \theta)$ notation of \cite{smith2025rethinking} and use the \emph{design-dependent} setting where an experimental design \(\xi\) alters the joint model and resulting posterior, yielding \(p_{\text{eval}}(y,\theta \mid \xi)\) and related \(p^*(y,\theta \mid \xi)\). We represent a simulator as a function
\[
f:\;\Xi\times\Theta\longrightarrow\mathcal{Y},\quad
\Xi\subset\mathbb{R}^{D},\;
\Theta\subset\mathbb{R}^{M},\;
\mathcal{Y}\subset\mathbb{R}^{B},
\]
where \(\Xi\), \(\Theta\), and \(\mathcal{Y}\) denote the spaces of experimental designs, latent model parameters, and observations, respectively. The simulator generates samples from the prior predictive distribution, $\theta \sim p(\theta)$ and $y \sim p_{\mathrm{eval}}(y\mid\theta,\xi)$, which provide training data for amortized density, ratio, or score estimators targeting quantities such as the posterior $p(\theta\mid y,\xi)$ or surrogate likelihood $p(y\mid\theta,\xi)$.

\subsection{Mask-conditioned joint diffusion models}
\label{sec:pretraining}

Diffusion-based SBI often trains an estimator for a particular inference task, such as posterior or likelihood estimation \citep{sharrock2022sequential}. Mask-conditioned joint diffusion models instead support different conditioning patterns within a single network \citep{gloeckler24a}. Variables such as $y$, $\xi$, and $\theta$ are represented by tokens encoding their identity, value, and whether they are latent or conditioned. For $d$ modeled variables, a condition mask $M_C\in\{0,1\}^d$ identifies those held fixed, with $M_C=(0,\dots,0)$ giving an unconditional model. An attention mask $M_E$ incorporates the expert model's graphical structure into the transformer \citep{weilbach2023diffusiongraph}.

Writing the tokenized variables collectively as $y_0$, a prescribed forward perturbation kernel $p_t(y_t\mid y_0)$ provides noisy versions $y_t$. The masked input $y_t^{M_C}$ retains clean values for conditioned variables and noisy values for latent variables. The network learns the conditional score, the gradient of the log density with respect to latent variables, through masked denoising score matching \citep{hyvarinen05a,vincent2011}. We define
\begin{align*}
\ell(\phi, M_C, t, y_0, y_t) = \zeta(t)(1-M_C)\Bigl(
  s_{\phi}^{M_E}\bigl(y_t^{M_C},t \bigr)
  - \nabla_{y_t} \log  p_t \bigl(y_t | y_0\bigr)
\Bigr),
\end{align*}
where $\zeta(t)$ weights the noise levels and $(1-M_C)$ selects latent coordinates. We minimize the expected squared residual,
\begin{equation*}
\mathcal{L}(\phi) = \mathbb{E}_{M_C,t,y_0,y_t}\!
       \Bigl[
           \bigl\|\ell\!\bigl(\phi,M_C,t,y_0,y_t\bigr)\bigr\|_2^{2}
       \Bigr], 
\end{equation*}
yielding the pretrained (PT) model $p_{\phi_0}(y,\theta\mid\xi)$, where $\phi_0$ denotes the PT network parameters. The choice of condition masks and their frequency during training influences conditional sampling performance (\cref{sec:train_deets}).

\subsection{Score-based diffusion sampling}
\label{sec:score-based-diffusion}

The learned conditional score defines a sampler that reverses the noising used during pretraining. We use a variance-preserving SDE (VP-SDE), the continuous-time limit of the forward Gaussian noising chain in denoising diffusion probabilistic models (DDPMs) \citep{ho2020denoisingdiffusionprobabilisticmodels,song2021scorebased}. At fixed conditioning, let $y_t$ denote the variables being generated and suppress the conditioning arguments. The forward process transforms $p_0(y_0)$ toward a tractable Gaussian $p_T(y_T)$, and the learned score $s_\phi(y_t,t)\approx\nabla_{y_t}\log p_t(y_t)$ gives the reverse dynamics,
\begin{equation*}
dy_t = \begin{cases}
f(y_t, t)dt + g(t)dw, & \text{(forward)} \\
[f(y_t, t) - g(t)^2 s_\phi(y_t, t)] dt + g(t)d\overset{\scriptscriptstyle\leftarrow}{w}, & \text{(reverse)},
\end{cases}
\end{equation*}
where $w$ and $\overset{\scriptscriptstyle\leftarrow}{w}$ are forward- and reverse-time Wiener processes, and $f$ and $g$ are the drift and diffusion coefficients \citep{anderson1982reverse,song2021scorebased}. For the VP-SDE, $f(y,t)=-\tfrac12\beta(t)y$ and $g(t)=\sqrt{\beta(t)}$ for a nonnegative noise rate $\beta(t)$. Sampling integrates the reverse SDE from $y_T\sim p_T$ at $t=T$ to $t=0$, holding conditioned variables fixed.

\section{Methods}
\label{sec:methods}

\begin{figure}[!t]
\centering
\includegraphics[width=\textwidth]{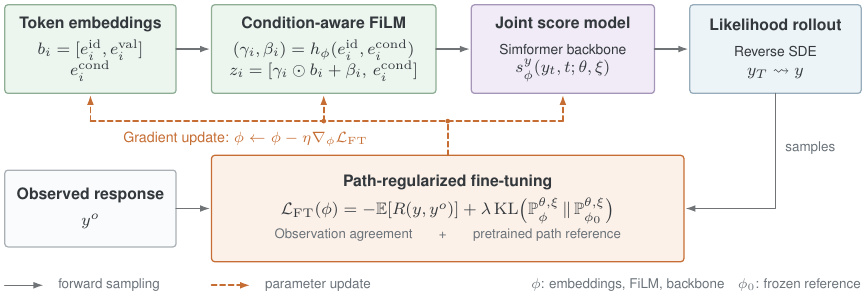}
\caption{Likelihood fine-tuning of the joint diffusion surrogate. At fixed reference parameters and design, generated responses are evaluated against observations with a path-space KL penalty relative to the PT model. Gradients through sampling update the token embeddings, FiLM layers, and score network, which also support posterior queries.}
\label{fig:likelihood-fine-tuning}
\vspace{-12pt}
\end{figure}

\subsection{Fine-tuning to address model misspecification}
\label{sec:fine-tune}

The PT joint surrogate provides a conditional likelihood $p_{\phi_0}(y\mid\theta,\xi) \approx p_{\mathrm{eval}}(y\mid\theta,\xi)$, which inherits simulator misspecification relative to $p^*(y\mid\theta,\xi)$. Following \citet{Cannon2022}, we conceptualize misspecification via an unknown transformation that maps simulator outputs to observations. In our setting this is naturally a map between \emph{different} sample spaces, since simulator outputs must be post-processed (e.g.\ normalized) to match the observation representation:
\[
    y \sim p_{\text{eval}}(\cdot\mid \theta,\xi),\qquad
    y_o = T(y),\qquad
    T:\mathcal{Y}_{\text{sim}}\to\mathcal{Y}_{\text{obs}}.
\]
Under this representation, \(p^\ast(\cdot\mid \theta,\xi)\) is the pushforward \(p^\ast(\cdot\mid \theta,\xi)=T_{\#}p_{\text{eval}}(\cdot\mid \theta,\xi)\). Rather than specifying \(T\) explicitly, we seek to learn an \emph{implicit} misspecification correction that absorbs this transformation into the joint model's parameters. We draw on diffusion fine-tuning to address simulator misspecification, adapting the surrogate likelihood to observed data. Specifically, we use the \textit{implicit diffusion} method of \citet{marion2025implicitdiffusionefficientoptimization}, which allows us to optimize agreement with observations by differentiating through likelihood sampling. We discuss alternative diffusion fine-tuning approaches and their suitability for SBI in \cref{sec:diffusion-fine-tuning-related-work}.

In SBI, observations $y_o$, or design--observation pairs $(y_o,\xi)$, provide a target for fine-tuning the surrogate likelihood. The learned correction must also inform the posterior over $\theta$ conditioned on those observations. We enable this transfer through token-aware Feature-wise Linear Modulation (FiLM) \citep{perez2017filmvisualreasoninggeneral}, which modulates token embeddings according to their identity and whether they are latent or conditioned. These embeddings are shared across likelihood and posterior queries in the masked transformer, allowing likelihood fine-tuning to update representations used for posterior inference; \cref{fig:likelihood-fine-tuning} illustrates this mechanism, with embedding details in \cref{sec:rep_learn}.

\paragraph{How can we fine-tune the likelihood $p_\phi(y\mid\theta,\xi)$ to observed data when the corresponding mechanistic parameters $\theta$ are unknown?} For observed design--response pairs $(y_i^o,\xi_i)$, we condition the surrogate likelihood on a fixed, simulator-aligned reference parameter $\theta^{\mathrm{ref}}$. In supervised settings, $\theta^{\mathrm{ref}}$ may be provided by a labeled real-world dataset $\mathcal{D}_{\mathrm{cal}}=\{(\theta^i,y_o^i)\}_{i=1}^{N_{\mathrm{cal}}}$, where each observation $y_o^i$ is paired with its ground-truth parameter $\theta^i$, as in \citet{wehenkel2025addressing}. When it is unknown, it may instead be an externally obtained best-fit estimate. In the BMP application, we use the domain-standard least-squares estimate $ \theta^{\mathrm{LSR}} \in \arg\min_{\theta} \sum_i \ell\!\left(f(\theta,\xi_i),y_i^o\right), $ which provides a parameter configuration explicitly grounded in the expert simulator. We do not interpret $\theta^{\mathrm{LSR}}$ as a true parameter vector; it serves only as a fixed reference for adapting the likelihood to observed data.

For fixed $(\theta^{\mathrm{ref}},\xi)$, we draw samples $y\sim p_\phi(y\mid\theta^{\mathrm{ref}},\xi)$ and evaluate their agreement with observations through a reward $R(y;y^o)$. The diffusion sampler induces an output distribution $\pi^\star(\phi;\theta^{\mathrm{ref}},\xi)$ defined implicitly by the network parameters $\phi$. We formulate fine-tuning as optimization of a functional $\mathcal{F}$ of this sampled distribution, differentiating through the sampler with the adjoint method of Implicit Diffusion \citep{marion2025implicitdiffusionefficientoptimization}. Suppressing the fixed conditioning arguments, the endpoint-regularized problem is $\min_\phi\,\mathcal{F}(\pi^\star(\phi))$, with
\begin{equation}
  \mathcal{F}(p_\phi)
  :=
  -\mathbb{E}_{y\sim p_\phi}[R(y;y^o)]
  +
  \lambda\,\mathrm{KL}\!\left(p_\phi\,\middle\|\,\pi^\star(\phi_0)\right),
\label{eq:fine-tune-objective}
\end{equation}
where $\lambda\geq0$ penalizes departure from the PT output law. In our experiments, $R$ is the negative mean squared error (MSE) between generated and observed responses. In implementation, we replace the endpoint KL by a path-space KL penalty, giving the stronger regularization described in \cref{sec:path-divergence,sec:soc-connection}. Alternative fine-tuning frameworks are discussed in \cref{sec:diffusion-fine-tuning-related-work}. Surprisingly, anchoring likelihood fine-tuning at a fixed $\theta^{\mathrm{ref}}$ does not reduce posterior inference to a point estimate --- in the BMP example, the fine-tuned (FT) posterior marginals generally shift toward the reference relative to pretraining while retaining substantial spread (\cref{fig:bmp-posterior-marginals}).

\subsection{Localizing simulator misspecification with path-space divergence}
\label{sec:path-divergence}

For fixed $(\theta,\xi)$, the reverse-SDE sampler underlying $\pi^\star(\phi;\theta,\xi)$ induces a path measure $\mathbb{P}_{\phi}^{\theta,\xi}$ over trajectories $Y=(Y_t)_{t\in[0,T]}$, whose generated endpoint distribution is $\pi^\star(\phi;\theta,\xi)$. Fine-tuning from $\phi_0$ to $\phi$ therefore changes not only the endpoint distribution of generated responses but the probability law of the complete sampling trajectory, giving PT and FT path measures $\mathbb{P}_{\phi_0}^{\theta,\xi}$ and $\mathbb{P}_{\phi}^{\theta,\xi}$. We quantify this change using path-space KL, which we use in place of the endpoint penalty in \cref{eq:fine-tune-objective}. Since the generated response is a projection of the complete diffusion trajectory, the data-processing inequality implies that path-space KL upper-bounds the corresponding endpoint divergence and can therefore capture adaptation not visible from the final samples alone. Under the regularity and integrability conditions stated in \cref{sec:girsanov_background}, Girsanov's theorem gives

\begin{equation}
\begin{aligned}
\mathrm{KL}\!\left(
\mathbb{P}_{\phi}^{\theta,\xi}
\,\middle\|\,
\mathbb{P}_{\phi_0}^{\theta,\xi}
\right)
=
\frac{1}{2}\,
\mathbb{E}_{\mathbb{P}_{\phi}^{\theta,\xi}}
\!\left[
\int_0^T g(t)^2
\left\|
s_{\phi}^{y}(Y_t,t;\theta,\xi)
-
s_{\phi_0}^{y}(Y_t,t;\theta,\xi)
\right\|_2^2\,dt
\right].
\end{aligned}
\label{eq:score-control-energy}
\end{equation}

We denote this discrepancy by $\mathcal{A}(\theta,\xi)$. Evaluating $\mathcal{A}$ at the same fixed reference parameter $\theta^{\mathrm{ref}}$ used during fine-tuning yields a non-negative scalar field over design space, assigning each design a measure of the trajectory-level displacement between the PT and FT reverse-SDEs.

\paragraph{Stochastic optimal control (SOC) interpretation.} Equation~\eqref{eq:score-control-energy} admits a SOC interpretation. In our reverse-time convention, a control field $u(x,t)$ perturbs a reference drift through $dX_t^u=[b_0(X_t^u,t)+\sigma(t)u(X_t^u,t)]dt+\sigma(t)d\overset{\scriptscriptstyle\leftarrow}{w}_t$, integrated from $t=T$ to $t=0$, inducing a controlled path measure $\mathbb{P}^u$. Taking the PT reverse-SDE as the reference and $\sigma(t)=g(t)I$ gives $u_\phi=-g(t)(s_\phi^y-s_{\phi_0}^y)$. Hence, $\mathcal{A}(\theta,\xi)$ is the quadratic control energy associated with transporting the PT path law to the FT one, connecting likelihood adaptation to SOC and KL-constrained path-space transport \citep{blessing2025trustregion}.


\paragraph{Fisher geometry over conditioning variables.} Beyond comparing PT and FT path laws at fixed $(\theta,\xi)$, the same path-space formulation induces a local geometry over the conditioning variables themselves. Let $z=(\theta,\xi)$ denote the continuous mechanistic and design coordinates, and fix the learned model $\phi$. Assuming the induced path laws vary smoothly with $z$, nearby path measures satisfy the absolute-continuity conditions required by Girsanov's theorem, and distinct nearby conditioning values induce distinct path laws, the family
\[
\mathcal{M}_{\phi}
=
\left\{
\mathbb{P}_{\phi}^{z}:z\in\Theta\times\Xi
\right\}
\subset
\mathcal{P}\!\left(C([0,T];\mathbb{R}^{d_y})\right)
\]
can be treated locally as a statistical manifold (\cref{fig:bmp-simformer-geometry}). Here, $C([0,T];\mathbb{R}^{d_y})$ is the space of continuous trajectories $Y:[0,T]\to\mathbb{R}^{d_y}$, and $\mathcal{P}(\cdot)$ denotes probability measures on that space. With a shared diffusion schedule and terminal law, the same Girsanov argument gives
\begin{equation}
\mathcal{I}_{\phi}^{\mathrm{path}}(z)
:=
\mathbb{E}_{Y\sim\mathbb{P}_{\phi}^{z}}
\!\left[
\int_0^T g(t)^2
J_z s_{\phi}^{y}(Y_t,t;z)^\top
J_z s_{\phi}^{y}(Y_t,t;z)\,dt
\right],
\label{eq:path-fisher-main}
\end{equation}

\begin{wrapfigure}[11]{r}{0.40\textwidth}
\vspace{-21pt}
\centering
\includegraphics[width=\linewidth]{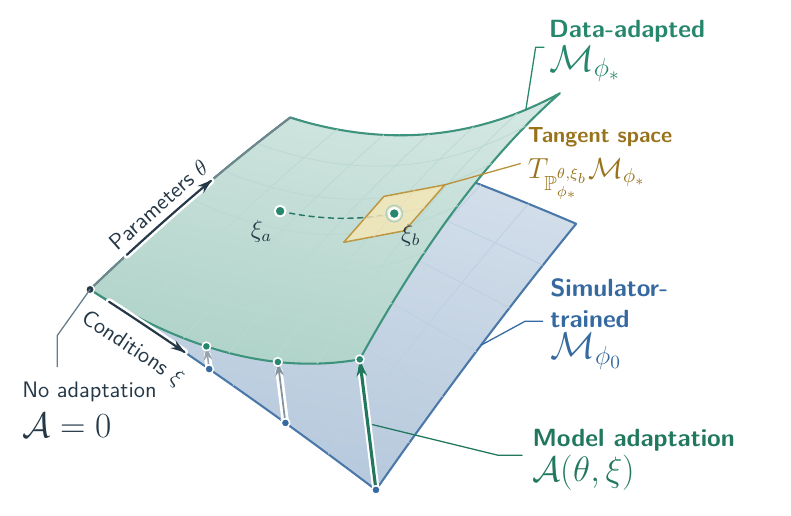}
\captionsetup{font=small,skip=3pt}
\caption{Conditional path-law geometry under fine-tuning.}
\label{fig:bmp-simformer-geometry}
\vspace{-6pt}
\end{wrapfigure}
where $J_zs_{\phi}^{y}=\partial s_{\phi}^{y}/\partial z$. Thus, $\mathcal{I}_{\phi}^{\mathrm{path}}(z)$ is the local Fisher--Rao metric induced by KL divergence on the conditional reverse-SDE path laws \citep{amari2000methods}. When the conditioning parameterization is locally non-identifiable, $\mathcal{I}_{\phi}^{\mathrm{path}}(z)$ may be singular, consistent with the classical connection between local identifiability and nonsingularity of the Fisher information matrix \citep{rothenberg1971identification}. This includes the familiar case of mechanistic parameter non-identifiability, where distinct parameter combinations $\theta$ induce indistinguishable responses \citep{raue2009structural,villaverde2019observability}. The tangent space $ T_{\mathbb{P}_{\phi}^{z}}\mathcal{M}_{\phi} $ captures first-order perturbations of the conditional path law around $\mathbb{P}_{\phi}^{z}$, with $\mathcal{I}_{\phi}^{\mathrm{path}}(z)$ defining the local metric on this space. We further discuss the parameter- and design-space structure of this geometry in \cref{sec:fisher-rao-analysis}.

\subsection{Compositional sampling}
\label{sec:comp_samp}

A simformer can represent arbitrary conditionals of a fixed joint distribution, but explicitly representing each design--observation pair in the joint makes the modeled state grow with the number of experiments and ties inference to a fixed number of observation tokens. Instead, for each observation $(y_j,\xi_j)$ we use the posterior score $ s_\phi(\theta_t,t;y_j,\xi_j) \approx \nabla_{\theta_t}\log p_t(\theta_t\mid y_j,\xi_j), $ and combine these per-observation scores at inference to approximate the posterior conditioned on $\mathcal D_n=\{(y_j,\xi_j)\}_{j=1}^n$. Since the same score model is reused for every $(y_j,\xi_j)$, the number and subset of observations can vary without changing or retraining the learned joint model. To achieve this, we first review the discrete F-NPSE construction of \citet{geffner23a}, then introduce a continuous-time extension.

\paragraph{Discrete compositional sampling.} Under conditional independence given $\theta$, the target posterior factorizes as $p(\theta\mid\mathcal D_n)\propto p(\theta)^{1-n}\prod_{j=1}^{n}p(\theta\mid y_j,\xi_j)$. F-NPSE exploits this factorization by defining intermediate distributions whose scores combine the per-observation posterior scores above with the corresponding prior-score correction. In the DDPM setting \citep{ho2020denoisingdiffusionprobabilisticmodels}, $k$ indexes diffusion steps and $\alpha_k$ denotes one-step signal retention, with signal scaling $\sqrt{\alpha_k}$ and added noise variance $1-\alpha_k$ per coordinate. The resulting normalized reverse transition can be written as
\begin{equation}
\widetilde p_k(\theta_{k-1}\mid\theta_k,\mathcal D_n)
=
\mathcal N\!\left(
\theta_{k-1};\mu_k^{\mathrm G},
\frac{1-\alpha_k}{\kappa_k^{(n)}}I
\right),
\qquad
\kappa_k^{(n)}
=
1+(n-1)(1-\alpha_k),
\label{eq:geffner-transition-main}
\end{equation}
where $\mu_k^{\mathrm G}$ combines the learned reverse scores for each $(y_j,\xi_j)$ with the prior-score correction implied by the posterior factorization above; its full form and assumptions are given in \cref{sec:rev_kernel_der}. The transition variance in Equation~\eqref{eq:geffner-transition-main} decreases with the number of observations $n$ and approaches $1/n$ as $\alpha_k\to0$. This transition-level contraction motivates a narrower cumulative noising path than the single-observation VP process, echoing the $n^{-1}$ covariance contraction of Bernstein--von Mises theory \citep[Ch.~10]{vanderVaart1998}. 

\paragraph{From discrete to continuous dynamics.} To construct a continuous-time noising process, consider a hypothetical draw $\theta_0$ from the desired posterior. Ordinary VP noising attenuates this draw by $\sqrt{\alpha(t)}$ and adds Gaussian noise with variance $1-\alpha(t)$ per coordinate, where $\alpha(t)$ denotes cumulative squared signal retention. Motivated by the discrete contraction in \cref{eq:geffner-transition-main}, we retain this mean attenuation but prescribe a smaller accumulated noise variance $V_n(t)=(1-\alpha(t))/\kappa^{(n)}(t)$, where $\kappa^{(n)}(t)=1+(n-1)(1-\alpha(t))$. This introduces observation-count dependence without estimating the posterior covariance. Keeping the ordinary VP-SDE drift $f(t)=\dot\alpha(t)/(2\alpha(t))$ \citep{song2021scorebased}, the variance evolution $\dot V_n(t)=2f(t)V_n(t)+g_n(t)^2$ then determines the diffusion coefficient for this noising process
\begin{equation}
g_n(t)^2
=
\beta(t)
\frac{
1+(n-1)\left(1-\alpha(t)\right)^2
}{
[\kappa^{(n)}(t)]^2
},
\qquad
\beta(t)=-\frac{\dot\alpha(t)}{\alpha(t)}.
\label{eq:compositional-diffusion-main}
\end{equation}
The coefficient satisfies $g_n(t)^2\leq\beta(t)$, with $g_1(t)^2=\beta(t)$ and $g_n(t)^2/\beta(t)\to1/n$ as $\alpha(t)\to0$. We call the resulting sampler \textsc{F-NPSE-SDE}. It uses $g_n(t)$ in the reverse-SDE predictor, which advances between noise levels, and in $L$ optional unadjusted Langevin corrections, which explore the current compositional marginal at fixed diffusion time. For a concentrated marginal, the score can change rapidly over short distances, so a correction step may overshoot even when its direction points toward higher density. The smaller coefficient moderates score-driven drift and noise, reducing the risk of overshoot in stiff dynamics. Our controlled exact-score Gaussian example shows that these numerical issues arise even in a simple model and that the schedule reduces both step-size sensitivity and sampling error (\cref{sec:compositional-coefficient-evaluation}). This yields an inexpensive schedule depending only on diffusion time and observation count, with its derivation in \cref{sec:forward_kernel_derivation}.

\paragraph{Practical scaling.} \textsc{F-NPSE-SDE} requires only score evaluations, avoiding the score Jacobians of JAC and auxiliary covariance estimation of GAUSS \citep{linhart2026tall}. These costs become especially important in HBDS (\cref{sec:em_diffusion}), where compositional sampling is repeated within each block update. Computational complexity and observed runtime and memory limitations on the BMP model are discussed in \cref{sec:compositional-cost-details}.

\subsection{Hierarchical blockwise diffusion sampling}
\label{sec:em_diffusion}

Na\"ive compositional sampling, which we refer to as \emph{pooled compositional sampling}, treats all experimental observations as if they share a single posterior distribution. This ignores the fact that different groups of experiments may induce distinct posteriors, leading to suboptimal inference. For example, the BMP dataset comprises $K=4$ cell lines consisting of a parent and three receptor-knockdown lines. Each line has its own receptor state $R^k$, shared across its ligand-dose experiments, while the biophysical parameters $\theta$ are shared across all four lines (Appendix~\ref{sec:BMP_model}). We use a hierarchical Bayesian formulation \citep{gelman1995bayesian} with shared parameters $\theta$ and group-specific latent states $R^{1:K}$. Given designs $\xi^{1:K}$ and observations $y^{1:K}$ across $K$ groups, we target the joint posterior $p(\theta,R^{1:K}\mid y^{1:K},\xi^{1:K})$. To approximately sample this posterior, HBDS initializes a shared parameter state and separate receptor states. At every reverse-diffusion time, each $R^k$ is updated with its own $(y^k,\xi^k)$ and current global $\theta$, then $\theta$ is updated using all receptor blocks. Repeating yields samples with one globally consistent parameter vector and $K$ distinct local latent states without training a separate hierarchical estimator or tokenizing the complete dataset. The HBDS algorithm and analytic example are in \cref{sec:hbda_details,app:em_simulator}.

\section{Experiments}
\label{sec:experiments}

We evaluate $g_n(t)$ on an exact toy Gaussian model and the Simple Likelihood, Complex Posterior (SLCP) \cite{papamakarios2019sequentialneurallikelihoodfast} benchmarks (\cref{sec:validate_comp_diffusion}), hierarchical inference on a controlled toy (\cref{sec:compare_pool_hbds}), and misspecification-aware fine-tuning (\cref{sec:imp_of_fine_tune}). The BMP application (\cref{sec:bmp-application}) combines fine-tuning with HBDS across 940 observations, assessing predictive improvement and its relationship to design-localized path divergence.


\subsection{Evaluating \textsc{F-NPSE-SDE}}
\label{sec:validate_comp_diffusion}

\begin{table*}[!b]
\vspace{-12pt}
\captionsetup{skip=2pt}
\caption{Comparison of compositional samplers on the Gaussian benchmark ($n=100$, $T=400$) and SLCP ($n=30$). Entries are mean $\pm$ standard deviation over five seeds for Gaussian and 25 test instances using one network seed for SLCP.}
\label{tab:fnpse-sde-benchmarks}
\centering
\begingroup
\setlength{\tabcolsep}{6pt}
\renewcommand{\arraystretch}{1.10}
\small
\begin{tabular*}{\textwidth}{@{\extracolsep{\fill}}lcccc@{}}
\toprule
& \multicolumn{1}{c}{\textbf{Gaussian ($T=400$, $n=100$)}} & \multicolumn{3}{c}{\textbf{SLCP ($n=30$)}} \\
\cmidrule(lr){2-2}\cmidrule(lr){3-5}
\textbf{Method} & \textbf{max-sW} $\downarrow$ & \textbf{sW} $\downarrow$ & \textbf{MMD} $\downarrow$ & \textbf{C2ST} $\rightarrow 0.5$ \\
\midrule
GAUSS & \textbf{$0.24 \pm 0.23$} & \textbf{$0.45 \pm 0.18$} & $0.05 \pm 0.16$ & $0.97 \pm 0.04$ \\
JAC & $0.26 \pm 0.22$ & $0.53 \pm 0.19$ & \textbf{$0.05 \pm 0.14$} & \textbf{$0.95 \pm 0.05$} \\
Langevin & $0.52 \pm 0.50$ & $0.81 \pm 0.28$ & $0.13 \pm 0.22$ & $0.97 \pm 0.03$ \\
\textsc{F-NPSE} & $0.29 \pm 0.24$ & $0.83 \pm 0.29$ & $0.11 \pm 0.13$ & $0.98 \pm 0.03$ \\
\midrule
\textsc{F-NPSE-SDE} ($L=1$) & $0.29 \pm 0.25$ & $0.74 \pm 0.31$ & $0.10 \pm 0.15$ & $0.98 \pm 0.03$ \\
\textsc{F-NPSE-SDE} ($L=5$) & $0.31 \pm 0.27$ & $0.73 \pm 0.28$ & $0.10 \pm 0.15$ & $0.98 \pm 0.03$ \\
\bottomrule
\end{tabular*}
\endgroup
\end{table*}

\paragraph{Exact-score Gaussian benchmark.} We compare \textsc{F-NPSE-SDE} with GAUSS, JAC, and Langevin \citep{linhart2026tall}, and the original F-NPSE \citep{geffner23a}. The 10-dimensional Gaussian model (\cref{sec:compositional-coefficient-evaluation}) has correlation $\rho=0.8$ and exact compositional scores and posteriors. Here, $n$ counts observations, $T$ predictor steps, and $L$ Langevin corrections per predictor. At $n=100$ and $T=400$, \cref{tab:fnpse-sde-benchmarks} shows that \textsc{F-NPSE-SDE} is competitive with GAUSS and JAC in maximum sliced Wasserstein distance, matches F-NPSE, and improves over Langevin. Full sweep settings and results appear in \cref{sec:gaussian-benchmark-details}.

\paragraph{SLCP benchmark.} We next evaluate compositional sampling on SLCP, where the scores must be learned and the posterior is multimodal. Following the protocol of \citet{linhart2026tall}, we evaluate three independently trained single-observation score networks at $n\in\{1,14,30\}$, with all samplers sharing the same frozen network within each comparison. At $n=30$, GAUSS and JAC remain the most accurate, while \textsc{F-NPSE-SDE} improves over Langevin and F-NPSE in sliced Wasserstein distance and MMD (\cref{tab:fnpse-sde-benchmarks}), without JAC's score Jacobians or GAUSS's auxiliary DDIM rollouts. Training, stabilization, and results across observation counts and network seeds are detailed in \cref{sec:slcp-benchmark-details}.

\subsection{Evaluating pooled compositional sampling and HBDS}
\label{sec:compare_pool_hbds}

\paragraph{Hierarchical toy model.} We first isolate the effect of hierarchical structure in a controlled model with a shared continuous parameter $\theta\in\mathbb{R}$ and group-specific latent states $z_g\in\{-1,+1\}$,
\[
z_g\sim\mathrm{Rad}\!\left(\tfrac12\right),\qquad
\theta\sim\mathcal N(0,1),\qquad
x_{gi}\mid\theta,z_g\sim\mathcal N(5z_g+\theta,1).
\]
\begin{wrapfigure}[13]{r}{0.50\textwidth}
\vspace{-6pt}
\centering
\includegraphics[width=\linewidth]{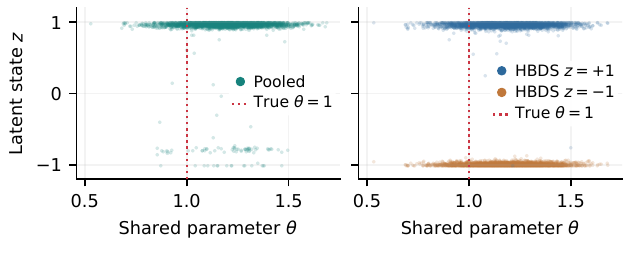}
\captionsetup{font=small,skip=3pt}
\caption{(\emph{Left}) Pooled sampling uses one shared latent. (\emph{Right}) HBDS retains separate group latents with a shared $\theta$. Colors identify the generating groups; plotted latent states are continuous sampler outputs. Dotted lines mark the generating $\theta=1$.}
\label{fig:hbds-toy}
\vspace{-6pt}
\end{wrapfigure}
Here $\mathrm{Rad}(\tfrac12)$ assigns equal probability to the two signs. We compare learned-score samplers using $n=20$ observations generated with $\theta=1$, and ten samples from $z=+1$ and $z=-1$. HBDS permits each group to select its own latent mode, whereas pooled compositional sampling ties all groups to a single latent state. \Cref{fig:hbds-toy} shows pooled sampling cannot represent groups requiring different latent modes, while HBDS maintains separate group-specific states together with a shared $\theta$. The analytic posterior structure and additional score-based diagnostics are given in \cref{sec:hbds-toy}. We further evaluate HBDS on the BMP model in \cref{sec:bmp-application}.

\begin{figure}[!b]
\vspace{-16pt}
\centering
\includegraphics[width=\linewidth]{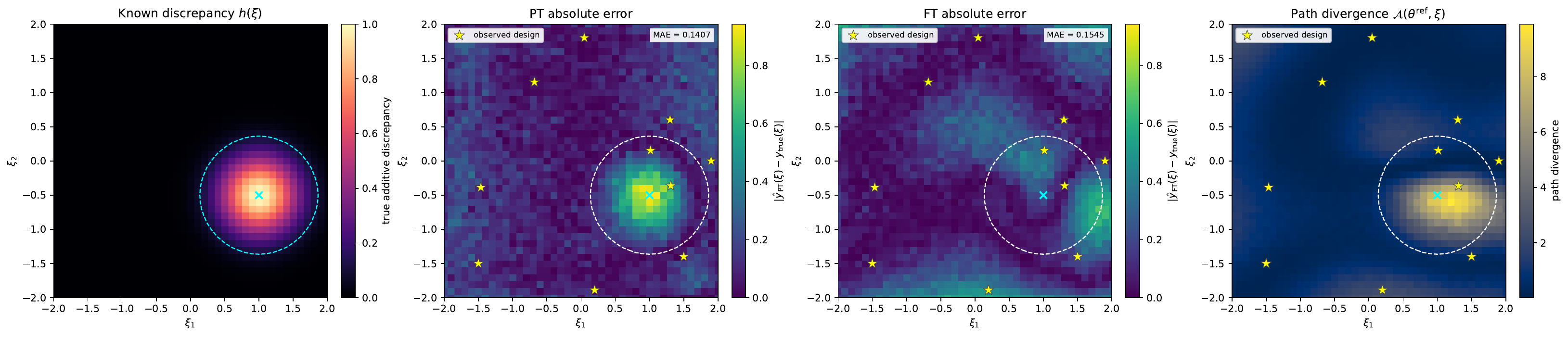}
\captionsetup{skip=4pt}
\caption{Likelihood adaptation and prediction error in the analytic toy at a fixed LSR reference. (\emph{Left to right}) Known discrepancy $h(\xi)$, PT absolute error $e_{\mathrm{PT}}(\xi)$, FT absolute error $e_{\mathrm{FT}}(\xi)$, and path divergence $\mathcal{A}(\theta^\text{ref}, \xi)$.}
\label{fig:design-misspecification-adjustment}
\end{figure}

\subsection{Fine-tuning under simulator misspecification}
\label{sec:imp_of_fine_tune}

\begin{figure}[!t]

\centering
\includegraphics[width=\textwidth]{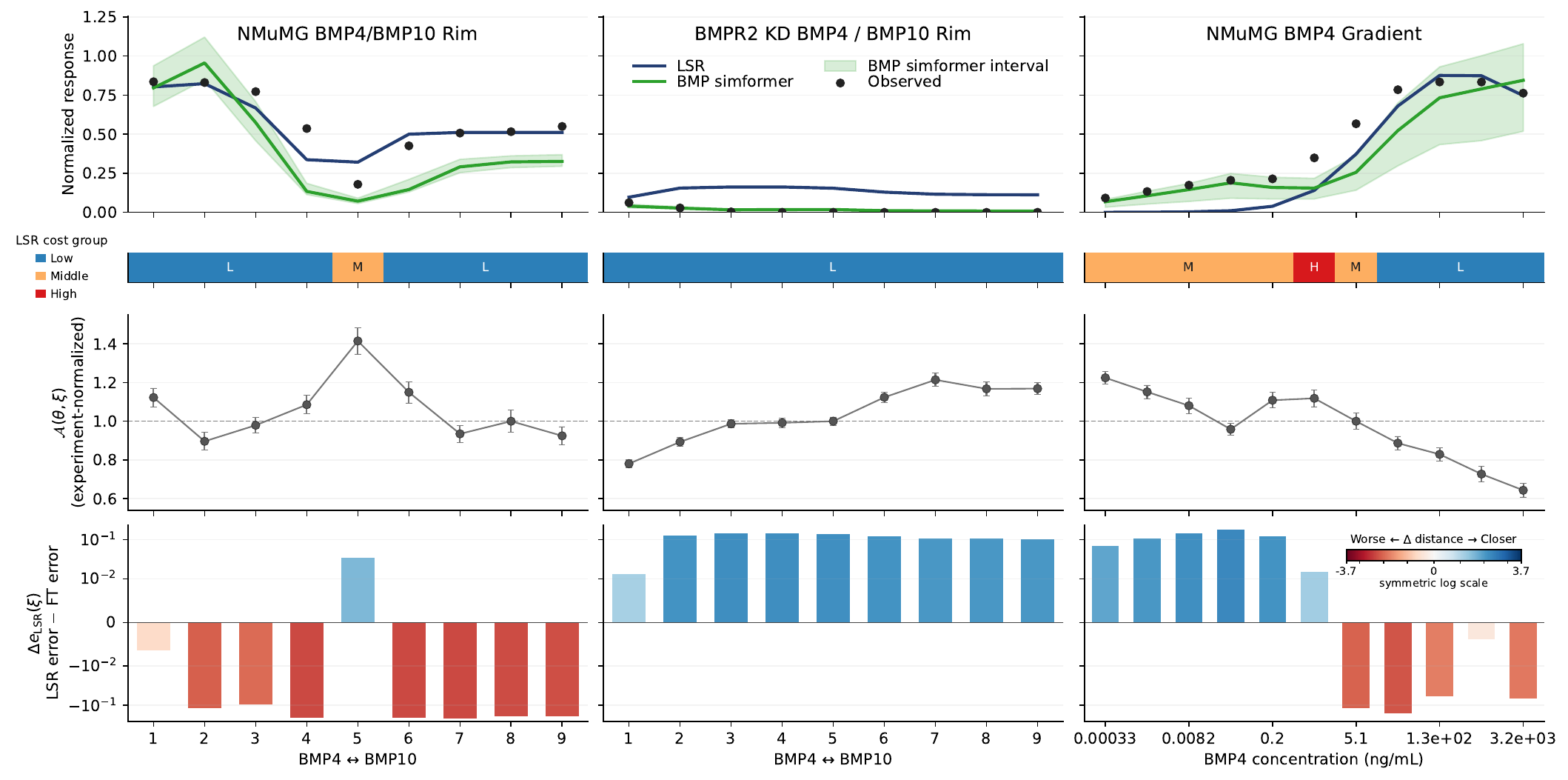}
\vspace{-16pt}
\caption{BMP predictions and design-localized adaptation. (\emph{Left}) BMP4/BMP10 competition in NMuMG. (\emph{Middle}) The same competition in BMPR2 KD. (\emph{Right}) BMP4 titration in NMuMG. From top to bottom: LSR and FT predictions against observations; low (L), middle (M), and high (H) LSR cost groups; experiment-normalized path divergence $\mathcal{A}(\theta,\xi)$; and reduction in absolute prediction error from LSR to the FT bootstrap median.}
\label{fig:rims_comparison}
\vspace{-20pt}
\end{figure}

\paragraph{Fine-tuning and path-divergence localization.} To study how fine-tuning responds to model misspecification and whether path divergence localizes the resulting adaptation, we use a toy linear-Gaussian simulator $y\mid\theta,\xi\sim\mathcal{N}(\xi^\top\theta,\sigma^2)$ with $\theta\sim\mathcal{N}(0,I_2)$, designs $\xi\in[-2,2]^2$, and $\sigma=0.2$. At a fixed generating parameter $\theta^\star$, the noise-free observation response is $y_o(\xi)=\xi^\top\theta^\star+h(\xi)$, where $h$ is a localized Gaussian bump omitted from the simulator. We fine-tune on ten observations with Gaussian noise of variance $\sigma^2$, then assess where the correction helps or hurts across design space. We then estimate the PT and FT likelihood means from 16 samples per design and measure their absolute errors against $y_o(\xi)$, giving the pointwise error maps $e_{\mathrm{PT}}(\xi)$ and $e_{\mathrm{FT}}(\xi)$ in \cref{fig:design-misspecification-adjustment}. We measure mean absolute error (MAE) averages over 1,681 grid points. Fine-tuning reduces error near the bump but adds errors elsewhere; however, employing autoguidance \citep{karras2024autoguidance} to blend the PT \& FT models lowers MAE to $0.12$ (\cref{fig:toy-autoguidance}). Without autoguidance, mean $\mathcal{A}(\theta^{\mathrm{LSR}},\xi)$ inside the misspecified region ($h\geq0.1$) is $4.73$ times that outside and correlates with $h$ (Pearson $r=0.75$). Path divergence and error maps show where fine-tuning helps or over-adapts, a trade-off we explore through regularization and autoguidance sweeps in \cref{sec:design-misspecification-toy}.

\paragraph{SBI misspecification comparison.}
On BMP (\cref{sec:bmp-application}), we compare with Flow Matching Corrected Posterior Estimation (FMCPE), which corrects an NPE using calibration pairs \citep{ruhlmann2026fmcpe}, here constructed from LSR fits. FMCPE uses subset-specific posterior models, whereas FT also adapts the likelihood and supports flexible conditioning, composition, and HBDS. Across the FMCPE observation- and LSR-budget sweep, FT achieves the lowest mean median predictive distance, while FMCPE-LSR has lower best-draw RMSE (\cref{eq:bmp-predictive-metrics,eq:bmp-median-predictive-distance,tab:fmcpe-lsr-budget}). Because both metrics evaluate predictions through the original simulator, the comparison remains subject to its misspecification. We provide the full evaluation protocol alongside PT comparisons and calibration diagnostics in \cref{sec:fmcpe-lsr-details}.

\subsection{BMP application}
\label{sec:bmp-application}

We combine fine-tuning with HBDS to infer shared biophysical parameters and cell-line-specific receptor states from 940 observations across four cell lines, using the model, priors, and training procedures described in \cref{sec:BMP_model,sec:hierarchical_receptors,sec:train_deets,sec:fine_tune_details}. We first hold parameters at the LSR reference to isolate how fine-tuning changes predictions. The bottom row of \cref{fig:rims_comparison} reports $\Delta e_{\mathrm{LSR}}(\xi)$, the LSR simulator's absolute error against $y_o(\xi)$ minus that of the FT bootstrap median, on a symmetric-logarithmic axis that is linear near zero. Positive values favor FT, with improvements throughout the BMPR2-knockdown competition series and at lower BMP4 doses. In the BMP4 titration, these gains generally coincide with larger path divergence, but adaptation in the NMuMG competition series often increases error, showing why the magnitude of adaptation must be interpreted alongside its predictive benefit. The LSR cost-group analysis adds context to these local differences, showing that globally higher-cost fits can perform better at individual designs (\cref{sec:bmp-simformer-misspecification}). When we use the FT model with $\lambda=5\times10^{-4}$ for posterior inference, HBDS yields RMSE $0.34$ and median distance $14.94$, compared with $0.36$ and $66.77$ under pooled sampling. We examine the contributions of the sampler, FiLM, and regularization, together with the aggregate LSR comparison, in \cref{tab:bmp_unified_summary,tab:lambda_sweep,tab:rmse_three_methods_transposed}.

\section{Discussion}
\label{sec:discussion}

Together, these methods enable flexible diffusion SBI across changing observation sets, hierarchical latent structure, and simulator misspecification, with the BMP model providing a large-scale scientific stress test. Rather than requiring a PT SBI model to anticipate this downstream structure, our methods adapt it to the problem at hand.

\paragraph{Limitations and future work.} \textit{Compositional sampling and HBDS} both rely on approximations to intermediate-time posteriors. More accurate intermediate targets and faster sampling with learned flow maps \citep{song2023consistencymodels,boffi2025flowmapmatchingstochastic} could improve compositional inference and HBDS block updates. \textit{Fine-tuning} at a fixed reference parameter $\theta^{\mathrm{ref}}$ can improve misspecified regions while degrading regions that were already well modeled. These results suggest tuning adaptation across experimental designs rather than imposing one global regularization strength, with predictive error and path divergence guiding where to focus updates and where to preserve the PT model. Path-space Fisher geometry over $(\theta,\xi)$ (\cref{sec:fisher-rao-analysis}) further suggests design-dependent adaptation, particle- or geometry-based refinement of $\theta^{\mathrm{ref}}$, and diffusion-based Bayesian experimental design over $\xi$ \citep{iollo2025codiff}.


\section*{Acknowledgements}

We thank the Elowitz Lab for helpful feedback. We also thank UCI's Research Cyberinfrastructure Center (RCIC) for providing and supporting the HPC3 computing resources used in this work \citep{uci2026rcic}. This research was partially funded by the National Institute of General Medical Sciences (NIGMS) of the National Institutes of Health (NIH) under award number 1F31GM145188-01.

\bibliography{uai2026-template}

\newpage
\onecolumn
\startcontents[appendices]
\appendix

\section*{Appendices}
\printcontents[appendices]{}{1}{\section*{Table of Contents}}

\newpage

\section{Compositional sampling details}
\label{sec:proof+derivation}

\subsection{Existence of cumulative transition score}
\label{sec:cum_trans_score}

We begin with a proposition for the existence of a cumulative transition score.

\begin{proposition}[Existence of a cumulative transition score]
Under the reverse--forward consistency assumption in \cref{eq:rev-fwd-consistency}, we construct a cumulative transition $\tilde{p}_k(\theta_{k-1} | \theta_k)$ whose score combines the per-observation reverse scores with the common forward-kernel correction. This gives the compositional transition used below.
\label{prop:cumul_tran_score}
\end{proposition}

\begin{proof}[Proof (existence of a cumulative transition kernel)]
Fix a discrete step $k\in\{1,\ldots,K\}$. For each observation $y_j$ let $p_k^{\,j}(\theta_k)$ denote the marginal at diffusion time $t_k$ and $q_k(\theta_k|\theta_{k-1})$ the common forward kernel.   Assume the (exact) reverse–forward consistency \begin{equation} p_k^{\,j}(\theta_k)\,p_k^{\,j}(\theta_{k-1}|\theta_k) \;=\; p_{k-1}^{\,j}(\theta_{k-1})\,q_k(\theta_k|\theta_{k-1}) \qquad\text{for each } j \in \{1,\dots,n\}.
\label{eq:rev-fwd-consistency}
\end{equation}

We seek a transition kernel $\tilde p_k(\theta_{k-1}|\theta_k)$ that satisfies the composition condition (cf.\ Eq.~(24) of \citet{geffner23a})
\begin{equation}
\tilde p_{k-1}(\theta_{k-1})
 \;=\; \int \tilde p_k(\theta_k)\,
 \tilde p_k(\theta_{k-1}|\theta_k)\,d\theta_k.
\label{eq:composition-condition}
\end{equation}

where $\tilde p_k(\theta_k)=Z_k^{-1}\prod_j^n p_k^{\,j}(\theta_k)$. Following \citet{geffner23a}, we rewrite \cref{eq:composition-condition} (cf.\ Eqs.~(25)–(26) therein) by using \cref{eq:rev-fwd-consistency} to express ratios of marginals and reverse kernels:
\[
\tilde p_{k-1}(\theta_{k-1})
= \int d\theta_k\, p_k^{\,1}(\theta_k)\,
   \frac{q_k(\theta_k|\theta_{k-1})}{p_k^{\,2}(\theta_{k-1}|\theta_k)}
   \cdots
   \frac{q_k(\theta_k|\theta_{k-1})}{p_k^{\,n}(\theta_{k-1}|\theta_k)}
   \frac{Z_k}{Z_{k-1}}\,
   \tilde p_k(\theta_{k-1}|\theta_k).
\]

One way to satisfy this equality, as noted in \citet{geffner23a}, is to set
\begin{equation}
\tilde p_k(\theta_{k-1}|\theta_k)
= p_k^{\,1}(\theta_{k-1}|\theta_k)\,
   \frac{Z_k}{Z_{k-1}}\,
   \frac{p_k^{\,2}(\theta_{k-1}|\theta_k)}{q_k(\theta_k|\theta_{k-1})}
   \cdots
   \frac{p_k^{\,n}(\theta_{k-1}|\theta_k)}{q_k(\theta_k|\theta_{k-1})}.
\label{eq:cum-transition}
\end{equation}

Hence, under exact reversibility, a cumulative transition kernel of the form exists, establishing the claim. \qedhere
\end{proof}

This derivation follows the compositional construction of \citet{geffner23a}, using the per-observation reverse--forward consistency relation in \cref{eq:rev-fwd-consistency}. We spell out the intermediate algebra to clarify how the individual reverse kernels and common forward kernel enter the cumulative transition, then derive its score under a Gaussian parameterization in \cref{sec:rev_kernel_der}. Our aim is to make the cumulative transition score explicit under these assumptions, rather than to establish an exact reverse kernel for the joint posterior.

\subsection{Derivation of the reverse transition kernel}
\label{sec:rev_kernel_der}

We next derive the analytic form of the cumulative transition score under a Gaussian diffusion parameterization. Starting from the cumulative reverse transition (cf.\ Eq.~\eqref{eq:cum-transition})
\[
\tilde p_k(\theta_{k-1}|\theta_k)
= p_k^{\,1}(\theta_{k-1}|\theta_k)\,
   \frac{Z_k}{Z_{k-1}}\,
   \prod_{j=2}^n \frac{p_k^{\,j}(\theta_{k-1}|\theta_k)}{q_k(\theta_k|\theta_{k-1})},
\]
where $Z_k$ and $Z_{k-1}$ are normalization constants ensuring integration to one. We retain the DDPM notation, with $k$ indexing discrete steps at diffusion times $t_k$ and $\alpha_k$ denoting one-step signal retention, distinct from the continuous-time cumulative factor $\alpha(t)$ used in \cref{sec:forward_kernel_derivation}. The Gaussian kernels are parameterized as
\begin{align}
    p_k^{\,j}(\theta_{k-1}|\theta_k)&=\mathcal{N} \big(\theta_{k-1};\,\mu_k^{\,j},(1-\alpha_k)I\big),
    \label{eq:backward_forward_kernels}\\
    q_k(\theta_k|\theta_{k-1})&=\mathcal{N} \big(\theta_k;\,\sqrt{\alpha_k}\,\theta_{k-1},(1-\alpha_k)I\big),
    \label{eq:gaussian-forward-kernel}
\end{align}
and we differentiate $\log \tilde p_k(\theta_{k-1} | \theta_k)$ with respect to $\theta_{k-1}$ to get
\begin{equation}
\nabla_{\theta_{k-1}}\log \tilde p_k(\theta_{k-1} | \theta_k)
= \sum_{j=1}^n \nabla_{\theta_{k-1}}\log p_k^{\,j}(\theta_{k-1} | \theta_k)
  - (n-1)\,\nabla_{\theta_{k-1}}\log q_k(\theta_k | \theta_{k-1}),
\label{eq:sum-of-reverse-minus-forward}
\end{equation}

where the normalization constants cancel after differentiation. For the reverse Gaussian kernels, the mean can be written in terms of the learned score function as
\begin{equation}
\mu_k^{\,j}
= \frac{1}{\sqrt{\alpha_k}}\,\theta_k
  + \frac{1-\alpha_k}{\sqrt{\alpha_k}}\,s_\phi^{\,j}(\theta_k,t_k),
\label{eq:mu-score-identity}
\end{equation}
where $s_\phi^{\,j}(\theta_k,t_k)$ denotes the estimated score for observation $y_j$. Before substituting the score of the $\log$ of these kernels into \cref{eq:sum-of-reverse-minus-forward}, we recall two Gaussian score identities. For a Gaussian density $\mathcal{N}(x;\mu,\Sigma)$, the score with respect to the argument $x$ is
\[
\nabla_x \log \mathcal{N}(x;\mu,\Sigma) = -\Sigma^{-1}(x-\mu),
\]
while the score with respect to the mean $\mu$ is
\[
\nabla_\mu \log \mathcal{N}(x;\mu,\Sigma) = \Sigma^{-1}(x-\mu).
\] 
Applying these identities gives
\begin{align}
\nabla_{\theta_{k-1}}\log p_k^{\,j}(\theta_{k-1} | \theta_k)
&= \frac{\mu_k^{\,j}-\theta_{k-1}}{1-\alpha_k},
\label{eq:gaussian-score-identities}\\
\nabla_{\theta_{k-1}}\log q_k(\theta_k | \theta_{k-1})
&= \frac{\sqrt{\alpha_k}}{1-\alpha_k}\,\theta_k - \frac{\alpha_k}{1-\alpha_k}\,\theta_{k-1}.
\label{eq:gaussian-forward-score}
\end{align}

We then substitute these into \cref{eq:sum-of-reverse-minus-forward} and group by $\theta_k$ and $\theta_{k-1}$ terms to form the cumulative transition kernel in terms of the mean values at the transition
\begin{align}
\nabla_{\theta_{k-1}}\log \tilde p_k(\theta_{k-1}|\theta_k)
&= \sum_{j=1}^n \frac{\mu_k^{\,j}-\theta_{k-1}}{1-\alpha_k}
   \;-\;(n-1)\left(\frac{\sqrt{\alpha_k}}{1-\alpha_k}\theta_k
                   - \frac{\alpha_k}{1-\alpha_k}\theta_{k-1}\right) \notag \\[6pt]
&= \frac{\sum_{j=1}^n \mu_k^{\,j} - n\,\theta_{k-1}}{1-\alpha_k}
   - \frac{(n-1)\sqrt{\alpha_k}}{1-\alpha_k}\,\theta_k
   + \frac{(n-1)\alpha_k}{1-\alpha_k}\,\theta_{k-1} \notag \\[6pt]
&= \frac{\sum_{j=1}^n \mu_k^{\,j} - (n-1)\sqrt{\alpha_k}\,\theta_k}{1-\alpha_k}
   - \frac{n - \alpha_k(n-1)}{1-\alpha_k}\,\theta_{k-1}.
   \label{eq:grouped-theta}
\end{align}

We can now identify the terms previously used by \citet{geffner23a} for compositional sampling. Matching the grouped score to the Gaussian identity
\[
\nabla_{\theta_{k-1}}\log \mathcal{N}(\theta_{k-1};\mu_k,\sigma_k^2)
= \sigma_k^{-2}\mu_k - \sigma_k^{-2}\theta_{k-1},
\]
we identify the coefficients of $\theta_{k-1}$ and the constant term as
\[
\sigma_k^{-2} = \frac{n-\alpha_k(n-1)}{1-\alpha_k},
\qquad
\sigma_k^{-2}\mu_k = \frac{1}{1-\alpha_k}\sum_{j=1}^n \mu_k^{\,j}
- \frac{(n-1)\sqrt{\alpha_k}}{1-\alpha_k}\,\theta_k.
\]
Therefore,
\begin{align}
\mu_k &= \frac{\sum_{j=1}^n \mu_k^{\,j} - (n-1)\sqrt{\alpha_k}\,\theta_k}
{\,n-\alpha_k(n-1)},
\label{eq:normalized-mu-sigma}\\
\sigma_k^2 &= \frac{1-\alpha_k}{\,n-\alpha_k(n-1)},
\label{eq:normalized-transition-variance}
\end{align}
as originally used by \citet{geffner23a}. We can now substitute approximate scores for a final cumulative score expression, $\nabla_{\theta_{k-1}}\log \tilde p_k(\theta_{k-1} | \theta_k)$. Using the equivalence between diffusion means and scores \eqref{eq:mu-score-identity},
\[
\sum_{j=1}^n \mu_k^{\,j}
= \frac{n}{\sqrt{\alpha_k}}\,\theta_k
  + \frac{1-\alpha_k}{\sqrt{\alpha_k}} \sum_{j=1}^n s_\phi^{\,j}(\theta_k,t_k),
\]
Writing the right-hand side of \cref{eq:grouped-theta} as $\mathrm{RHS}$, we substitute the score parameterization to obtain
\begin{align}
\text{RHS} 
&= \frac{1}{1-\alpha_k}\sum_{j=1}^n \mu_k^{\,j}
   - \frac{(n-1)\sqrt{\alpha_k}}{1-\alpha_k}\,\theta_k
   + \frac{(n-1)\alpha_k - n}{1-\alpha_k}\,\theta_{k-1}
   \notag \\
&= \frac{1}{1-\alpha_k}\left(
      \frac{n}{\sqrt{\alpha_k}}\theta_k
      + \frac{1-\alpha_k}{\sqrt{\alpha_k}}\sum_{j=1}^n s_\phi^{\,j}(\theta_k,t_k)
   \right)
   - \frac{(n-1)\sqrt{\alpha_k}}{1-\alpha_k}\,\theta_k
   + \frac{(n-1)\alpha_k - n}{1-\alpha_k}\,\theta_{k-1} \notag \\[4pt]
&= \frac{1}{\sqrt{\alpha_k}}\sum_{j=1}^n s_\phi^{\,j}(\theta_k,t_k)
   + \left[\frac{n}{(1-\alpha_k)\sqrt{\alpha_k}}
           - \frac{(n-1)\sqrt{\alpha_k}}{1-\alpha_k}\right]\theta_k
   + \frac{(n-1)\alpha_k - n}{1-\alpha_k}\,\theta_{k-1} \notag \\[4pt]
&= \frac{1}{\sqrt{\alpha_k}}\sum_{j=1}^n s_\phi^{\,j}(\theta_k,t_k)
   + \frac{\big(n-(n-1)\alpha_k\big)\theta_k
          + \sqrt{\alpha_k}\,\big((n-1)\alpha_k - n\big)\theta_{k-1}}
          {(1-\alpha_k)\sqrt{\alpha_k}}. 
\label{eq:cumulative_score}
\end{align}

Finally, we arrive at the cumulative transition score. To apply this insight in practice, we need to define the $\mu_k$ and $\sigma_k^2$ terms that will be used in a Gaussian transition. Grouping the $\theta_k$ and $\theta_{k-1}$ terms as before and simplifying we arrive at the mean transition in terms of different scores

\begin{align}
\mu_k
&=
\frac{1}{\sqrt{\alpha_k}}\theta_k
+
\frac{1-\alpha_k}
{\sqrt{\alpha_k}\big(n-\alpha_k(n-1)\big)}
\sum_{j=1}^n s_\phi^j(\theta_k,t_k),
\label{eq:normalized-mu-sigma-final}\\
\sigma_k^2
&=
\frac{1-\alpha_k}{\,n-\alpha_k(n-1)}.
\label{eq:score-transition-variance}
\end{align}

Notice that both the aggregate score contribution and the transition variance contain the common factor $n-\alpha_k(n-1)$, which depends jointly on the number of composed observations and the noise schedule. We denote this quantity by
\begin{equation*}
\kappa_k^{(n)}
:=
n-\alpha_k(n-1),
\end{equation*}
and refer to it as the \emph{compositional precision factor for $n$ composed observations}. It controls how composition rescales the aggregate score and the variance of the normalized reverse transition.

The mean and variance of this Gaussian reverse kernel define a DDPM-style predictor,
\begin{equation}
\theta_{k-1}
= \frac{1}{\sqrt{\alpha_k}}\,\theta_k
\,+\,
\frac{1-\alpha_k}{\sqrt{\alpha_k} \, \kappa_k^{(n)} }
\sum_{j=1}^{n} s_\phi^{\,j}(\theta_k,t_k)
\,+\,
\sqrt{\frac{1-\alpha_k}{ \kappa_k^{(n)} }}\,z_k,
\; z_k \sim \mathcal{N}(0,I).
\label{eq:comp_sde}
\end{equation}
To relate these steps to a continuous VP schedule, take a forward-time grid $0=t_0<\cdots<t_K=1$ and let $\alpha(t)$ denote cumulative squared signal retention, with $\alpha(0)=1$. The one-step factors and their cumulative product satisfy
\begin{align}
  \alpha_k&=\frac{\alpha(t_k)}{\alpha(t_{k-1})},
  \label{eq:ddpm-cumulative-retention}\\
  \prod_{j=1}^{k}\alpha_j&=\alpha(t_k).
  \label{eq:ddpm-retention-product}
\end{align}
The score network is evaluated at the corresponding continuous time $t_k$, while $\alpha_k$ controls only the transition between $t_{k-1}$ and $t_k$.

\subsection{Deriving the continuous-time diffusion coefficient}
\label{sec:forward_kernel_derivation}

We seek a continuous-time counterpart to the discrete predictor in \cref{eq:comp_sde}. Here $t$ denotes continuous diffusion time, $\alpha(t)$ denotes cumulative squared signal retention from the clean state, and a dot denotes differentiation with respect to $t$. The one-step DDPM factors remain $\alpha_k$, with discrete steps indexed by $k$. For a forward process with drift $f(t)\theta$, diffusion coefficient $g(t)$, and corresponding marginal $p_t$, the reverse SDE admits a predictor-corrector (PC) decomposition \citep{song2021scorebased},
 
\begin{equation}
  \mathrm{d}\Theta_t = 
  \underbrace{\big(f(t)\theta_t - \tfrac{1}{2} g(t)^2 \nabla \log p_t(\theta_t)\big)\,\mathrm{d}t}_{\text{Probability Flow Prediction ODE}}
  + 
  \underbrace{\big(-\tfrac{1}{2} g(t)^2 \nabla \log p_t(\theta_t)\,\mathrm{d}t + g(t)\,\mathrm{d}\widetilde{w}_t\big)}_{\text{Langevin Correction SDE}}.
\label{eq:ode_langevin}
\end{equation}

Sampling follows this equation toward decreasing $t$. The probability-flow term advances samples between noise levels, while the Langevin term provides score-driven exploration at the current level. The shared coefficient $g(t)$ determines the strength of both the score-driven motion and the injected noise. For compositional sampling, we seek an observation-count-dependent choice $g(t)=g_n(t)$ to improve the stability and numerical accuracy of these corrections. We examine its effects and the associated exploration trade-offs in \cref{sec:langevin-refinement}.

Here we construct $g_n(t)$ by retaining the ordinary VP mean attenuation and prescribing a Gaussian noising variance that contracts with the observation count $n$, motivated by the discrete variance contraction in \cref{eq:comp_sde}. Specifying this conditional kernel does not require knowing the clean posterior or assuming that it is Gaussian. With the VP drift fixed, the diffusion coefficient then follows from the chosen variance evolution.

\paragraph{Deriving the diffusion coefficient $g_n(t)$.}

Let $x=x_{1:n}$ collect the observed design--response pairs. Suppose we could draw $\theta_0\sim p_0(\theta\mid x)$ from the clean posterior. Ordinary VP noising would perturb this draw according to
\[
\Theta_t\mid\Theta_0=\theta_0
\sim
\mathcal N\!\left(
\sqrt{\alpha(t)}\,\theta_0,\,
[1-\alpha(t)]I
\right).
\]
Here $\alpha(t)$ is the cumulative squared signal retention, with $\alpha(0)=1$ and $\dot\alpha(t)=-\beta(t)\alpha(t)$ for the VP noise-rate schedule $\beta(t)\geq0$. This kernel has accumulated noise variance $1-\alpha(t)$ per coordinate, independent of the number of composed observations $n$.

The discrete compositional update in \cref{eq:comp_sde} contracts its one-step Gaussian reverse-transition variance by $\kappa_k^{(n)}=1+(n-1)(1-\alpha_k)$. We use the same functional form with cumulative retention $\alpha(t)$ to prescribe the conditional noising variance, defining a distinct cumulative contraction factor $\kappa^{(n)}(t)$,
\begin{align}
\operatorname{Cov}(\Theta_t\mid\Theta_0=\theta_0)
&=
V_n(t)I,
\label{eq:compositional-conditional-covariance}\\
V_n(t)
&\coloneqq
\frac{1-\alpha(t)}{\kappa^{(n)}(t)},
\label{eq:compositional-cumulative-variance}\\
\kappa^{(n)}(t)
&\coloneqq
1+(n-1)(1-\alpha(t)).
\label{eq:continuous-compositional-precision}
\end{align}
This cumulative contraction is a modeling prescription, not the infinitesimal limit of the DDPM transition. The factors $\kappa_k^{(n)}$ and $\kappa^{(n)}(t)$ share a functional form but act on one-step and accumulated noise variances, respectively. The prescribed kernel recovers ordinary VP noising at $n=1$ and reduces the noising variance for $n>1$.

With $V_n(t)$ specified, the remaining task is to find the diffusion coefficient $g_n(t)$ that produces it while retaining the VP drift $f_n(t)=-\beta(t)/2$. For the forward SDE $\mathrm d\Theta_t=f_n(t)\Theta_t\,\mathrm dt+g_n(t)\,\mathrm dw_t$, applying It\^o's formula to each coordinate's squared deviation from its conditional mean and taking expectations gives the variance evolution
\begin{equation}
\dot V_n(t)
=
2f_n(t)V_n(t)+g_n(t)^2,
\qquad
V_n(0)=0.
\label{eq:variance-evolution}
\end{equation}
The first term comes from differentiating the squared deviation under the drift. With the chosen VP drift, $2f_n(t)V_n(t)=-\beta(t)V_n(t)$, so this contribution contracts the existing variance. The second term is It\^o's additional contribution from the noise increment's variance $g_n(t)^2\,\mathrm dt$, which the ordinary chain rule would miss. The remaining stochastic term has zero expectation. Consequently, the required diffusion variance rate is
\begin{equation}
g_n(t)^2
=
\dot V_n(t)-2f_n(t)V_n(t)
=
\dot V_n(t)+\beta(t)V_n(t).
\label{eq:diffusion-from-variance}
\end{equation}
Thus, $g_n(t)^2$ must both compensate for contraction under the
drift and produce the prescribed increase in cumulative variance;
it is not simply $V_n(t)$ or $\dot V_n(t)$.

To evaluate this expression, we first differentiate the prescribed
variance with respect to cumulative retention, then apply the chain
rule to obtain its time derivative. Holding $n$ fixed and temporarily
writing $\alpha$ for $\alpha(t)$, the numerator $1-\alpha$ has
derivative $-1$, while the denominator
$n-(n-1)\alpha$ has derivative $-(n-1)$. The quotient rule gives
\[
\begin{aligned}
\frac{\mathrm d}{\mathrm d\alpha}
\left(
\frac{1-\alpha}{n-(n-1)\alpha}
\right)
&=
\frac{
(-1)\bigl[n-(n-1)\alpha\bigr]
-
(1-\alpha)\bigl[-(n-1)\bigr]
}{
[n-(n-1)\alpha]^2
}
\\
&=
\frac{
-n+(n-1)\alpha+(n-1)-(n-1)\alpha
}{
[n-(n-1)\alpha]^2
}
\\
&=
-\frac{1}{[n-(n-1)\alpha]^2}.
\end{aligned}
\]
The terms proportional to $\alpha$ cancel, leaving
$-n+(n-1)=-1$ in the numerator. Restoring the time dependence
and applying the chain rule yields
\begin{equation}
\begin{aligned}
\dot V_n(t)
&=
\left.
\frac{\mathrm d}{\mathrm d\alpha}
\left(
\frac{1-\alpha}{n-(n-1)\alpha}
\right)
\right|_{\alpha=\alpha(t)}
\dot\alpha(t)
\\
&=
-\frac{\dot\alpha(t)}{[\kappa^{(n)}(t)]^2}
=
\frac{\beta(t)\alpha(t)}{[\kappa^{(n)}(t)]^2},
\end{aligned}
\label{eq:compositional-variance-derivative}
\end{equation}
where the final equality uses
$\dot\alpha(t)=-\beta(t)\alpha(t)$.

Substituting this derivative and the prescribed variance
$V_n(t)=(1-\alpha(t))/\kappa^{(n)}(t)$ into
\cref{eq:diffusion-from-variance} gives
\[
\begin{aligned}
g_n(t)^2
&=
\frac{\beta(t)\alpha(t)}{[\kappa^{(n)}(t)]^2}
+
\beta(t)\frac{1-\alpha(t)}{\kappa^{(n)}(t)}
\\
&=
\beta(t)
\frac{
\alpha(t)+(1-\alpha(t))\kappa^{(n)}(t)
}{
[\kappa^{(n)}(t)]^2
}.
\end{aligned}
\]
To simplify the numerator, substitute
$\kappa^{(n)}(t)=1+(n-1)(1-\alpha(t))$:
\[
\begin{aligned}
\alpha(t)+(1-\alpha(t))\kappa^{(n)}(t)
&=
\alpha(t)
+
(1-\alpha(t))
\bigl[1+(n-1)(1-\alpha(t))\bigr]
\\
&=
\alpha(t)+(1-\alpha(t))
+(n-1)(1-\alpha(t))^2
\\
&=
1+(n-1)(1-\alpha(t))^2.
\end{aligned}
\]
Consequently, the diffusion variance rate is
\begin{equation*}
g_n(t)^2
=
-\frac{\dot\alpha(t)}{\alpha(t)}
\frac{
1+(n-1)(1-\alpha(t))^2
}{
[\kappa^{(n)}(t)]^2
}.
\end{equation*}
Equivalently, in terms of the VP noise-rate schedule,
\begin{equation}
g_n(t)^2
=
\beta(t)
\frac{
1+(n-1)(1-\alpha(t))^2
}{
[1+(n-1)(1-\alpha(t))]^2
}.
\label{eq:compositional-diffusion-beta}
\end{equation}

For $n=1$, $\kappa^{(1)}(t)=1$, so the prescribed cumulative variance reduces to $V_1(t)=1-\alpha(t)$ and the diffusion coefficient satisfies $g_1(t)^2=\beta(t)$, recovering the ordinary VP-SDE. For general $n$, these coefficients realize the chosen linear Gaussian forward process with ordinary VP mean attenuation and reduced cumulative noising variance $V_n(t)$.

The coefficient satisfies $g_n(t)^2\leq\beta(t)$ and depends only on diffusion time and observation count. It therefore provides a composition-dependent sampling scale without score Jacobians or auxiliary covariance estimates. This construction specifies a chosen forward noising process, rather than establishing an exact reverse process for the product bridge.

\subsection{Langevin refinement and numerical accuracy}
\label{sec:langevin-refinement}

Discretizing the Langevin term in \cref{eq:ode_langevin} gives an unadjusted Langevin algorithm (ULA) correction between predictor steps. We now examine how the choice $g(t)=g_n(t)$ affects the stability and accuracy of these updates. We assess how efficiently these corrections explore the compositional marginals $\widetilde p_t$ introduced in \cref{sec:comp_samp}.

\paragraph{Mixing at fixed diffusion time.} The predictor advances between noise levels, while Langevin corrections help particles explore each intermediate target before moving on. Their effectiveness depends on both the amount of exploration and the numerical accuracy of the updates. \textit{Strongly curved or anisotropic marginals} can make these corrections sensitive, e.g. rapidly contracting directions may force small steps even when exploration in other directions requires much longer trajectories -- a form of numerical stiffness. Starting from the predictor output $\theta_t^{(0)}$ and writing $\widetilde s_t=\nabla_\theta\log\widetilde p_t$, one ULA correction is
\begin{equation}
  \theta_t^{(\ell+1)}
  =\theta_t^{(\ell)}
   +\underbrace{\frac{\eta g(t)^2}{2}\,\widetilde s_t(\theta_t^{(\ell)})}_{\text{drift update}}
   +\underbrace{\sqrt{\eta}\,g(t)\,\varepsilon_\ell}_{\text{diffusion noise}},
  \qquad \varepsilon_\ell\overset{\mathrm{iid}}{\sim}\mathcal N(0,I_d).
\label{eq:ula-corrector}
\end{equation}
Here $L$ is the number of corrections, $\ell=0,\ldots,L-1$ indexes them, and $d$ is the parameter dimension. The Langevin correction step size $\eta>0$ controls the size of each update while the diffusion noise level $t$ remains unchanged. This update uses the coefficient $g(t)=g_n(t)$ from \cref{eq:ode_langevin}. Its score-driven drift $\tfrac12 g(t)^2\widetilde s_t$ moves toward higher density, while its noise covariance $\eta g(t)^2I_d$ supplies exploration. The term $f(t)\theta$ belongs to the predictor and is not part of this within-level correction. ``Unadjusted'' means there is no Metropolis accept/reject step. Consequently, for a fixed nonzero step size, a stable ULA chain generally has an invariant distribution that differs from $\widetilde p_t$, introducing discretization bias \citep{roberts1996langevin,durmus2017ula}.

\paragraph{Numerical accuracy of Langevin correction.} As a one-dimensional Gaussian illustration of this step-size restriction, consider the target $\mathcal N(m,v)$ with mean $m$ and variance $v$. Substituting its score $\widetilde s_t(\theta)=-(\theta-m)/v$ into \cref{eq:ula-corrector} gives
\begin{equation}
  \theta_t^{(\ell+1)}-m
  =\left(1-\frac{\eta g(t)^2}{2v}\right)(\theta_t^{(\ell)}-m)
   +\sqrt{\eta}\,g(t)\,\varepsilon_\ell,
  \qquad \varepsilon_\ell\sim\mathcal N(0,1).
\label{eq:ula-gaussian-contraction}
\end{equation}
For example, if $\eta g(t)^2=6v$, the deterministic part takes a displacement $\delta$ from the mean to $-2\delta$. It crosses the mean and ends twice as far away. More generally, two trajectories receiving identical noise have their separation multiplied by $1-\eta g(t)^2/(2v)$, which amplifies its magnitude when $\eta g(t)^2>4v$. Here the score changes at rate $-1/v$ with the parameter, so a narrower target produces a more rapidly varying restoring drift. \textit{Stability therefore depends on the diffusion coefficient and correction step size relative to the target variance, not on score magnitude alone.}

\textit{Numerical stability does not by itself imply accurate sampling.} Even in the stable range $0<\eta g(t)^2<4v$, this chain has stationary variance $v/[1-\eta g(t)^2/(4v)]$, which exceeds the target variance $v$. Reducing $g(t)^2$ or $\eta$ reduces this discretization bias but also slows exploration, potentially requiring more corrections. Increasing $L$ alone does not remove the bias.

In this illustration, using $g_n(t)^2\leq\beta(t)$ reduces the update scale relative to $v$ and, within the stable range, the stationary-variance inflation at a given $\eta$. Because the schedule depends on observation count rather than local curvature, the step size and correction count still need to be chosen appropriately.

\paragraph{Choosing correction step sizes and counts.} \citet[\S3]{touron2026annealed} give rules for choosing the step size and number of Langevin steps together to meet a desired Wasserstein error tolerance in annealed Langevin sampling. Their guarantees require smooth, strongly log-concave intermediate densities and controlled score error. In their bound, more steps reduce mixing error but do not remove the residual error from the finite step size and score approximation. The required number of steps also depends on how much the target changes between noise levels, so the correction count need not be the same throughout sampling.

\paragraph{Choosing diffusion noise levels.} Choosing the noise levels themselves is a separate optimization problem. For example, \citet{williams2024score} construct score-optimal schedules by estimating the cost of moving between successive marginals and redistributing the time points to approximately equalize that cost. This requires an additional cost-estimation and grid-adaptation procedure, distinct from choosing the number of ULA corrections.

\subsection{Experimental evaluation}
\label{sec:compositional-coefficient-evaluation}

\paragraph{Gaussian benchmark.} We used the correlated Gaussian model of \citet{linhart2026tall}, with parameter dimension $d=10$ and observation-noise correlation $\rho=0.8$. The prior and conditionally independent observations are
\begin{align}
  \theta&\sim\mathcal N(0,I_d),
  \label{eq:gaussian-benchmark-prior}\\
  x_j\mid\theta&\overset{\mathrm{iid}}{\sim}\mathcal N(\theta,\Sigma),
  \label{eq:gaussian-benchmark-model}\\
  \Sigma&=(1-\rho)I_d+\rho\mathbf 1\mathbf 1^\top,
  \label{eq:gaussian-benchmark-covariance}
\end{align}
for $j=1,\ldots,n$, where $\mathbf 1\in\mathbb R^d$ is the all-ones vector. The exact posterior is $p(\theta\mid x_{1:n})=\mathcal N(\theta;m_0,C_0)$, with mean $m_0$ and covariance $C_0$ before diffusion noise is added, given by
\[
  C_0=(I_d+n\Sigma^{-1})^{-1},\qquad
  m_0=C_0\Sigma^{-1}\sum_{j=1}^n x_j.
\]

\paragraph{Controlled Gaussian setup.} We compared the derived diffusion variance rate $g_n(t)^2$ with the ordinary VP rate $g_{\mathrm{VP}}(t)^2=\beta(t)$ on the same prescribed Gaussian path. We examined numerical stability, covariance tracking, and posterior-sampling accuracy as $n$ increases. Covariance tracking compares the sample covariance with its analytic target at each noise level. It checks whether the samples have the intended spread and correlations, which numerical stability alone does not guarantee. Exact scores let us isolate these effects without score-estimation error, and each paired comparison held the terminal particles, random noise draws, time grid, and correction count fixed. For Euler--Maruyama sampling, we used $\beta(t)=0.05+19.95t$, $\alpha(t)=\exp(-0.05t-9.975t^2)$, and the forward process from \cref{eq:compositional-cumulative-variance}. At diffusion time $t$, this gives a Gaussian marginal with mean $m_t$, covariance $C_t$, and exact score $s_t^{\mathrm G}$, the gradient of its log density,
\[
  m_t=\sqrt{\alpha(t)}\,m_0,\qquad
  C_t=\alpha(t)C_0+\frac{1-\alpha(t)}{\kappa^{(n)}(t)}I_d,\qquad
  s_t^{\mathrm G}(\theta)=-C_t^{-1}(\theta-m_t).
\]
Both coefficient choices, $c(t)=g_n^2(t)$ and $c(t)=\beta(t)$, retained this score, the linear reverse drift $\beta(t)\theta/2$, and the terminal law $\mathcal N(m_1,C_1)$. Each run used 1,000 particles and 400 equal reverse-time steps from $1$ to $0$, with step size $h=1/400$ and no final Tweedie correction. We replaced the diffusion coefficient in both the predictor and ULA correctors while keeping the score and initialization unchanged. This isolates the effect of using $g_n(t)^2$ rather than $\beta(t)$ for the prescribed compositional noise path.

\begin{figure*}[!t]
  \centering
  \includegraphics[width=0.98\textwidth]{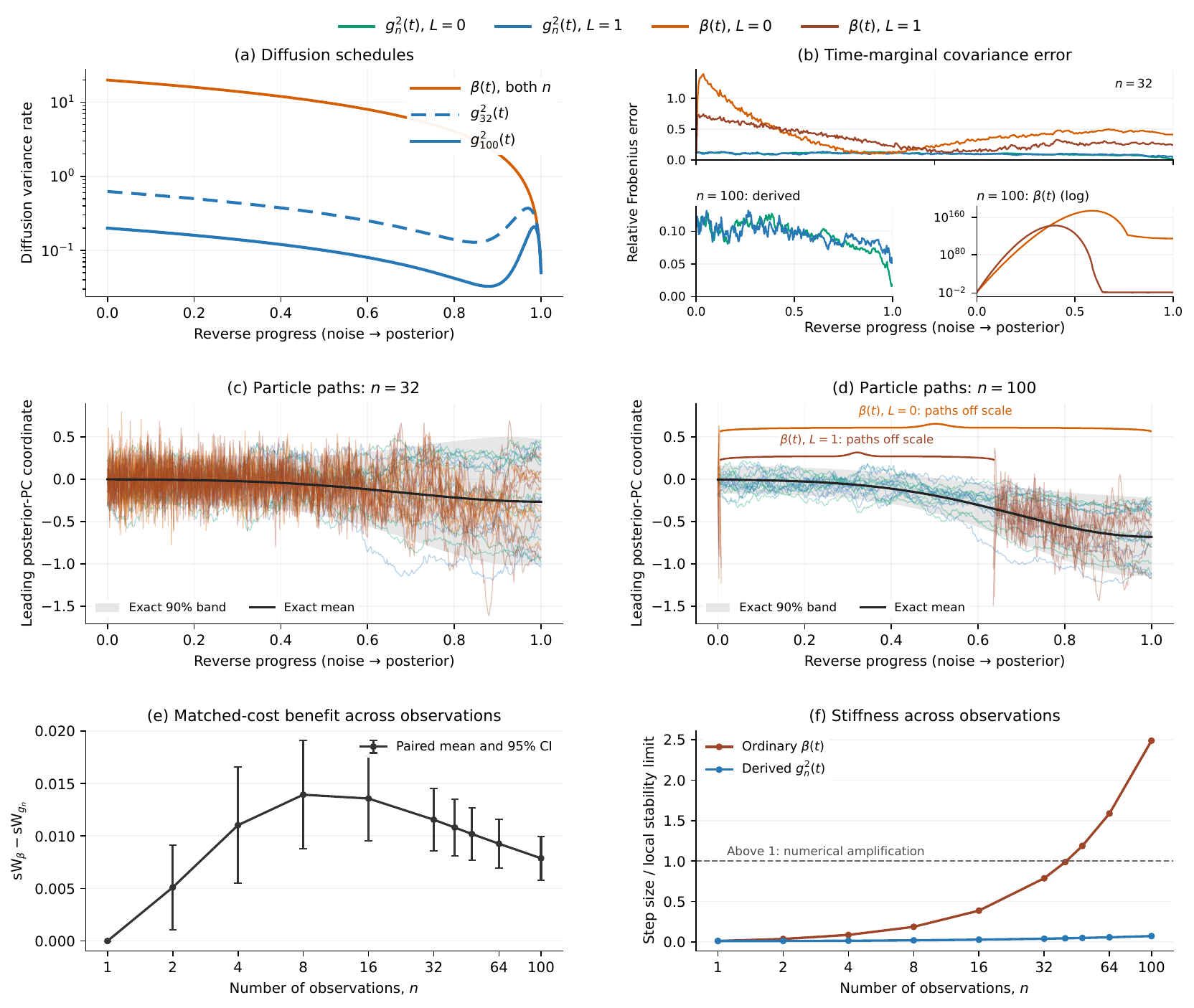}
  \caption{Compositional diffusion improves finite-budget Gaussian sampling. (\emph{a--d}) Diffusion schedules, covariance tracking, and matched particle paths. At $n=100$, ordinary-VP covariance errors require a logarithmic scale; path gaps and braces mark off-scale excursions. (\emph{e}) Endpoint sW reduction at matched cost, with positive values favoring $g_n^2(t)$; bars are pointwise 95\% intervals over five paired seeds. (\emph{f}) Step size relative to the local contraction limit. Ordinary VP crosses $1$, indicating numerical amplification, while the derived coefficient stays below it. Accuracy improves even below this threshold.}
  \label{fig:pooled-sde-comparison}
  \label{fig:derived-vs-beta-n32}
  \label{fig:pooled-generation-n100}
\end{figure*}

\paragraph{Stiffness and step-size sensitivity.} Here the predictor advances between noise levels using an Euler--Maruyama step of the full reverse SDE in \cref{eq:ode_langevin}, before any additional ULA corrections. To assess stability, imagine two nearby particles receiving the same Gaussian noise. Their noise increments cancel in the difference, isolating how the drift changes their separation. This depends on the spatial derivative of the drift, its Jacobian, rather than on the drift's magnitude alone. Let $h>0$ denote the sampling step size, so each predictor advances from diffusion time $t$ to $t-h$. Combining the two score terms in \cref{eq:ode_langevin} gives the sampling drift $b_t$ and its Jacobian in this reverse-time direction,
\begin{align}
  b_t(\theta)
  &=\frac{\beta(t)}{2}\theta+c(t)s_t^{\mathrm G}(\theta)
   =\frac{\beta(t)}{2}\theta-c(t)C_t^{-1}(\theta-m_t),
  \label{eq:gaussian-predictor-drift}\\
  J_t\coloneqq\frac{\partial b_t}{\partial\theta}
  &=\frac{\beta(t)}{2}I_d-c(t)C_t^{-1}.
  \label{eq:gaussian-predictor-jacobian}
\end{align}
At a given noise level, $m_t$ and $C_t$ are fixed with respect to $\theta$, so differentiating the Gaussian score contributes $-C_t^{-1}$. Writing $\delta$ for the separation between the two particles, one predictor step gives $\delta_{\mathrm{next}}=(I_d+hJ_t)\delta$. To check whether discretization turns contraction into expansion, we examine this update one eigendirection at a time. Here $J_t$ is symmetric, so its eigenvectors form an orthogonal basis and any separation can be decomposed into components along these directions. Along each eigenvector, applying the Jacobian is exactly multiplication by its eigenvalue. This reduces the matrix update to a scalar calculation for each component, without changing the dynamics. For a negative eigenvalue $-r$, its magnitude $r>0$ is the local contraction rate, so the continuous drift brings nearby particles closer along that direction. Taking $\delta$ along this eigenvector gives $J_t\delta=-r\delta$. Euler adds $h$ times this rate of change to the current separation, giving
\[
  \delta_{\mathrm{next}}
  =\delta+hJ_t\delta
  =\delta-hr\delta
  =(1-hr)\delta.
\]
When $hr>2$, the multiplier $1-hr$ is less than $-1$. Its negative sign means that the particles exchange their ordering along this direction, while its absolute value exceeding one means that their distance increases. For example, $hr=3$ gives $\delta_{\mathrm{next}}=-2\delta$, so the particles overshoot each other and end twice as far apart. This is numerical amplification in a direction where the continuous drift contracts. Checking all negative eigenvalues identifies the directions most sensitive to this overshoot.

We can summarize this sensitivity by asking how close each update comes to its local contraction limit. Along a predictor direction with contraction rate $r$, this limit is $2/r$, so the ratio $h/(2/r)=hr/2$ reaches one at the boundary between contraction and amplification. The same reasoning extends to a ULA corrector of duration $h$, whose contraction limit along a direction with variance $v$ is $4v/c(t)$, as in \cref{eq:ula-gaussian-contraction}. The corresponding ratio $hc(t)/(4v)$ is largest in the narrowest direction of the Gaussian, where the covariance eigenvalue is smallest. To capture the most restrictive update encountered during sampling, we define $R(n)$ as the largest of these ratios across both update types, their contracting directions, and all sampling steps for $n$ observations. With one correction per predictor step ($L=1$) on the grid $t_k=kh$, the predictor uses time $t_k$ and the subsequent corrector uses $t_{k-1}$, giving
\begin{equation}
  R(n)=\max_{1\leq k\leq400}\left\{
    \underbrace{\frac{h}{2}\max_i[-\lambda_i(J_{t_k})]_+}_{\text{predictor}},\;
    \underbrace{\frac{h\,c(t_{k-1})}{4\lambda_{\min}(C_{t_{k-1}})}}_{\text{corrector}}
  \right\},
  \label{eq:compositional-step-sensitivity}
\end{equation}
where $[a]_+=\max(a,0)$ and $\lambda_i$ denotes an eigenvalue. We exclude positive predictor eigenvalues because they describe expansion already present in the continuous drift, rather than amplification introduced by discretization. Panel (f) of \cref{fig:pooled-sde-comparison} reports $R(n)$ for each coefficient choice, showing how the same sampling step size fares as more observations are composed. Values below one mean that every tested contracting direction remains within its local step-size limit, whereas values above one identify at least one update that amplifies separation where the continuous drift would contract. A smaller $R(n)$ therefore leaves more room below these limits, although it does not by itself imply faster mixing or better posterior accuracy. It remains a local diagnostic rather than a guarantee of global stability or a measure of accumulated sampling error.

\paragraph{Assessing sampling accuracy.} Avoiding numerical amplification does not ensure that the particles accurately represent the target distribution. For a Gaussian target, the mean and covariance together determine the distribution, so samples can be centered correctly yet misrepresent uncertainty by being too concentrated, too dispersed, or incorrectly correlated. We therefore tracked whether the particles reproduce the intended spread and correlations throughout sampling by comparing their empirical covariance with the analytic target $C_t$, using the relative covariance error $\|\widehat{\mathrm{Cov}}(\theta_t)-C_t\|_F/\|C_t\|_F$, where $\|\cdot\|_F$ is the Frobenius norm. The plotted paths show eight matched particles projected onto the leading posterior principal component, against the exact mean and pointwise 90\% band. For panels (e,f), we fixed $L=1$ and swept $n\in\{1,2,4,8,16,32,40,48,64,100\}$. Each method used 800 full compositional-score calls per run, each processing all particles. We measured empirical sliced 2-Wasserstein distance (sW) against 1,000 exact-posterior draws using 512 random unit directions, taking the square root of the mean squared difference between sorted projections. At each $n$, five paired seeds shared observations, terminal particles, predictor/corrector noise, reference draws, and projection directions across methods. Panel (e) reports the mean paired difference $\mathrm{sW}_{\beta}-\mathrm{sW}_{g_n}$, with pointwise 95\% Student-$t$ intervals using four degrees of freedom.

\paragraph{Effect of the diffusion coefficient.} The two coefficients give identical samples at $n=1$, while the derived coefficient lowers sW in all 45 paired comparisons with $n>1$. At $n=100$, mean sW falls from $0.0153$ to $0.0074$, a reduction of approximately 52\%. The absolute improvement peaks near $n=8$, rather than increasing monotonically. Ordinary VP crosses $R=1$ between the tested values $n=40$ and $n=48$; at $n=100$, $R$ is $2.49$ for ordinary VP versus $0.07$ for the derived coefficient. Finite ordinary-VP endpoints with $L=1$ can conceal transient relative covariance errors of order $10^{142}$; these runs remain in the accuracy comparison. The coefficient therefore improves finite-budget accuracy and reduces step-size sensitivity, but the accuracy gains already appear below the local instability threshold. Changing the coefficient while retaining the score and linear drift also changes compatibility with the prescribed continuous-time marginals, so numerical stiffness alone does not explain the difference. These controlled Gaussian results do not establish the same behavior for learned-score SLCP or BMP sampling.

\subsection{Exact-score Gaussian benchmark}
\label{sec:gaussian-benchmark-details}

We evaluated the complete observation-count sweep $n\in\{1,2,4,8,16,32,64,100\}$ using the exact compositional score and posterior of the 10-dimensional correlated Gaussian model defined in \cref{sec:compositional-coefficient-evaluation}. For every $n$, each method used five common data seeds, 1,000 posterior particles, and $T=400$ reverse steps. The F-NPSE-SDE comparisons used $L\in\{1,5\}$ Langevin corrections per predictor. This paired design separates changes caused by the compositional sampler from variation in observations or reference samples.


\subsection{SLCP benchmark}
\label{sec:slcp-benchmark-details}

We evaluated $n\in\{1,14,30\}$ across three independently trained single-observation score networks and 25 common SLCP test instances per network seed. Each network was trained on simulator-generated SLCP pairs. Within each seed and test instance, all compositional samplers used the same frozen score network and reference posterior, so the comparison isolates the sampling method. Following \citet{linhart2026tall}, JAC and Langevin bounded standardized iterates coordinatewise to $[-3,3]$ for numerical stability. We report sliced Wasserstein distance, MMD, and C2ST using the same definitions as in \cref{tab:fnpse-sde-benchmarks}.


\subsection{Computational cost of compositional samplers}
\label{sec:compositional-cost-details}

Consider sampling $P$ posterior particles, or $P$ simultaneous reverse-diffusion chains, over $d$ unknown parameters conditioned on $n$ observations using $T$ reverse steps. Let $C_s$ denote the architecture-dependent cost of one unbatched score-network evaluation for one particle-observation pair at one diffusion time. \textsc{F-NPSE-SDE} evaluates $n$ scores for every particle in each predictor and repeats these evaluations in each of $L$ ULA corrector steps. Its compute cost is therefore $\mathcal O(T(L+1)PnC_s)$. Because scores can be accumulated over observations and ULA steps run sequentially, its working state requires $\mathcal O(Pd)$ memory beyond the score network's minibatch activations. Thus, increasing $L$ increases runtime but not the asymptotic peak state memory.

JAC forms a full $d\times d$ score Jacobian for every particle-observation pair at every reverse step. Computing these Jacobians requires roughly $\mathcal O(TPndC_s)$ work, corresponding to approximately $d$ reverse-mode passes per score, in addition to dense covariance inversions and solves. A fully vectorized implementation stores $\mathcal O(Pnd^2)$ Jacobian values, an $nd$-fold increase over the $\mathcal O(Pd)$ particle state of \textsc{F-NPSE-SDE}. For pooled BMP inference, $n=940$ and $d=65$, so this auxiliary state is approximately $6.1\times10^4$ times larger. Streaming over observations can instead reduce peak Jacobian storage to $\mathcal O(Pd^2)$, but serializes the same computation and retains the additional autodifferentiation and dense linear-algebra costs.

GAUSS avoids online Jacobians but estimates a covariance for each observation from $M$ auxiliary samples generated by $S$-step DDIM rollouts. For fixed conditioning, its preprocessing cost is $\mathcal O(nMSC_s)$. A straightforward batched implementation uses $\mathcal O(nMd+nd^2)$ storage for auxiliary samples and per-observation covariances. Online covariance accumulation can remove the $nMd$ term, and processing contexts sequentially can further reduce peak storage, but each context still requires the rollout computation and estimation of a $d\times d$ covariance. Sampling then requires $\mathcal O(TPnC_s)$ score evaluations and dense precision operations.

These costs are amplified by the conditional block updates in HBDS. Each reverse-time level updates four five-dimensional receptor blocks conditional on the current 60-dimensional global state, then updates the global state conditional on the refined receptors. JAC consequently requires four sets of local Jacobians and a global $60\times60$ Jacobian family. If fully vectorized, the global family alone is approximately $n(60)^2/(60+4\times5)\approx4.2\times10^4$ times larger than the complete \textsc{F-NPSE-SDE} HBDS particle state. Streaming reduces peak memory but repeats this Jacobian computation serially within every block update. For GAUSS, the relevant conditional covariances change with the current global and receptor states, so its DDIM rollouts become nested inside the block refinements rather than being amortized once. By contrast, \textsc{F-NPSE-SDE} stores only the HBDS particle state, $\mathcal O(P(60+4\times5))$, beyond common score-network activations. HBDS therefore adds sequential score evaluations for \textsc{F-NPSE-SDE} without requiring these Jacobian or covariance calculations.

\paragraph{BMP runtime and memory.} In our BMP experiments, JAC increased HBDS sampling from approximately one hour to more than ten hours, while GAUSS exhausted GPU memory during context-specific covariance estimation across 940 designs. We therefore omit both from the subsequent BMP evaluation.

\subsection{Compositional sampling related works}
\label{sec:compositional-related-work}

\paragraph{Factorized posterior sampling.} Our construction builds on F-NPSE \citep{geffner23a}, which combines single-observation posterior scores with a prior correction for inference from multiple observations. We retain this factorization and use its discrete transition to motivate a continuous-time Gaussian surrogate. The contribution is the resulting variance prescription and compatible diffusion coefficient, rather than score composition itself; \citet{arruda2025diffusiontutorial} review this broader family of SBI methods.

\paragraph{Inference-time schedule adjustments.} Our scalar approximation is closely related to the sampling-time modifications of \citet{arruda2026compositional}, who stabilize composition through score damping and an adjustable constant shift of the log signal-to-noise ratio (SNR), alongside minibatch score estimation to reduce memory requirements. For our Gaussian surrogate, define $\mathrm{SNR}_n(t)=\alpha(t)/V_n(t)$ and $\mathrm{SNR}_{\mathrm{VP}}(t)=\alpha(t)/(1-\alpha(t))$ at times with $0<\alpha(t)<1$. The prescribed variance in \cref{eq:compositional-cumulative-variance} gives
\begin{equation*}
\log\mathrm{SNR}_n(t)
=
\log\mathrm{SNR}_{\mathrm{VP}}(t)
+\log\kappa^{(n)}(t).
\end{equation*}
Thus, $\kappa^{(n)}(t)$ induces a time- and observation-count-dependent log-SNR shift. We obtain this adjustment by prescribing a composition-dependent cumulative variance and deriving its compatible diffusion coefficient while retaining the ordinary VP mean attenuation. Since the mean attenuation is unchanged while the variance is reduced, this is not simply a reparameterization of the ordinary VP schedule, nor is it equivalent to multiplying the composed score by a damping factor. The derivation realizes the chosen Gaussian surrogate but does not establish an exact reverse process for the full compositional posterior.

\paragraph{Covariance-aware sampling.} The samplers of \citet{linhart2026tall} address errors in composing diffusion-time scores through score Jacobians (JAC) or auxiliary posterior covariance estimates (GAUSS). We compare these methods directly on Gaussian and SLCP benchmarks in \cref{sec:validate_comp_diffusion}; their additional computations become more expensive when repeated inside HBDS block updates. In our BMP implementation, JAC increased HBDS sampling from approximately one hour to more than ten hours, while GAUSS exhausted GPU memory during covariance estimation across 940 designs (\cref{sec:compositional-cost-details}). These observations concern the tested implementation, not an impossibility of adapting either method. Our score-only compositional sampler makes the repeated block updates practical in this setting. The hierarchical inference construction is discussed separately in \cref{sec:hbds-related-work}.

\section{Hierarchical blockwise diffusion sampling details}
\label{sec:hbda_details}

\paragraph{Grouped posterior target.}
Suppose there are $K$ experimental groups, each with design $\xi^k$, observation $y^k$, and local latent state $R^k$, together with global parameters $\theta$. Assuming condition-specific independent receptor priors, as in our BMP model, the grouped posterior factorizes as
\begin{equation}
p(\theta,R^{1:K}\mid y^{1:K},\xi^{1:K})
\propto
p(\theta)\prod_{k=1}^{K}p(R^k\mid\xi^k)\,p(y^k\mid R^k,\theta,\xi^k).
\label{eq:hbds-grouped-posterior}
\end{equation}
A pooled model restricts this family by tying $R^1=\cdots=R^K$. If two groups require different local states on a set with positive posterior mass, that restriction cannot represent the grouped posterior; \cref{app:em_simulator} gives an analytic example.

\paragraph{Blockwise intermediate target.} HBDS uses conditional scores to approximate the unavailable noised joint posterior $p_t(\theta,R^{1:K}\mid y^{1:K},\xi^{1:K})$. Over each diffusion interval, it advances and refines the local receptor blocks before updating the shared parameters using all groups. Additional block rounds repeat this interval after re-noising the updated states, rather than adding another reverse transition after the predictor--corrector updates.

\begin{algorithm}[!t]
  \caption{Hierarchical Blockwise Diffusion Sampling (HBDS)}
  \label{alg:dem-a2e}
  \DontPrintSemicolon
  \LinesNotNumbered

  \KwIn{Groups $\{(\xi^k,y^k)\}_{k=1}^K$; score model $s_\phi$; sampling schedule $(f,g_n)$; time grid $t_T>\cdots>t_0=0$; block rounds $M\geq1$; corrector steps $L\geq0$.}
  $\mathrm{PC}_{t\to s}$ combines one predictor step $t\to s$ with $L$ Langevin corrections at $s$.\;
  \KwInit{$\theta_{t_T}\sim q_{\mathrm{init}}^\theta$; $R_{t_T}^k\sim q_{\mathrm{init}}^{R^k}$ for $k=1,\dots,K$.}

  \For{$i=T,T-1,\dots,1$}{
    $(t,s)\leftarrow(t_i,t_{i-1})$\;
    \For{$m=1,\dots,M$}{
      \For{$k=1$ \KwTo $K$}{
        $R_s^k\leftarrow\mathrm{PC}_{t\to s}^{R}\!\left(R_t^k\mid\theta_t,\xi^k,y^k\right)$\;
      }
      $\mathcal G_s\leftarrow\{(R_s^k,\xi^k,y^k)\}_{k=1}^K$\;
      $\theta_s\leftarrow\mathrm{PC}_{t\to s}^{\theta}\!\left(\theta_t\mid\mathcal G_s\right)$\;
      \If{$m<M$}{
        Re-noise $(\theta_s,R_s^{1:K})$ to obtain $(\theta_t,R_t^{1:K})$ at noise level $t$\;
      }
    }
  }

  \KwOut{$(\theta_0,\{R_0^k\}_{k=1}^K)$ conditioned on $\{(\xi^k,y^k)\}_{k=1}^K$.}
\end{algorithm}

\Cref{alg:dem-a2e} uses $\mathrm{PC}_{t\to s}^{B}$ to denote a predictor--corrector update that advances block $B$ from $t$ to $s$ and applies $L$ Langevin corrections at $s$. In the continuous implementation, the predictor uses the probability-flow part of \cref{eq:ode_langevin} and is deterministic. Each correction follows \cref{eq:ula-corrector} with noise covariance $\eta g_{n_B}(s)^2I$, where $\eta=(t-s)/L$ for $L>0$. Here $n_B$ is the observation count for that block, namely the size of group $k$ for $R^k$ and the total count for $\theta$. Setting $L=0$ skips corrections, which are also omitted on the final diffusion interval. Thus, no additional Gaussian reverse step is needed after these updates. The initialization laws $q_{\mathrm{init}}$ use Gaussian noise for $\theta$ and, in the BMP experiments, the receptor prior in \cref{sec:hierarchical_receptors}.

\subsection{Toy simulator model example}
\label{app:em_simulator}
\label{sec:hbds-toy}

We consider a hierarchical model with a discrete latent \(k\in\{-1,+1\}\), a continuous latent \(\theta\in\mathbb{R}\), and an observation \(x\in\mathbb{R}\). The generative process is
\begin{align*}
k &\sim \mathrm{Rad}\!\left(\tfrac12\right),
\qquad \text{(Rademacher: }\mathbb{P}[k=+1]=\mathbb{P}[k=-1]=\tfrac12\text{)}, \\
\theta &\sim \mathcal{N}(0,1), \\
x | \theta,k &\sim \mathcal{N}\!\big(5k+\theta,\,1\big),
\quad\text{i.e., } x=5k+\theta+\varepsilon,\;\varepsilon\sim\mathcal{N}(0,1)\ \text{indep.} 
\end{align*}

The joint density factorizes as
\begin{equation*}
p(\theta,k,x)
= p(k)p(\theta)p(x | \theta,k)
= \tfrac12\,\mathcal{N}(\theta;0,1)\,\mathcal{N}\!\big(x;5k+\theta,1\big).
\end{equation*}

Thus, the likelihood is \(p(x | \theta,k)=\mathcal{N}(x;5k+\theta,1)\), with prior \(p(\theta)=\mathcal{N}(0,1)\) and \(p(k)=\tfrac12\). The corresponding posterior distributions are analytically tractable. The discrete latent variable has
\begin{align}
p(k=+1 | x)
  &= \frac{\mathcal{N}(x;5,\,2)}
         {\mathcal{N}(x;5,\,2)+\mathcal{N}(x;-5,\,2)},
\label{eq:post-k}\\
p(k=-1 | x)&=1-p(k=+1 | x),
\label{eq:post-k-negative}
\end{align}
while the continuous latent variable conditioned on \(x\) and \(k\) follows
\begin{equation}
p(\theta | x,k)
  = \mathcal{N}\!\big(\theta;\tfrac12(x-5k),\,\tfrac12\big).
\label{eq:post-theta}
\end{equation}
Together these define the exact posterior for a single observation,
\begin{equation*}
p(\theta,k | x)
  = p(\theta | x,k)\,p(k | x),
\end{equation*}
which serves as the analytic ground truth for single-sample posterior evaluation.

\paragraph{Pooled versus grouped posterior structure.}
\begin{figure}[!t]
\centering
\includegraphics[width=\linewidth]{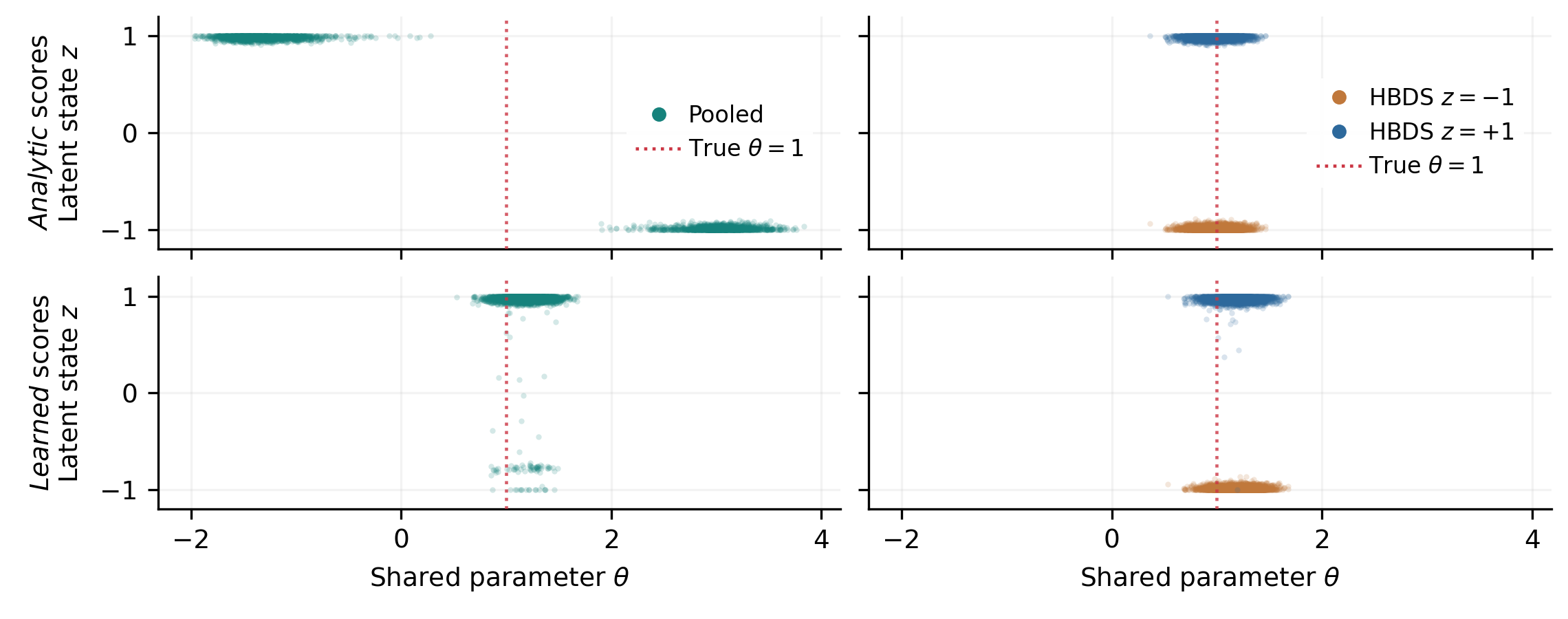}
\caption{Joint-state diagnostic for the hierarchical toy with 20 observations across two groups and no ULA correctors. (\emph{Left}) Pooled sampling with one shared latent. (\emph{Right}) HBDS with separate group latents. Rows use analytic (top) and learned (bottom) scores; dotted lines mark the generating $\theta=1$. This oracle diagnostic conditions HBDS latent updates on the true $\theta$ and parameter updates on the true group signs; pooled sampling receives no such information.}
\label{fig:hbds-toy-joint-oracle}
\end{figure}

For $K$ experimental groups, we introduce one group-specific discrete latent $z_g\in\{-1,+1\}$ per group while sharing the continuous parameter $\theta$:
\[
z_g\sim \mathrm{Rad}\!\left(\tfrac12\right),
\qquad
\theta\sim\mathcal N(0,1),
\qquad
x_g\mid \theta,z_g\sim \mathcal N(5z_g+\theta,1).
\]
The hierarchical posterior is
\begin{equation}
p(\theta,z_{1:K}\mid x_{1:K})
\propto
p(\theta)\prod_{g=1}^K p(z_g)p(x_g\mid \theta,z_g),
\label{eq:hbds-toy-grouped-posterior}
\end{equation}
which allows each group to select its own latent mode. A pooled model instead ties all groups to a single latent $z$,
\begin{equation}
p_{\mathrm{pool}}(\theta,z\mid x_{1:K})
\propto
p(\theta)p(z)\prod_{g=1}^K p(x_g\mid \theta,z).
\label{eq:hbds-toy-pooled-posterior}
\end{equation}
Thus, whenever different groups are better explained by different latent modes, the hierarchical posterior can represent the oracle structure while the pooled posterior cannot. For example, observations near $+5$ and $-5$ are naturally explained by different group-specific latents with $\theta\approx 0$, whereas the pooled model must force both groups to share the same latent explanation.

\paragraph{Learned-score illustration.} The main-text comparison in \cref{fig:hbds-toy} used 20 observations, ten per group, generated with $\theta=1$, signs $z_1=-1$ and $z_2=+1$, and unit observation-noise standard deviation. Both samplers reused a frozen single-observation score network, with 1,000 particles per sampling seed across seeds 0, 1, and 2, 200 predictor intervals, and no ULA correctors. HBDS conditioned its block updates on cumulative-VP Tweedie estimates, without access to the generating parameters or signs. The plot shows all saved paired particles without jitter or latent rounding; continuous latent states are sampler approximations to the discrete target.

\subsection{Hierarchical SBI related works}
\label{sec:hbds-related-work}

\paragraph{Hierarchical posterior estimators.} Hierarchical SBI explicitly distinguishes shared parameters from local latent variables. \citet{heinrich2024hierarchical} learn dataset-wide representations of event ensembles and separate global and local posterior estimators, supporting varying ensemble sizes. \citet{habermann2025multilevel} introduce multilevel neural posterior estimation (ML-NPE), pairing hierarchical summary networks with global and local posterior estimators. Their framework supports covariates, shared parameters, and varying numbers of groups and observations per group within the training range, with validation through simulation-based calibration and comparisons to Stan. \citet{charles2026tokenised} combine group-aware tokenization with flow matching, using a per-site likelihood surrogate to generate synthetic multi-site datasets for training a separate hierarchical posterior estimator. These approaches build the hierarchy into the learned inference model.

\paragraph{Hierarchical score models.} \citet{arruda2026compositional} train global and local score estimators on individual groups, compose the global scores across groups, and then sample local parameters conditional on the sampled globals. This supports hierarchical inference without simulating the full hierarchy for each training example. HBDS instead reuses the conditional queries of one joint score model and alternates local and global updates within the reverse-diffusion process, rather than training separate estimators for the hierarchical posterior factorization. The distinction is conditional-model reuse and blockwise sampling, not hierarchical composition itself. Their compositional sampling modifications are discussed in \cref{sec:compositional-related-work}.

\paragraph{Nested, design-dependent observations.} In BMP, each of four cell lines has one receptor vector shared across 235 distinct ligand-dose observations, while biophysical parameters are shared across all cell lines (\cref{sec:BMP_model}). Assigning an independent local latent to every scalar observation would discard this within-cell-line sharing. Conversely, treating a cell line as one group requires representing its panel of design--response pairs within the group-level inference model. ML-NPE already accommodates nested observations and changing dataset sizes \citep{habermann2025multilevel}; HBDS instead performs the aggregation through score composition at sampling time. It retains the single-observation representation and composes scores within receptor blocks and across groups for the shared parameters. Observation subsets can therefore change without training another hierarchical estimator, provided the underlying model has learned the required conditionals. The same sampler can also reuse the joint model adapted through likelihood fine-tuning (\cref{sec:fine-tune}).

\section{BMP signaling pathway mathematical model}
\label{sec:BMP_model}

The BMP dataset contains $n=940$ steady-state cell response measurements from $C=4$ cell lines, each profiled under 235 distinct ligand-dose conditions. We write $y=(y_1,\dots,y_n)$ for these observations. The parent and three receptor-knockdown lines form four groups, each with a distinct receptor state $R^k\in\mathbb{R}^5$ shared across its ligand-dose experiments. These groups share the simulator's biochemical affinities and catalytic efficiencies, collected in $\theta\in\mathbb{R}^{60}$, giving the hierarchical structure used by HBDS in \cref{sec:em_diffusion}. Each single-observation latent state $(\theta,R^k)$ therefore contains 65 quantities.

BMPs are homodimeric and heterodimeric ligands that transmit information by assembling signaling complexes with two classes of serine/threonine‐kinase receptors: type~I (\(A_i\)) and type~II (\(B_k\)). In a parsimonious “one-step’’ mass–action kinetics model \citep{su2022ligand}, each ligand \(L_j\) binds one receptor of each class in a single equilibrium step
\begin{equation*}
    A_i + B_k + L_j
    \;\xrightleftharpoons{K_{ijk}}\;
    T_{ijk},
    \quad
    \varepsilon_{ijk}\,T_{ijk}\;\longrightarrow\; S,
\end{equation*}
where \(K_{ijk}\) is an effective affinity for forming the active trimer \(T_{ijk}\) and \(\varepsilon_{ijk}\) is its catalytic efficiency in phosphorylating SMAD1/5/8, producing a transcriptional signal \(S\). Interestingly, despite treating the 10 known ligands and \(4\times3=12\) canonical receptors as effectively promiscuous, the model recovers dose–response data and explains how receptor context rewires downstream transcriptional programs during development or oncogenesis \citep{antebi2017combinatorial, su2022ligand}. Here, we model a subset of the known ligands and canonical receptors where \(I=3\), \(J=5\), and \(K=2\) denote the numbers of type‑I receptors, ligands, and type‑II receptors, respectively, so that \(i=1,\dots,I\), \(j=1,\dots,J\), and \(k=1,\dots,K\). \Cref{tab:bmp_variables} summarizes how the SBI variables map onto this BMP model. The receptor concentrations \(A\) and \(B\) can also be treated as fixed design variables, as in perturbation-modeling settings where biological covariates are specified interventions \citep{dimitrov2026interpretation}, giving \( (L_j , A_i , B_k) \in \Xi^{J+I+K}\) and \((\varepsilon_{ijk} , K_{ijk}) \in\Theta_{>0}^{2IJK}\), but since receptor measurements are noisy, we treat them as latent variables throughout the experiments.

\begin{table}[htbp]
\centering
\caption{BMP-to-SBI mapping and dimensions.}
\label{tab:bmp_variables}
\begingroup
\renewcommand{\arraystretch}{1.15}
\small
\begin{tabular}{@{}L{0.12\textwidth}L{0.53\textwidth}L{0.27\textwidth}@{}}
\toprule
\textbf{SBI symbol} & \textbf{BMP meaning} & \textbf{Dimension} \\
\midrule
\(\xi\) & Experimental design: ligand doses \(L_{1:J}\) & \(J=5\) \\
\(\theta\) & Shared affinities \(K_{ijk}\) and efficiencies \(\varepsilon_{ijk}\) & \(2IJK=60\) \\
\(R^k\) & Cell-line receptor state \(A_{1:I},B_{1:K}\) & \(I+K=5\) \\
\(y\) & Observation: steady-state SMAD response \(S\) & \(1\) \\
\bottomrule
\end{tabular}
\endgroup
\end{table}

\section{Cell receptor hierarchical modeling}
\label{sec:hierarchical_receptors}

We implement diffusion guidance by constructing a probabilistic prior over receptor expression informed by quantitative polymerase chain reaction (qPCR) measurements and then using the resulting (marginal) log-density as an energy function for guidance during sampling. The goal is to bias samples toward receptor configurations consistent with known KD conditions while explicitly accounting for the substantial uncertainty inherent to qPCR.

\paragraph{Indices and observations.} Let $s \in \{0,1,2,3\}$ index the four experimental cell-line conditions (one wild-type and three KD conditions). Let $j \in \{0,1,2,3,4\}$ index the five BMP-pathway receptors considered in this work (ACVR1, BMPR1A, ACVR2A, ACVR2B, BMPR2). For each condition $s$ and receptor $j$ we observe a positive relative qPCR readout $x_{s,j} > 0$, interpreted as a noisy proxy for receptor abundance (relative mRNA expression). Because qPCR variability is multiplicative, we model observations in log space, $y_{s,j} \triangleq \log x_{s,j}$.

\paragraph{Knockdown targets.} For each condition $s$, the KD target index is $k_s \in \{-1,0,1,2,3,4\}$, where $k_s=-1$ denotes no knockdown (wild-type) and $k_s=j$ denotes that receptor $j$ is knocked down in condition $s$. We define an indicator
\begin{equation*}
\mathbb{I}_{s,j} \triangleq \mathbb{I}[\,k_s=j\,],
\end{equation*}
which is $1$ iff receptor $j$ is the KD target in condition $s$ and $0$ otherwise.

\paragraph{Hierarchical model (baseline receptor distribution).} We posit receptor-specific baseline log-means and log-standard-deviations that are partially pooled across receptors via global hyperparameters. Let $\mu_j$ denote the baseline log-mean for receptor $j$, and let $\sigma_j$ denote the baseline log-standard-deviation for receptor $j$. We introduce global hyperparameters $(\mu_0, \kappa)$ controlling the population distribution of receptor means, and $(\ell_0, \eta)$ controlling the population distribution of receptor log-variances. In standard hierarchical form,
\begin{align*}
\mu_0 &\sim \mathcal{N}(m_0, v_0), \\
\kappa &\sim \mathrm{HalfNormal}(\lambda_\kappa), \\
\mu_j \mid \mu_0, \kappa &\sim \mathcal{N}(\mu_0, \kappa^2), \qquad j=0,\dots,4, \\
\ell_0 &\sim \mathcal{N}(m_\ell, v_\ell), \\
\eta &\sim \mathrm{HalfNormal}(\lambda_\eta), \\
\log \sigma_j \mid \ell_0, \eta &\sim \mathcal{N}(\ell_0, \eta^2), \qquad j=0,\dots,4.
\end{align*}
This hierarchy induces partial pooling: when data are weak or noisy, $\{\mu_j,\sigma_j\}$ are shrunk toward the global centers $(\mu_0,\ell_0)$, while still allowing receptor-specific deviations.

\paragraph{Condition-level variability.} To account for inter-condition variability (e.g., biological differences between cell lines and experimental perturbations not explained by the KD), we add a condition-specific random effect $\delta_{s,j}$ for each receptor:
\begin{align*}
\tau &\sim \mathrm{HalfNormal}(\lambda_\tau), \\
\delta_{s,j} \mid \tau &\sim \mathcal{N}(0,\tau^2), \qquad s=0,\dots,3,\;\; j=0,\dots,4.
\end{align*}
This term captures the fact that receptor expression can vary appreciably across conditions even in the absence of knockdown.

\paragraph{Knockdown effect model (multiplicative fold-change).} Knockdown acts multiplicatively in the original receptor domain, which corresponds to an additive shift in log space. We therefore model a KD fold-change via a log-shift $\nu_j \le 0$ and additional KD uncertainty $\omega_j \ge 0$:
\begin{align*}
\nu_j &\sim \mathcal{N}(m_{\nu,j}, v_{\nu,j}) \;\;\text{with support on } (-\infty,0], \\
\omega_j &\sim \mathrm{HalfNormal}(\lambda_{\omega,j}),
\end{align*}
and incorporate KD by modifying the conditional distribution of $y_{s,j}$. Specifically, the KD target receptor ($\mathbb{I}_{s,j}=1$) receives a negative mean shift and an additional variance term:
\begin{equation*}
y_{s,j} \mid \mu_j,\sigma_j,\delta_{s,j}, \nu_j,\omega_j
\;\sim\;
\mathcal{N}\!\Big(\mu_j + \delta_{s,j} + \mathbb{I}_{s,j}\,\nu_j,\;\;
\sigma_j^2 + \mathbb{I}_{s,j}\,\omega_j^2\Big).
\label{eq:kd_normal_log}
\end{equation*}
Non-target receptors ($\mathbb{I}_{s,j}=0$) follow the baseline distribution.

\paragraph{Observation model in the original receptor domain.} Let $r_{s,j} \triangleq \exp(y_{s,j})$ denote the receptor expression in the original (positive) domain. In our implementation we further impose a physically plausible upper bound $h>0$ and use a truncated LogNormal likelihood with support $r_{s,j} \in (0,h)$:
\begin{equation*}
r_{s,j} \mid \mu_j,\sigma_j,\delta_{s,j},\nu_j,\omega_j
\sim
\mathrm{TruncLogNormal}\!\Big(\mu_j + \delta_{s,j} + \mathbb{I}_{s,j}\,\nu_j,\;
\sqrt{\sigma_j^2 + \mathbb{I}_{s,j}\,\omega_j^2};\; h\Big).
\end{equation*}
Equivalently, this is a LogNormal distribution in $r_{s,j}$ with truncation at $h$, which provides numerical stability and reflects the bounded receptor regime used in simulation.

\begin{figure*}[!t]
  \centering
  \includegraphics[width=\textwidth]{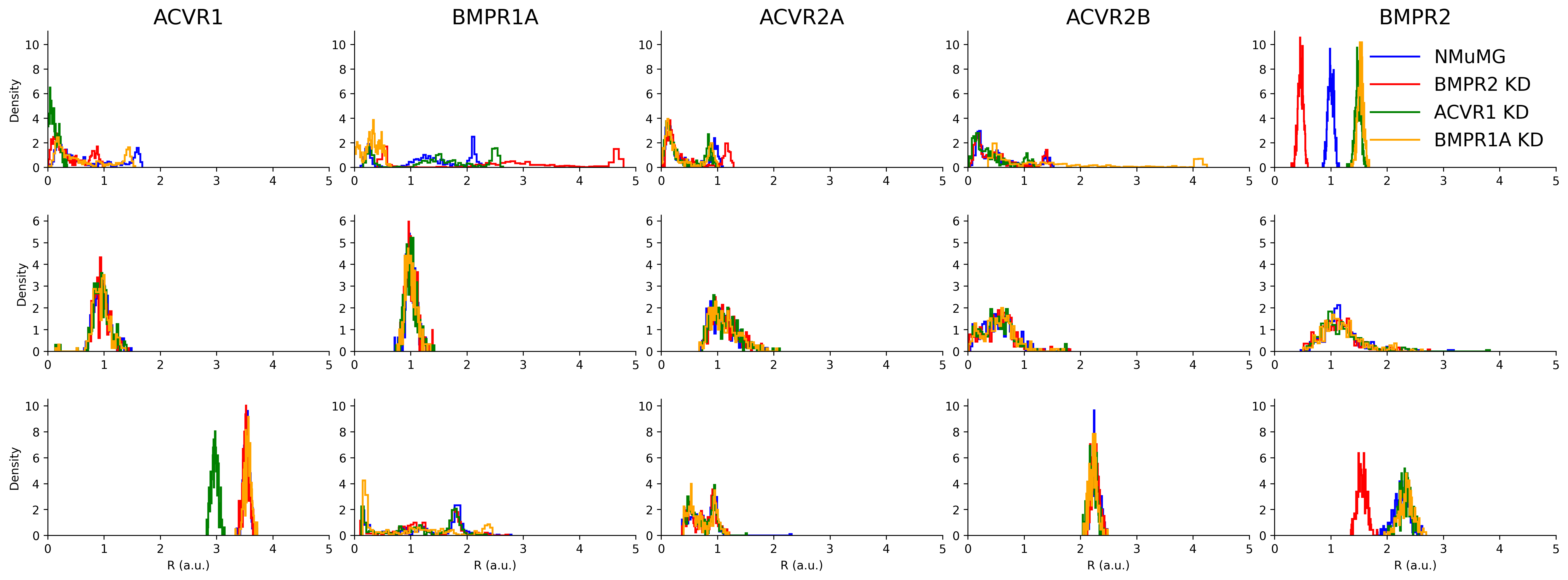}
  \vspace{-15pt}
  \caption{Comparison between BMP receptor posterior samples (columns) across LSR best-fit $\theta^{\text{LSR}}$ (\emph{top}), applying a truncated lognormal prior to the BMP receptors' distributions (\emph{middle}), and applying the marginal hierarchical receptor priors \cref{sec:hierarchical_receptors} (\emph{bottom}). Model identifiability in the $\theta^{\text{LSR}}$ parameter set  can be seen in the multi-modal distributions of the $\theta^{\text{LSR}}$ fits as well as the case of BMPR2, where $\theta^{\text{LSR}}$ only works if the ACVR1 KD and BMPR1A KD conditions are 1.5$\times$ the amount of the \emph{parent} cell line. Imposing a diffusion prior on the receptors in the BMP simformer avoids this clash of parameters. }
  \label{fig:three_tier_r}
  \vspace{-6pt}
\end{figure*}

\paragraph{Using the model for guidance.} The hierarchical model defines a receptor prior density $p(r \mid s)$ for each condition $s$. During diffusion sampling, we compute a denoised receptor estimate and apply guidance by adding the gradient of the (marginal) log-prior,
\begin{equation*}
g(r,s) \;\triangleq\; \nabla_{r}\log p(r \mid s),
\end{equation*}
scaled by a time-dependent factor (analogous to classifier-free guidance when used as a prior). In practice, marginalization over latent variables (e.g., $\mu_0,\kappa,\ell_0,\eta,\tau$) can be performed either by a tractable approximation (e.g., moment-matching) or by mixture density networks \citep{bishop1994mixture}; both yield a differentiable log-density suitable for guidance.

\section{Additional BMP experimental results}

\subsection{Calibrated receptor diffusion}
\label{sec:cal_receptors}

The LSR receptor estimates in \cref{fig:three_tier_r} illustrate potential identifiability issues. Some knockdown (KD) conditions have higher fitted basal receptor levels than the parent cell line, contrary to the intended knockdown. Low predictive error alone therefore does not establish physical plausibility. The receptor priors in \cref{sec:hierarchical_receptors} allow posterior sampling to incorporate these biological constraints.

\paragraph{LSR fit and receptor sensitivity.}
\label{sec:lsr_bench}
We compared FT posterior-predictive fit with the historical least-squares regression estimate $\theta^{\mathrm{LSR}}$ and a sensitivity analysis in which its ACVR1 estimate was reduced by $10\%$ (\cref{tab:rmse_three_methods_transposed}). The FT model in this comparison has three transformer layers and a feed-forward widening factor of two, with $\lambda=5\times10^{-4}$. Its all-data row used HBDS on 940 observations, whereas each cell-line row used an independent pooled posterior conditioned on that cell's 235 observations. These rows are not subsets of a single all-data posterior. Each run used 250 draws, 100 diffusion steps, five Langevin corrections per step, and sampling seed 42. The all-data FT RMSE is $0.26$. The historical LSR errors were retained without recomputation, and their all-data value is not the root-mean-square of the four cell-wise values, so they should not be treated as a consistently aggregated baseline until the original evaluation protocol is reconciled. The receptor perturbation illustrates local fit sensitivity.


\subsection{BMP posterior-sampling ablations}
\label{sec:bmp-sampling-ablations}

We compared pooled sampling and HBDS on 940 BMP observations, with 235 observations from each of four cell lines. Pooled sampling infers the 60 shared biophysical parameters and one shared five-dimensional receptor vector, whereas HBDS maintains a separate receptor state for each cell line. Both samplers used the derived diffusion coefficient, 100 diffusion steps, five Langevin corrections per step, 250 posterior draws, and sampling seed 42. HBDS used three block-refinement rounds with Tweedie conditioning for both blocks. Pooled sampling used three resampling passes in one sequential round, which are not hierarchical block updates. These are single-seed results, with no between-seed uncertainty estimates.

\paragraph{Models and predictive metrics.} The PT and FT models in \cref{tab:bmp_unified_summary,tab:lambda_sweep} have four transformer layers and a feed-forward widening factor of four. FT started from the matched PT model and used $\lambda=5\times10^{-4}$ for the sampler comparison. The no-FiLM model is a separately pretrained and fine-tuned architectural ablation that retains mask-state embeddings, not the same model with FiLM switched off at sampling time. It uses the same depth and width as the FiLM model. For posterior draw $k$, let $d_k=\|\mathbf y_k-\mathbf y_o\|_2$, where $\mathbf y_k$ is evaluated through the original mechanistic simulator over $N$ observations. We report $\mathrm{RMSE}=\min_k d_k/\sqrt{N}$ and median distance $\operatorname{median}_k d_k$. Thus, RMSE measures the best predictive draw rather than the posterior mean, while median distance summarizes response-space error across draws.

\paragraph{Calibration.}
\label{sec:bmp-calibration-status}
We report the logged local classifier two-sample test ($\ell$-C2ST) \citep{linhart2023c2st} statistics separately for parameters and receptors. The tests used an MLP with three folds and two ensemble members, classifier seed zero, and 100 permutation-based null trials. These are single-run statistics, not averages across independent sampling seeds. 

\begin{table}[!t]
\centering
\begin{minipage}[t]{0.42\linewidth}
\centering
\captionsetup{font=small,skip=4pt}
\caption{RMSE comparison of historical LSR, ACVR1 $-10\%$ LSR ($\theta^{\mathrm{LSR}\dagger}$), and FT predictions. FT uses HBDS for all data and independent pooled inference for each cell line. Lower values are better.}
\label{tab:rmse_three_methods_transposed}
\setlength{\tabcolsep}{2pt}
\renewcommand{\arraystretch}{1.10}
\small
\begin{tabular*}{\linewidth}{@{\extracolsep{\fill}}lrrr@{}}
\toprule
\textbf{Condition} & $\theta^{\mathrm{LSR}}$ & $\theta^{\mathrm{LSR}\dagger}$ & \textbf{FT} \\
\midrule
All data   & 0.08 & 1.02 & 0.26 \\
NMuMG      & 0.08 & 0.08 & 0.48 \\
BMPR1A KD  & 0.07 & 0.06 & 0.31 \\
BMPR2 KD   & 0.09 & 0.06 & 0.13 \\
ACVR1 KD   & 0.06 & 2.04 & 0.22 \\
\bottomrule
\end{tabular*}
\end{minipage}\hfill
\begin{minipage}[t]{0.56\linewidth}
\centering
\captionsetup{font=small,skip=4pt}
\caption{Pooled/HBDS comparison for PT, FT, and separately trained FT without FiLM ($\neg$FiLM), using 250 draws and one sampling seed.}
\label{tab:bmp_unified_summary}
\setlength{\tabcolsep}{2pt}
\renewcommand{\arraystretch}{1.10}
\small
\begin{tabular*}{\linewidth}{@{\extracolsep{\fill}}lrrrr@{}}
\toprule
& & & \multicolumn{2}{c}{$\boldsymbol{\ell}$\textbf{-C2ST}} \\
\cmidrule(lr){4-5}
\textbf{Method} & \textbf{RMSE} $\downarrow$ & \textbf{Median} $\downarrow$ & $\boldsymbol{\theta}$ & $\boldsymbol{R}$ \\
\midrule
PT Pooled & 0.58 & 1,594.55 & 0.22 & 0.06 \\
FT Pooled & 0.36 & 66.77 & 0.18 & 0.02 \\
$\neg$FiLM Pooled & 12.30 & 4,538.82 & 0.13 & 0.03 \\
\midrule
PT HBDS & 0.78 & 194.01 & 0.05 & 0.04 \\
FT HBDS & \textbf{0.34} & \textbf{14.94} & 0.02 & 0.07 \\
$\neg$FiLM HBDS & 0.91 & 274.61 & 0.09 & 0.01 \\
\bottomrule
\end{tabular*}
\end{minipage}
\par\vspace{-8pt} 
\end{table}

Fine-tuning improves both predictive-fit metrics over its matched PT model under pooled sampling and HBDS (\cref{tab:bmp_unified_summary}). For each FT architecture, HBDS gives lower errors than pooled sampling, and FT with FiLM has the lowest reported values among the completed comparisons. The three-layer results in \cref{tab:rmse_three_methods_transposed} use a different model family and are not a controlled comparison of network size.

\section{Pretraining details}
\label{sec:train_deets}

\textbf{Transformer architecture.} The simformer uses value, ID, condition, and attention dimensions of 32, with three attention heads of size 32. The model used for the cell-line and LSR comparison in \cref{tab:rmse_three_methods_transposed} has three layers and a feed-forward widening factor of two. The model/sampler and regularization ablations in \cref{tab:bmp_unified_summary,tab:lambda_sweep} instead used four layers and a widening factor of four. Rather than summing token embeddings (ID, condition, and value), we concatenate them prior to projection, which empirically produced more stable training and better downstream posterior quality.

We trained the simformer using 1,000 prior parameter draws evaluated across 940 experimental designs. For each draw $\theta\sim p(\theta)$ and design $\xi$, we generated expert-model outputs $y\sim p_{\text{eval}}(y\mid\theta,\xi)$ to form training examples $(\theta,\xi,y)$. While preliminary results suggest that increasing the number of prior parameter draws improves downstream posterior quality, we restricted this study to 1,000 draws for computational simplicity.

During training we employed two forms of masking. For \emph{edge masking}, we applied masks corresponding to marginals, the full joint, and the expert model graph, each with equal probability ($\tfrac{1}{3}$ of the time). For \emph{condition masking}, we used four different masks: (i) joint masking (noise all variables), (ii) posterior conditioning, (iii) likelihood conditioning, and (iv) random conditioning on 30\% of tokens in arbitrary groupings. These were applied with probabilities of 20\% for the joint, posterior, and likelihood masks, and 40\% for the random conditioning. We found that heavier use of posterior and likelihood masks tended to produce overconfident models, whereas increasing the joint masking improved generalization.

\paragraph{Pretraining pipeline.} We trained for 100 epochs with a batch size of $2{,}048$. Evaluating these parameter draws across all 940 designs produced approximately $1{,}000\times940=940{,}000$ simulator outputs. While this represents a large training set, the dimensionality of the parameter space ($\theta \in \mathbb{R}^{60}$) implies a trade-off between the number of prior parameter draws available and the coverage of experimental conditions. In practice, this dataset size was sufficient for demonstrating the model's capabilities within available compute, but there is room for future work studying scaling laws of the number of prior parameter draws and the number of parameters of the architecture.

\textbf{Compute resources.} All models were trained on NVIDIA A100 GPUs. Sampling and smaller ablation runs were also performed on NVIDIA A30 GPUs when memory constraints allowed.

\section{Fine-tuning and path-space details}
\label{sec:fine_tune_details}

\subsection{Fine-tuning objective and optimization}
\label{sec:fine-tuning-objective-optimization}

\paragraph{MSE reward.} We use the negative-MSE reward $R(y;y^o)=-\|y-y^o\|_2^2$ from \cref{sec:fine-tune} to measure agreement between conditional likelihood samples and observed responses. The fixed conditioning reference $\theta^{\mathrm{ref}}$ combines the biophysical parameters $\theta$ and receptor states $R$ in the BMP model. For an observation $y^o$ under design $\xi$ and $n$ Monte Carlo response samples $\{y_j\}_{j=1}^n$ drawn from the diffusion sampler at this fixed condition,
\[
\{y_j\}_{j=1}^n \sim \pi^\star(\phi;\theta^{\mathrm{ref}},\xi),
\]
we estimate the expected reward by
\[
\mathbb E_{y\sim p_\phi(\cdot\mid\theta^{\mathrm{ref}},\xi)}[R(y;y^o)]
\;\approx\;
-\frac{1}{n}\sum_{j=1}^n \bigl\|y_j - y^o\bigr\|_2^2.
\]
We aggregate these per-condition rewards across a minibatch during fine-tuning. At fixed conditioning, $\mathbb{E}\|Y-y^{\mathrm o}\|_2^2=\|\mathbb{E}Y-y^{\mathrm o}\|_2^2+\operatorname{tr}\operatorname{Cov}(Y)$, so the corresponding sample-level loss penalizes both mean prediction error and predictive variance.

\paragraph{Adjoint gradients.} The diffusion sampler induces the mapping $\pi^\star(\phi;\theta^{\mathrm{ref}},\xi)$ implicitly through the reverse-time SDE. At fixed $(\theta^{\mathrm{ref}},\xi)$, the likelihood score defines the sampling drift $\mu_\phi^y$, with dynamics integrated from $t=T$ to $t=0$,
\begin{align}
\mu_\phi^y(y,t) &\coloneqq f(y,t)-g(t)^2s_\phi^y(y,t;\theta^{\mathrm{ref}},\xi),
\label{eq:likelihood-sampling-drift}\\
dY_t &= \mu_\phi^y(Y_t,t)\,dt + g(t)\,d\overset{\scriptscriptstyle\leftarrow}{w}_t.
\label{eq:likelihood-sampling-sde}
\end{align}
Here $f$ and $g$ are the forward-SDE coefficients from \cref{sec:score-based-diffusion}. The superscript $y$ distinguishes this likelihood drift from the compositional transition mean $\mu_k^{\mathrm G}$ in \cref{eq:geffner-transition-main}. The diffusion coefficient and starting noise law are independent of $\phi$. We assume the drift is differentiable in $y$ and $\phi$, with integrable sensitivities permitting differentiation under the expectation. Holding the sampled noise realization fixed, the adjoint starts from the reward gradient at $Y_0$ and follows the same trajectory from $t=0$ to $t=T$. The system of \citet[Appendix A.3, Eq.\,14]{marion2025implicitdiffusionefficientoptimization}, expressed in this clock with column-vector gradients, becomes
\begin{align*}
A_0 &= \nabla_y R(Y_0;y^o), &
\frac{dA_t}{dt} &= -\bigl[\partial_y \mu_\phi^y(Y_t,t)\bigr]^\top A_t,\\
G^\phi_0 &= 0, &
\frac{dG^\phi_t}{dt} &= -\bigl[\partial_\phi \mu_\phi^y(Y_t,t)\bigr]^\top A_t.
\end{align*}
Here $A_t$ propagates the endpoint reward sensitivity along the sampled path, while $G^\phi_t$ accumulates the parameter gradient. The minus signs account for expressing the adjoint in the original diffusion clock. Taking the expectation of $G^\phi_T$ over paths yields the endpoint reward gradient with respect to $\phi$, which enters the minimized objective with a minus sign. To include the path penalty, we differentiate the full sampled trajectory cost,
\[
-R(Y_0;y^o)
+\frac{\lambda}{2}\int_0^T g(t)^2
\bigl\|s_\phi^y(Y_t,t)-s_{\phi_0}^y(Y_t,t)\bigr\|_2^2\,dt.
\]
Averaging this cost over FT paths gives the SOC objective in \cref{eq:soc-ft-objective}, with the integrated penalty contributing $\lambda\mathcal A(\theta^{\mathrm{ref}},\xi)$ from \cref{eq:score-control-energy}. Beyond the reward-only adjoint above, differentiation includes the integral's dependence on both the sampled states $Y_t$ and the parameters $\phi$. Our implementation uses checkpointed backpropagation through the discretized rollout to compute discrete-adjoint gradients of both the endpoint reward and the integrated penalty (\cref{sec:diffusion-fine-tuning-related-work}).

For BMP, we fixed $\theta^{\mathrm{ref}}=(\theta^{\mathrm{LSR}},R^{\mathrm{LSR}})$ at a best-fit least-squares regression (LSR) solution. We note that the choice of reference is an important hyperparameter during fine-tuning, similar to the choice of a known pair in \citet{wehenkel2025addressing}. In \cref{sec:fisher-rao-analysis}, we propose alternating Fisher--Rao-guided updates of posterior proposal particles with likelihood fine-tuning, allowing the mechanistic reference to evolve alongside the model.


\begin{table}[!t]
\centering
\caption{BMP regularization and Langevin-correction sweep. Each setting used 250 HBDS draws and one sampling seed. $L$ is the number of corrector steps per diffusion step.}
\label{tab:lambda_sweep}
\begingroup
\setlength{\tabcolsep}{4pt}
\renewcommand{\arraystretch}{1.10}
\small
\begin{tabular*}{\linewidth}{@{\extracolsep{\fill}}lrrrrrrrr@{}}
\toprule
& \multicolumn{2}{c}{$L=0$} & \multicolumn{2}{c}{$L=1$} & \multicolumn{2}{c}{$L=3$} & \multicolumn{2}{c}{$L=5$} \\
\cmidrule(lr){2-3}\cmidrule(lr){4-5}\cmidrule(lr){6-7}\cmidrule(lr){8-9}
$\boldsymbol{\lambda}$ & \textbf{RMSE} $\downarrow$ & \textbf{Median} $\downarrow$ & \textbf{RMSE} $\downarrow$ & \textbf{Median} $\downarrow$ & \textbf{RMSE} $\downarrow$ & \textbf{Median} $\downarrow$ & \textbf{RMSE} $\downarrow$ & \textbf{Median} $\downarrow$ \\
\midrule
$0$ & 0.29 & \textbf{12.16} & 0.36 & \textbf{12.96} & 0.36 & \textbf{12.97} & 0.36 & \textbf{12.97} \\
$10^{-4}$ & 0.30 & 28.15 & \textbf{0.32} & 89.55 & \textbf{0.32} & 90.63 & \textbf{0.32} & 90.87 \\
$5\times10^{-4}$ & \textbf{0.28} & 14.30 & 0.34 & 14.89 & 0.33 & 14.62 & 0.34 & 14.94 \\
$10^{-3}$ & 0.34 & 29.97 & 0.36 & 41.20 & 0.36 & 41.36 & 0.36 & 41.86 \\
$10^{-2}$ & 0.33 & 45.26 & 0.40 & 93.51 & 0.42 & 94.78 & 0.41 & 99.30 \\
$10^{-1}$ & 0.30 & 52.17 & 0.52 & 157.55 & 0.53 & 153.48 & 0.48 & 159.86 \\
$1$ & 0.30 & 47.00 & 0.40 & 130.73 & 0.41 & 134.90 & 0.41 & 139.27 \\
\bottomrule
\end{tabular*}
\endgroup
\end{table}

\paragraph{Regularization and corrector sweeps.} We compared seven regularization strengths and $L\in\{0,1,3,5\}$ Langevin corrections per diffusion step using the four-layer FiLM model (\cref{tab:lambda_sweep}). Each setting used HBDS with 100 diffusion steps, three block-refinement rounds, Tweedie conditioning for both blocks, 250 posterior draws, and sampling seed 42. All settings retained the derived diffusion coefficient, including the predictor-only $L=0$ case. At $L=5$, $\lambda=10^{-4}$ gives the lowest RMSE among positive penalties, while $\lambda=5\times10^{-4}$ gives the lowest median distance. The unregularized model has the lowest median distance overall. In this single-seed sweep, omitting corrections yields lower errors than adding them for every $\lambda$, showing that more corrector steps need not improve predictive fit at a fixed discretization. This is a correction-count ablation, not a comparison with the ordinary VP coefficient.

\subsection{Path-space formulation and Girsanov's theorem}
\label{sec:girsanov_background}

We briefly state the change-of-measure result used in \cref{sec:fine-tune,sec:path-divergence}. Let $\Omega=C([0,T];\mathbb R^d)$ denote the space of continuous trajectories. Consider two SDEs evolving from $t=0$ to $t=T$ with the same initial law and scalar diffusion coefficient,
\begin{align*}
    dX_t &= b_0(X_t,t)\,dt + g(t)\,dw_t,
    \qquad X_0\sim p_0,\\
    dX_t &= b_1(X_t,t)\,dt + g(t)\,d\widetilde w_t,
    \qquad X_0\sim p_0,
\end{align*}
and let $\mathbb{P}_0$ and $\mathbb{P}_1$ denote their induced probability measures on $\Omega$. We assume $g(t)>0$ and that solutions exist, remain finite on $[0,T]$, and have uniquely determined path laws. Here $w$ and $\widetilde w$ are Wiener processes under $\mathbb{P}_0$ and $\mathbb{P}_1$, respectively. It is convenient to express the drift perturbation in diffusion-normalized coordinates,
\[
u_t(x)
\coloneqq
\frac{b_1(x,t)-b_0(x,t)}{g(t)}.
\]
In stochastic optimal control (SOC) this normalized drift perturbation is conventionally interpreted as the control; we return to this connection in \cref{sec:soc-connection}. We further assume that $u_t$ is sufficiently integrable for the associated stochastic exponential to be a true martingale under $\mathbb P_0$ (Novikov's condition is a standard sufficient condition). This ensures that the trajectory weights have expectation one under the reference law and define a normalized probability measure. Under these assumptions, Girsanov's theorem \citep{girsanov1960transforming} expresses the change from drift $b_0$ to $b_1$ through the following Radon--Nikodym derivative, which reweights reference trajectories to obtain the modified path law,
\begin{equation}
    \frac{d\mathbb{P}_1}{d\mathbb{P}_0}(X)
    =
    \exp\!\left(
        \int_0^T u_t(X_t)^\top dw_t
        -
        \frac{1}{2}
        \int_0^T \|u_t(X_t)\|_2^2\,dt
    \right).
\label{eq:girsanov_rn}
\end{equation}
Here the stochastic integral is taken under the reference law $\mathbb{P}_0$, induced by drift $b_0$. The Radon--Nikodym derivative reweights these trajectories to recover expectations under $\mathbb{P}_1$, the law induced by the changed drift $b_1=b_0+gu$ with the same diffusion coefficient and initial distribution. Consequently, provided the control energy is integrable under $\mathbb{P}_1$,
\begin{equation}
    \mathrm{KL}\!\left(
        \mathbb{P}_1\,\middle\|\,\mathbb{P}_0
    \right)
    =
    \frac{1}{2}
    \mathbb{E}_{\mathbb{P}_1}
    \left[
        \int_0^T
        \|u_t(X_t)\|_2^2\,dt
    \right].
\label{eq:girsanov_kl}
\end{equation}

For the PT and FT reverse SDEs in our setting, the shared starting law is the noise distribution at $t=T$ and the drift difference is $\mu_\phi^y-\mu_{\phi_0}^y=-g(t)^2(s_{\phi}^{y}-s_{\phi_0}^{y})$. The normalized reverse-drift change is therefore $u_t=-g(t)(s_{\phi}^{y}-s_{\phi_0}^{y})$. Rewriting the processes in increasing sampling time reverses this sign but leaves the squared control energy unchanged, yielding \cref{eq:score-control-energy}.

\subsection{Stochastic optimal control connection}
\label{sec:soc-connection}

\paragraph{Likelihood fine-tuning as drift control.}
We fix $(\theta^{\mathrm{ref}},\xi)$ and suppress these conditioning arguments. To connect the Implicit Diffusion notation in \cref{sec:fine-tuning-objective-optimization} to the SOC notation, we write the same drifts as $b_\phi=\mu_\phi^y$ and $b_0=b_{\phi_0}=\mu_{\phi_0}^y$, where the subscript $0$ denotes the PT reference, not a time. We retain the reverse-SDE convention in which diffusion time $t$ decreases from $T$ to $0$, so sampling ends at $t=0$. Let $Y_t^{\phi_0}$ and $Y_t^\phi$ denote the PT and FT reverse-SDE trajectories, with path laws $\mathbb P_{\phi_0}$ and $\mathbb P_\phi$, respectively, and dynamics
\[
\begin{aligned}
    dY_t^{\phi_0} &= b_0(Y_t^{\phi_0},t)\,dt + g(t)\,d\overset{\scriptscriptstyle\leftarrow}{w}_t, \\
    dY_t^\phi &= \bigl[b_0(Y_t^\phi,t)+g(t)u_\phi(Y_t^\phi,t)\bigr]\,dt + g(t)\,d\overset{\scriptscriptstyle\leftarrow}{\widetilde w}_t,
\end{aligned}
\]
where
\begin{equation}
u_\phi(y,t)
:=
g(t)^{-1}\bigl[b_\phi(y,t)-b_0(y,t)\bigr]
=
-g(t)\bigl[s_\phi^y(y,t)-s_{\phi_0}^y(y,t)\bigr].
\label{eq:score-induced-control}
\end{equation}
Thus, fine-tuning changes the reverse drift from $b_0$ to $b_\phi=b_0+g u_\phi$, while preserving the diffusion coefficient $g(t)$ and the common starting noise law at $t=T$. Let $\pi_\phi=\pi^\star(\phi;\theta^{\mathrm{ref}},\xi)$ denote the corresponding endpoint distribution. Extracting the generated response $Y_0$ is a measurable projection of the full trajectory, so the data-processing inequality and the control-energy identity in \cref{eq:score-control-energy} give
\begin{equation}
\mathrm{KL}(\pi_\phi\|\pi_{\phi_0})
\leq
\mathrm{KL}(\mathbb{P}_\phi\|\mathbb{P}_{\phi_0})
=
\frac12\mathbb{E}_{\mathbb{P}_\phi}
\int_0^T\|u_\phi(Y_t,t)\|_2^2\,dt,
\label{eq:endpoint-path-kl-bound}
\end{equation}
under the assumptions in \cref{sec:girsanov_background}. Replacing the endpoint penalty in \cref{eq:fine-tune-objective} by this path penalty gives the parameterized SOC cost
\begin{equation}
\begin{aligned}
\mathcal{J}_\lambda(\phi)
&:=-\mathbb{E}_{y\sim\pi_\phi}[R(y;y^o)]
+\lambda\,\mathrm{KL}(\mathbb{P}_\phi\|\mathbb{P}_{\phi_0}) \\
&=\mathbb{E}_{\mathbb{P}_\phi}\!\left[-R(Y_0;y^o)
+\frac{\lambda}{2}\int_0^T\|u_\phi(Y_t,t)\|_2^2\,dt\right],
\end{aligned}
\label{eq:soc-ft-objective}
\end{equation}
This is KL-regularized SOC within the score-network parameterization, with observation loss $-R$ as the sampling-endpoint cost and $\lambda$ pricing control effort \citep{uehara2024entropycontrol}. Since the expected reward is unchanged and $\lambda\geq0$, data processing gives $\mathcal{F}(\pi^\star(\phi))\leq\mathcal{J}_\lambda(\phi)$. The SOC cost is therefore a path-regularized upper bound on the endpoint objective, not an identical rewrite of it. We optimize this cost through the score-network parameters using the checkpointed rollout differentiation described in \cref{sec:fine-tuning-objective-optimization}, without learning a separate control network or value function. Rewriting the sampler in increasing sampling time reverses the drift and control signs but leaves the quadratic cost unchanged.

\paragraph{Path-space trust regions.} The same penalty admits a constrained interpretation,
\begin{equation}
\min_\phi\;-\mathbb{E}_{\mathbb{P}_\phi}[R(Y_0;y^o)]
\qquad\text{subject to}\qquad
\mathrm{KL}(\mathbb{P}_\phi\|\mathbb{P}_{\phi_0})\leq\rho,
\label{eq:soc-path-trust-region}
\end{equation}
where $\rho\geq0$ is a divergence budget. Its Lagrangian is $\mathcal{J}_\lambda(\phi)-\lambda\rho$ with multiplier $\lambda\geq0$, so fixing $\lambda$ recovers the penalized optimization over $\phi$. This motivates choosing the penalty through a divergence budget rather than treating it only as a tuning constant, in the spirit of path-space trust-region methods \citep{blessing2025trustregion}. For multiple observed designs, constraints $\mathrm{KL}(\mathbb{P}_\phi^{\theta^{\mathrm{ref}},\xi}\|\mathbb{P}_{\phi_0}^{\theta^{\mathrm{ref}},\xi})\leq\rho(\xi)$ would introduce design-specific multipliers $\lambda_\xi$. Adjusting these against their budgets could regulate how strongly each condition is adapted, alongside a curriculum that selects which designs to revisit. This is a prospective extension, not an implemented adaptive optimizer or a claim of strong duality for the neural parameterization.

\subsection{Fisher--Rao optimization analysis}
\label{sec:fisher-rao-analysis}

\noindent\begin{minipage}[t]{0.56\linewidth}
\vspace{0pt}
Let $z=(\theta,\xi)$ collect the mechanistic parameters and experimental design. Under the regularity assumptions in \cref{sec:path-divergence}, the conditional reverse-SDE path laws $\{\mathbb{P}_{\phi}^{z}:z\in\Theta\times\Xi\}$ can be viewed locally as a statistical manifold. Its Fisher--Rao geometry is Riemannian on neighborhoods where the Fisher metric is nonsingular, and degenerate where it is singular. The path-space Fisher metric in \cref{eq:path-fisher-main} can be partitioned according to the two types of conditioning variables (\cref{fig:joint-fisher-geometry}),
\begin{equation}
\mathcal{I}_{\phi}^{\mathrm{path}}(z)
=
\begin{bmatrix}
I_{\theta\theta}(z) & I_{\theta\xi}(z) \\
I_{\xi\theta}(z) & I_{\xi\xi}(z)
\end{bmatrix},
\label{eq:path-fisher-blocks}
\end{equation}
\end{minipage}\hfill
\begin{minipage}[t]{0.40\linewidth}
\vspace{0pt}
\centering
\captionsetup{type=figure,font=small,skip=3pt}
\vspace{-3pt}
\includegraphics[width=\linewidth]{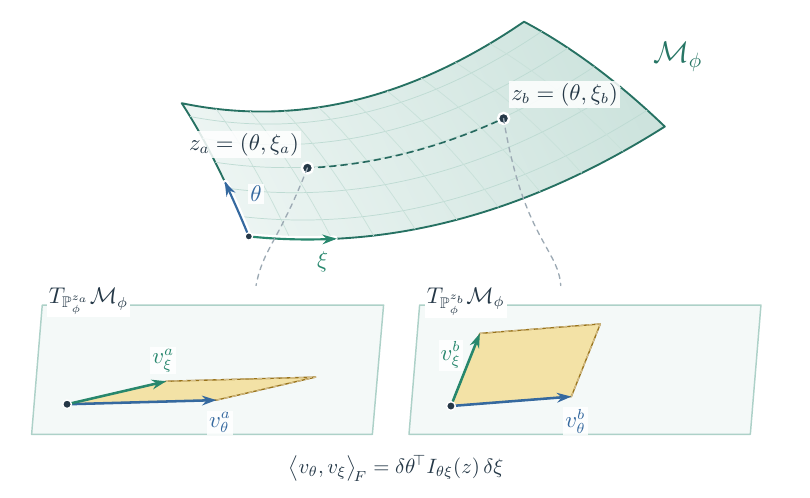}
\caption{Joint parameter Fisher geometry.}
\label{fig:joint-fisher-geometry}
\end{minipage}
\par\medskip

For example, the parameter block of \cref{eq:path-fisher-main} is
\begin{equation*}
I_{\theta\theta}(z)
=
\mathbb{E}_{Y\sim\mathbb{P}_{\phi}^{z}}
\left[
\int_0^T
g(t)^2
J_{\theta}s_{\phi}^{y}(Y_t,t;z)^{\top}
J_{\theta}s_{\phi}^{y}(Y_t,t;z)\,dt
\right].
\end{equation*}
We first consider how this geometry could refine a population of mechanistic reference parameters during fine-tuning. We then examine design sensitivity and parameter--design coupling, before propagating reference uncertainty into path-divergence diagnostics.

\paragraph{Refining the posterior proposal.}
The block $I_{\theta\theta}$ measures how perturbations of the mechanistic parameters change the conditional path law while holding the experimental design fixed. As a proposed extension, the fixed reference $\theta^{\mathrm{ref}}$ could be replaced by a posterior proposal distribution $q(\theta)$ that is updated during fine-tuning. Consider one observed design--response pair $(\xi,y^o)$ and $m$ proposal particles $\theta^{(i)}\sim q$. For each particle, we draw $n$ predictive response samples $y_j^{(i)}\sim p_\phi(\cdot\mid\theta^{(i)},\xi)$ at this same design. We use the same reward $R(y;y^o)$ as in \cref{sec:fine-tuning-objective-optimization} and write $\mathbb E_\phi[\cdot\mid\theta,\xi]$ for expectation over responses from this conditional likelihood. Averaging over proposal particles and predictive responses gives
\begin{equation*}
\mathbb{E}_{\theta\sim q}\mathbb E_\phi[R(y;y^o)\mid\theta,\xi]
\approx
-\frac{1}{mn}\sum_{i=1}^m\sum_{j=1}^n\|y_j^{(i)}-y^o\|_2^2.
\end{equation*}

At fixed $\phi$, the same reward can be used to tilt the posterior proposal toward parameter values whose conditional likelihood better predicts the observation. Specifically, consider
\begin{align}
q_\phi^\star
&=
\arg\min_q
\left\{
-\mathbb E_{\theta\sim q}
\left[
\mathbb E_\phi[R(y;y^o)\mid\theta,\xi]
\right]
+
\tau\,\mathrm{KL}(q\|p)
\right\},
\label{eq:reference-variational-objective}\\
q_\phi^\star(\theta)
&\propto
p(\theta)
\exp\!\left[
\frac{
\mathbb E_\phi[R(y;y^o)\mid\theta,\xi]
}{\tau}
\right],
\label{eq:reward-tilted-references}
\end{align}
where $p(\theta)$ is the prior and $\tau>0$ controls the strength of the reward tilt, so predictive agreement reweights the prior toward a posterior proposal concentrated on parameter values that better explain the observation. Recent concurrent work similarly contrasts KL reward tilting with transport-based adaptation, moving samples toward higher reward under a geometry induced by the pretrained generative dynamics \citep{mammadov2026wtf}. Motivated by a related transport view in mechanistic parameter space, we instead use the Fisher geometry of the conditional diffusion path law to move proposal particles in $\theta$.

The Fisher block $I_{\theta\theta}(\theta,\xi)$ provides a local geometry for reward-driven motion of the posterior proposal particles. At iteration $k$, with $\phi_k$ fixed, we propose the particle update
\begin{align}
\theta_{k+1}^{(i)}
&=
\theta_k^{(i)}+\Delta\theta_k^{(i)},
\label{eq:theta-particle-update}\\
\Delta\theta_k^{(i)}
&=
\eta\,I_{\theta\theta}(\theta_k^{(i)},\xi)^\dagger
\left.
\nabla_\theta
\mathbb E_{\phi_k}[R(y;y^o)\mid\theta,\xi]
\right|_{\theta=\theta_k^{(i)}}\notag\\
&\quad+
\eta\tau\,\nabla_\theta\log p(\theta_k^{(i)})
+
\sqrt{2\eta\tau}\,\varepsilon_k^{(i)},
\qquad
\varepsilon_k^{(i)}\sim\mathcal N(0,I).
\label{eq:theta-natural-gradient}
\end{align}
Here $\eta>0$ is the particle-update step size, and $\dagger$ denotes the Moore--Penrose pseudoinverse. For a positive-semidefinite Fisher matrix, it inverts positive eigenvalues while leaving zero eigenvalues at zero, reducing to the ordinary inverse when the matrix is nonsingular. The first term moves each proposal particle toward improved predictive agreement using the learned path geometry. In Fisher-null directions this term vanishes, while the prior drift and stochastic term remain active, allowing the proposal population to explore mechanistic variation that is not locally resolved by the conditional path law.

The reward-tilted distribution $q_{\phi_k}^\star$ expresses an ideal balance between predictive agreement and proximity to the prior. Motivated by this balance, the updates above use predictive sensitivities, prior drift, and stochastic exploration to refine an evolving proposal population, without assuming that its particles are distributed according to $q_{\phi_k}^\star$. After updating all $m$ particles, we represent the next proposal by the empirical measure
\begin{equation*}
q_{k+1}(\theta)
=
\frac{1}{m}\sum_{i=1}^m
\delta_{\theta_{k+1}^{(i)}}(\theta).
\end{equation*}
Holding these proposal particles fixed, we update the likelihood parameters by averaging the fine-tuning objective in \cref{eq:fine-tune-objective} across them, using the path-space penalty $\mathcal{A}$ from \cref{eq:score-control-energy} in place of the endpoint KL as in the main text,
\begin{equation}
\phi_{k+1}
\in
\arg\min_\phi
\frac{1}{m}\sum_{i=1}^m
\left[
-\mathbb E_\phi\!\left[R(y;y^o)\mid\theta_{k+1}^{(i)},\xi\right]
+\lambda\,\mathcal{A}\!\left(\theta_{k+1}^{(i)},\xi\right)
\right],
\label{eq:alternating-reference-fine-tuning}
\end{equation}
Here $\mathcal{A}(\theta,\xi)$ is evaluated for the candidate model $\phi$ relative to the fixed PT model $\phi_0$, so each summand is the conditional SOC cost $\mathcal{J}_\lambda$ in \cref{eq:soc-ft-objective}. The procedure therefore alternates between updating the posterior proposal under the current likelihood and fine-tuning the likelihood across the resulting proposal particles.

\paragraph{Design sensitivity and active simulation acquisition.}
The block $I_{\xi\xi}$ measures how strongly the conditional path law changes under local perturbations of the experimental design. For a small $\delta\xi$,
\begin{equation*}
\mathrm{KL}
\left(
\mathbb{P}_{\phi}^{\theta,\xi+\delta\xi}
\middle\|
\mathbb{P}_{\phi}^{\theta,\xi}
\right)
\approx
\frac12
\delta\xi^{\top}
I_{\xi\xi}(\theta,\xi)
\delta\xi.
\end{equation*}

\textit{PT model.} Large $I_{\xi\xi}$ identifies regions where the learned simulator surrogate changes rapidly with design. Scalar summaries such as $\mathrm{tr}\,I_{\xi\xi}$ or $\log\det(I_{\xi\xi}+\epsilon I)$ could therefore prioritize additional simulator calls where denser coverage is most useful, rather than expanding the pretraining set uniformly. Related work uses diffusion models to acquire simulation parameters from surrogate loss or uncertainty in online PDE-surrogate training \citep{cesar2026learning}, whereas our proposed signal is derived from local path-space sensitivity.

\textit{FT model.} $I_{\xi\xi}$ instead measures sensitivity of the adapted likelihood across experimental conditions. Designs with large sensitivity may warrant greater weight in a fine-tuning curriculum, particularly when accompanied by high predictive error or large PT--FT path divergence. Sensitivity alone does not imply misspecification, since additional PT simulations improve fidelity to the simulator whereas $\mathcal{A}(\theta,\xi)$ measures departure from the simulator toward observations.

\textit{Experimental design.} The same geometry could also guide the choice of experiments for learning mechanistic parameters. At a candidate design $\xi$, large $I_{\theta\theta}(\theta,\xi)$ indicates strong local sensitivity of the predictive path law to $\theta$. An objective such as
\begin{equation*}
\log\det
\left(
I_{\theta\theta}(\theta,\xi)+\epsilon I
\right)
\end{equation*}
could therefore prioritize experimental conditions that are locally informative about the mechanistic parameters.

\paragraph{Parameter--design coupling for proposal refinement.}
Beyond selecting new designs, the joint geometry could help refine posterior proposals using the designs already observed.

\textit{Parameter--design overlap.}
The cross block $I_{\theta\xi}$ measures whether perturbations of $\theta$ and $\xi$ induce similar first-order changes in the conditional path law. Removing their respective local scales gives the normalized coupling
\begin{equation}
C(\theta,\xi)
=
I_{\theta\theta}^{\dagger/2}
I_{\theta\xi}
I_{\xi\xi}^{\dagger/2},
\label{eq:theta-xi-coupling}
\end{equation}
with $I^{\dagger/2}=(I^\dagger)^{1/2}$. The singular vectors of $C$ identify parameter and design directions with aligned effects. Larger singular values indicate parameter-induced changes that can be closely reproduced by local design perturbations, while smaller values indicate more distinct effects.

\textit{Residual parameter geometry.}
Removing this design-associated component from the parameter block gives the Schur complement
\begin{equation}
I_{\theta\mid\xi}^{\mathrm{eff}}
=
I_{\theta\theta}
-
I_{\theta\xi}
I_{\xi\xi}^{\dagger}
I_{\xi\theta}.
\label{eq:fisher-schur}
\end{equation}
Geometrically, $I_{\theta\mid\xi}^{\mathrm{eff}}$ retains the component of parameter-induced path-law variation that cannot be reproduced locally by changing the design.

\textit{Proposal refinement across observed designs.}
For a collection of observed designs, a non-negative weighted average of
$I_{\theta\mid\xi}^{\mathrm{eff}}(\theta,\xi)$ could replace
$I_{\theta\theta}$ in the proposal-particle update in
\cref{eq:theta-natural-gradient}. The resulting geometry emphasizes
mechanistic directions whose effects remain distinguishable across the
fine-tuning curriculum after removing locally design-mimickable variation.

The same effective information could also be used prospectively to favor new designs at which weakly identified parameter directions become distinguishable from nearby design perturbations.

\paragraph{Uncertainty-aware path divergence.} A distribution over reference parameters also yields an uncertainty-aware version of the adaptation diagnostic. In place of evaluating the path divergence $\mathcal{A}(\theta,\xi)$ from \cref{eq:score-control-energy} at one $\theta^{\mathrm{ref}}$, define
\begin{align}
\overline{\mathcal{A}}_{q}(\xi)
&=
\mathbb{E}_{\theta\sim q}
\left[
\mathcal{A}(\theta,\xi)
\right],
\label{eq:uncertain-adaptation}\\
V_{\mathcal{A},q}(\xi)
&=
\mathrm{Var}_{\theta\sim q}
\left[
\mathcal{A}(\theta,\xi)
\right].
\label{eq:adaptation-reference-variance}
\end{align}
The first quantity identifies designs at which substantial adaptation is required across plausible mechanistic parameterizations, whereas the second identifies designs whose apparent misspecification depends strongly on parameter uncertainty. Here $V_{\mathcal{A},q}(\xi)$ quantifies how sensitive the inferred adaptation is to uncertainty in the mechanistic reference $\theta$. This is distinct from Monte Carlo uncertainty arising from the finite trajectory samples used to estimate each $\mathcal{A}(\theta,\xi)$.

\paragraph{Future work.} Although we do not develop these extensions here, they suggest promising ways to improve both PT and FT models in future SBI applications. For example, path-space geometry could guide refinement and particle transport of posterior proposals over mechanistic parameters, while incorporating parameter uncertainty into adaptation diagnostics (\cref{eq:uncertain-adaptation,eq:adaptation-reference-variance}). It could also inform simulation acquisition for pretraining, complementing diffusion-based active sampling \citep{cesar2026learning}, while supporting design-adaptive fine-tuning curricula and local parameter--design identifiability analysis (\cref{eq:theta-xi-coupling,eq:fisher-schur}).

\subsection{Related work on simulator misspecification in SBI}
\label{sec:sbi-misspecification-related-work}

Simulation-based inference (SBI) treats the expert model simulator as an implicit, or intractable, likelihood model from which one can sample, \(y\sim p(y | \theta)\), but whose density cannot be evaluated. Starting with samples from a prior \(\theta \sim p(\theta)\), synthetic pairs \((y,\theta)\) are generated to fit a neural density estimator for the likelihood \(p_\phi(y | \theta)\) \citep{papamakarios2019sequentialneurallikelihoodfast, lueckmann2018likelihood}, posterior \(p_\phi(\theta | y)\) \citep{Papamakarios2016, greenberg2019automatic, lueckmann2017flexible, sharrock2022sequential}, or likelihood-to-evidence ratio \citep{miller_contrastive_2022, Durkan2020} using observed data \(y_o\), enabling parameter inference without explicit likelihoods. Thus, SBI is well-suited for biological simulators that use stochastic differential equations or simulation by convex optimization \citep{dirks2007, su2022ligand}. SBI relies on a simulator that accurately reflects the true, and unknown, data-generating process of the observed data, i.e., $p^\star(y_o) \coloneqq \int p^\star(\theta)p^\star(y_o | \theta) d\theta$, which is typically unavailable and can be the cause of poor performance of SBI models \citep{Cannon2022, schmitt2022detectingmodelmisspec}. This has resulted in a variety of techniques to improve robustness of SBI inference to model misspecification \citep{ward2022robust, huang2023learning, gao2023generalized, wehenkel2025addressing, mishra2025robust} but there is a gap in methods that \emph{transfer} improvement in predictive accuracy to posterior inference.

\subsection{Related work on diffusion fine-tuning}
\label{sec:diffusion-fine-tuning-related-work}

Diffusion fine-tuning is a rapidly growing field with diverse optimization and regularization strategies. We review a selection of approaches relevant to SBI, summarized in \cref{tab:diffusion_ft_related}, rather than attempting an exhaustive survey. The comparison is intended to highlight the main tradeoffs among methods, including what feedback they require, how that signal is converted into model updates, whether gradients must pass through sampling, what auxiliary quantities must be learned, and how strongly the adapted model is tied to its pretrained reference. We then motivate our use of Implicit Diffusion for the present SBI setting while highlighting alternative strategies that may be preferable under different feedback, computational, and regularization requirements.

\begin{table*}[!t]
\centering
\caption{
Representative approaches to reward- and constraint-based generative-model adaptation, compared by the feedback they require, how model updates are obtained, whether gradients propagate through sample generation, any additional learned objects, and how deviation from a reference model is controlled. In the feedback column, $R$ denotes reward values without sample derivatives, $\nabla_y R$ denotes differentiable reward or loss feedback, $(h,h^\star)$ denotes constraint statistics and target moments, and $\nabla_y E$ denotes gradients of a differentiable unnormalized target energy $E(y)$, with $p^\star(y)\propto\exp[-E(y)]$; when an energy is instead defined relative to a reference distribution $p_0$, the equivalent tilted target is $p^\star(y)\propto p_0(y)\exp[-E(y)]$. In the sampler-gradient column, ``Full'' denotes differentiation through the complete sampled trajectory and explicit step labels indicate partial differentiation. In the learned-auxiliary column, ``--'' means that no additional function or parameter is fitted beyond the adapted generator or control and any supplied reward model; computed adjoint states, trajectory queues, and replay buffers are not counted. PT denotes a pretrained model. For Implicit Diffusion, \textbf{Full}$^\ast$ denotes full finite-horizon sampler differentiation. Our implementation uses checkpointed backpropagation through a complete discretized rollout rather than the original queued joint sampling--optimization scheme, so that each gradient update corresponds to a complete rollout under a single selected design.
}
\label{tab:diffusion_ft_related}
\small
\setlength{\tabcolsep}{3pt}
\renewcommand{\arraystretch}{1.12}
\begin{tabular}{@{}L{\dimexpr0.19\linewidth-5pt\relax}L{\dimexpr0.105\linewidth-5pt\relax}L{\dimexpr0.215\linewidth-5pt\relax}L{\dimexpr0.12\linewidth-5pt\relax}L{\dimexpr0.18\linewidth-5pt\relax}L{\dimexpr0.19\linewidth-5pt\relax}@{}}
\toprule
\textbf{Method} & \textbf{Feedback} & \textbf{Update mechanism} & \textbf{Sampler gradients} & \textbf{Learned auxiliaries} & \textbf{Reference regularization} \\
\midrule
DPOK \citep{fan2023dpok}
& $R$ & Policy gradient & -- & Optional value baseline & Discrete path KL \\
\addlinespace[3pt]
DDPO \citep{black2024ddpo}
& $R$ & Policy gradient & -- & -- & No explicit PT KL \\
\addlinespace[3pt]
ReFL \citep{xu2023imagereward}
& $\nabla_y R$ & Intermediate clean-prediction gradient & 1 selected step & -- & Pretraining denoising loss \\
\addlinespace[3pt]
AlignProp \citep{prabhudesai2023alignprop}
& $\nabla_y R$ & Randomly truncated rollout backpropagation & Random trailing steps & -- & No explicit PT KL \\
\addlinespace[3pt]
DRaFT-$K$ \citep{clark2024draft}
& $\nabla_y R$ & Truncated rollout backpropagation & Last $K$ steps & -- & LoRA weight decay (implementation) \\
\addlinespace[3pt]
ELEGANT \citep{uehara2024entropycontrol}
& $\nabla_y R$ & SOC objective via rollout gradients & Full & Initial-time value function and initial-law sampler & Drift-control cost and initial-law KL \\
\addlinespace[3pt]
Adjoint Matching \citep{domingoenrich2025adjointmatchingfinetuningflow}
& $\nabla_y R$ & Regression on detached adjoint-derived control targets & -- & -- & SDE path KL (SOC) \\
\addlinespace[3pt]
TR-SOCM \citep{blessing2025trustregion}
& $\nabla_y R$ & Reweighted adjoint regression & -- & -- & Inter-iterate path-KL constraint \\
\addlinespace[3pt]
Adjoint Sampling \citep{havens2025adjointsamplinghighlyscalable}
& $\nabla_y E$ & Adjoint regression with reference bridges & -- & -- & Control cost, energy target \\
\addlinespace[3pt]
Implicit Diffusion \citep{marion2025implicitdiffusionefficientoptimization}
& $\nabla_y R$ & Joint sampling--optimization with adjoint gradients & Full$^\ast$ & -- & Optional SDE path KL \\
\addlinespace[3pt]
CGM-relax / CGM-reward \citep{smith2026calibrating}
& $(h,h^\star)$ & Likelihood-ratio gradients & -- & Tilt coefficients (reward only) & KL to PT (path KL for SDEs) \\
\addlinespace[3pt]
Value Matching \citep{jensen2026value}
& $R$ & Online value-function learning & -- & Value function & Control-based \\
\bottomrule
\end{tabular}
\end{table*}

\paragraph{Reward optimization and gradient estimators.} Reward-based diffusion fine-tuning encompasses policy-gradient and direct-differentiation methods, with different requirements on reward access \citep{uehara2024tutorial}. DPOK \citep{fan2023dpok} and DDPO \citep{black2024ddpo} treat denoising as a sequential decision process, allowing optimization without differentiating the reward. DPOK explicitly regularizes deviation from the PT model through a KL penalty and optionally learns a value baseline to reduce gradient variance. When rewards are differentiable, ReFL \citep{xu2023imagereward} backpropagates through a clean prediction from one selected late denoising step, holding earlier states fixed. AlignProp \citep{prabhudesai2023alignprop} uses randomly truncated rollout backpropagation, with low-rank adaptation and gradient checkpointing to reduce memory requirements. DRaFT-$K$ \citep{clark2024draft} restricts backpropagation to the last $K$ denoising steps, and DRaFT-LV reduces gradient variance for $K=1$. These approaches trade access to reward gradients and the cost of differentiating long trajectories against the sampling variance of policy-gradient estimates. Preference-based alternatives such as Diffusion-DPO optimize an ELBO-derived objective from paired comparisons rather than a supplied scalar reward \citep{wallace2024diffusiondpo}. We focus here on methods compatible with observation-based rewards and distributional constraints.

\paragraph{Stochastic optimal control.} SOC treats fine-tuning as learning a change to the PT drift that improves terminal reward while penalizing the accumulated control effort. With a shared initial law and diffusion coefficient, Girsanov's theorem expresses path-space KL as a quadratic cost of this drift change (\cref{sec:girsanov_background}). This gives a trajectory-level notion of staying close to the PT model, rather than a penalty on network weights. ELEGANT \citep{uehara2024entropycontrol} uses this formulation to optimize both the denoising drift and the initial noise distribution, regularizing both changes. Its implementation estimates the value function at the initial diffusion time and learns an auxiliary process for sampling the adapted initial distribution. Adjoint Matching \citep{domingoenrich2025adjointmatchingfinetuningflow} takes a different route, fitting the control to targets obtained from a simplified backward adjoint equation. Trajectories and targets are held fixed during each regression update, avoiding differentiation through the controlled rollout, although reward and base-drift derivatives are still needed to construct the targets. Its reward fine-tuning formulation uses a memoryless noise schedule to recover the intended reward-tilted output distribution without separately learning the initial noise law. TR-SOCM \citep{blessing2025trustregion} fits controls using reweighted adjoint targets with a path-KL constraint relative to the preceding iterate. A scalar dual multiplier controls the trust region, while buffered trajectories support reuse across updates. Adjoint Sampling \citep{havens2025adjointsamplinghighlyscalable} instead applies adjoint regression to unnormalized energy targets, using reference bridges to reuse terminal samples and energy gradients across multiple parameter updates. It learns an energy-targeting sampler rather than preserving a PT data distribution. These options differ in the auxiliary quantities they learn, their sampling requirements, and their target distributions.

\paragraph{Implicit Diffusion and sampler gradients.} Implicit Diffusion \citep{marion2025implicitdiffusionefficientoptimization} separates gradient estimation through a parameterized sampler from the algorithm used to interleave sampling and optimization. Its finite-horizon adjoint formulation propagates terminal-loss sensitivity through the sampling dynamics and can incorporate a Girsanov path-KL penalty \citep[Appendix A.3]{marion2025implicitdiffusionefficientoptimization}. Algorithm 3 additionally maintains a queue of trajectories that advance while model parameters change, coupling sampling with parameter updates. Our implementation instead completes reverse-SDE rollouts at fixed parameters and uses checkpointed backpropagation, computing discrete-adjoint gradients while trading recomputation for memory. We therefore use the finite-horizon sampler-gradient formulation, not the queued optimization algorithm. This shares the pathwise differentiation principle of DRaFT \citep{clark2024draft}, whereas Adjoint Matching holds trajectories and adjoint targets fixed during control regression. Neither checkpointing nor adjoint differentiation itself specifies the regularizer.

\paragraph{Controlling departure from the pretrained model.}
Fine-tuning methods differ not only in how they obtain parameter updates, but also in how they constrain departure from the pretrained (PT) model. DRaFT uses AdamW weight decay on its LoRA factors, shrinking the learned update toward the PT parameters \citep[Appendix A.1]{clark2024draft}. This is parameter-space regularization rather than a KL between sampling laws; a separate ablation instead penalizes PT--FT noise-prediction differences at the final denoising step \citep[Appendix B.8]{clark2024draft}. DPOK directly regularizes the sampling distribution by summing transition KLs over the discrete denoising chain, yielding a path-space KL that also bounds endpoint divergence \citep{fan2023dpok}. ReFL instead retains a pretraining denoising loss \citep{xu2023imagereward}. For continuous SDEs with shared diffusion and initial laws, the corresponding path KL is the integrated drift-control energy in \cref{eq:girsanov_kl}; this is the regularizer we pair with Implicit Diffusion. ELEGANT additionally regularizes changes to the initial noise law. Adjoint Sampling uses a control cost to reach an energy-defined target, so its regularization serves target transport rather than preservation of a PT data distribution.

\paragraph{Feedback and distributional constraints.} Beyond the gradient estimator, the reliability and availability of feedback determine which adaptations are useful. \citet{uehara2024feedback} study online fine-tuning with limited reward queries, while BRAID \citep{uehara2024braid} uses conservative reward estimates to limit overoptimization outside the support of offline data. Calibrating Generative Models (CGM) \citep{smith2026calibrating} seeks the KL-closest model satisfying moment constraints. CGM-relax penalizes constraint violations, whereas CGM-reward fits exponential-tilt coefficients that define a reward. Both differentiate density ratios with sampled outputs held fixed, avoiding derivatives of the constraint functions and backpropagation through sample generation. For continuous diffusions, these are path-measure ratios evaluated using Girsanov's theorem, so tractable endpoint likelihoods are not required. A response-mean constraint also differs from our expected sample-wise MSE, which penalizes predictive variance as well as mean error (\cref{sec:fine-tuning-objective-optimization}). Flow Density Control \citep{desanti2025flowdensitycontrol} extends optimization to broader utilities and divergences through mirror-flow updates, and Value Matching \citep{jensen2026value} studies gradient-free reward-guided flow adaptation. These works motivate considering both the computational cost of fine-tuning and whether its objective remains informative away from the observed data.

\paragraph{Why this fine-tuning formulation for SBI?} The challenge in our SBI setting is to match a finite set of observed design--response pairs $(\xi,y_o)$ without access to the true data-generating process for additional samples or queries at unmeasured designs. We adapt the conditional likelihood at a simulator-aligned $\theta^{\mathrm{ref}}$, using discrepancy from these observations as the learning signal and the PT path law as a regularizing reference. We optimize this path-regularized objective using checkpointed backpropagation through the discretized reverse-SDE sampler. This computes pathwise parameter gradients through a discrete adjoint, following the finite-horizon sampler-gradient formulation of Implicit Diffusion \citep{marion2025implicitdiffusionefficientoptimization} and sharing the differentiation principle of direct reward fine-tuning \citep{clark2024draft}. This lets us differentiate the observation-based loss directly without fitting a separate value function or replacing it with a control-regression objective. The SBI contribution is the adaptation at mechanistically grounded conditioning values, its transfer to posterior queries through shared FiLM-modulated embeddings (\cref{sec:rep_learn}), and the interpretation of conditional path divergence alongside predictive error, rather than a new gradient estimator. Because the required correction varies with $\xi$, choosing which designs to revisit, how to weight their losses, and how strongly to regularize raises a \textit{curriculum} problem \citep{bengio2009curriculum}. Our toy illustrates how updates that correct a misspecified region can degrade regions that were already accurate (\cref{sec:design-misspecification-toy}). Interpreting $\mathcal A(\theta,\xi)$ alongside predictive error reveals these heterogeneous effects and motivates adaptive curricula and design-specific trust regions, with weights $\lambda_\xi$ adjusted to control departure from the PT path law \citep{blessing2025trustregion}. Such design-dependent penalties remain a proposed extension. Adjoint Matching, TR-SOCM, and Adjoint Sampling provide alternative optimization routes worth investigating for this conditional likelihood setting.

\subsection{Embedding-based transfer learning}
\label{sec:rep_learn}

Fine-tuning a joint simformer through its likelihood requires the learned correction to transfer to posterior queries, but the original architecture of \citet{gloeckler24a} zeros condition inputs for latent tokens and leaves this transfer largely to the transformer. We instead apply token-aware Feature-wise Linear Modulation (FiLM) \citep{perez2017filmvisualreasoninggeneral} before the transformer. For token $i$, let $e_i^{\mathrm{val}}$, $e_i^{\mathrm{id}}$, and $e_i^{\mathrm{cond}}$ denote its value, identity, and latent-or-conditioned state embeddings, and define $b_i=[e_i^{\mathrm{val}},e_i^{\mathrm{id}}]$. A linear projection of $[e_i^{\mathrm{id}},e_i^{\mathrm{cond}}]$ produces $[\tilde\gamma_i,\tilde\beta_i]$, which we bound as $\gamma_i=1+\varepsilon\tanh(\tilde\gamma_i)$ and $\beta_i=\varepsilon\tanh(\tilde\beta_i)$ before forming the transformer input $z_i=[\gamma_i\odot b_i+\beta_i,e_i^{\mathrm{cond}}]$. Thus, the condition state both enters the transformer explicitly and modulates the value--identity representation, while $\varepsilon$ limits how far latent and conditioned embeddings can separate; likelihood fine-tuning can therefore update a representation that posterior queries reuse.

\subsection{FMCPE-LSR baseline}
\label{sec:fmcpe-lsr-details}

\begin{table}[!b]
\centering
\captionsetup{font=small,skip=4pt}
\caption{Predictive errors through the original BMP simulator on all 940 measurements, using 250 parameter draws per run. Values are mean $\pm$ SE over three FMCPE training seeds or three sampling seeds of each fixed PT or FT simformer. FMCPE $\ell$-C2ST uses separate shared-normalization correction refits and a joint 80-dimensional parameter target.}
\label{tab:fmcpe-lsr-budget}
\setlength{\tabcolsep}{2pt}
\renewcommand{\arraystretch}{1.05}
\footnotesize
\begin{tabular*}{\linewidth}{@{\extracolsep{\fill}}lrrccc@{}}
\toprule
Method & Observations & LSR fits & RMSE $\downarrow$ & Median distance $\downarrow$ & $\ell$-C2ST $\downarrow$ \\
\midrule
FMCPE-LSR & 235 & 1,204 & $0.242 \pm 0.004$ & $203.661 \pm 38.307$ & $0.198 \pm 0.023$ \\
FMCPE-LSR & 235 & 2,408 & $0.231 \pm 0.026$ & $203.451 \pm 70.019$ & $0.150 \pm 0.039$ \\
FMCPE-LSR & 235 & 3,612 & $0.184 \pm 0.010$ & $77.887 \pm 19.521$ & $0.194 \pm 0.050$ \\
FMCPE-LSR & 235 & 4,816 & $0.168 \pm 0.011$ & $136.460 \pm 77.855$ & $0.157 \pm 0.049$ \\
\midrule
FMCPE-LSR & 470 & 1,204 & $0.256 \pm 0.008$ & $283.935 \pm 26.484$ & $0.137 \pm 0.020$ \\
FMCPE-LSR & 470 & 2,408 & $0.202 \pm 0.006$ & $108.803 \pm 19.193$ & $0.081 \pm 0.013$ \\
FMCPE-LSR & 470 & 3,612 & $0.195 \pm 0.003$ & $154.748 \pm 46.537$ & $0.095 \pm 0.019$ \\
FMCPE-LSR & 470 & 4,816 & $0.177 \pm 0.003$ & $31.505 \pm 10.206$ & $0.065 \pm 0.004$ \\
\midrule
FMCPE-LSR & 705 & 1,204 & $0.246 \pm 0.003$ & $296.322 \pm 88.535$ & $0.074 \pm 0.027$ \\
FMCPE-LSR & 705 & 2,408 & $0.197 \pm 0.014$ & $161.322 \pm 41.237$ & $0.099 \pm 0.022$ \\
FMCPE-LSR & 705 & 3,612 & $0.193 \pm 0.002$ & $98.697 \pm 18.042$ & $0.124 \pm 0.046$ \\
FMCPE-LSR & 705 & 4,816 & $0.168 \pm 0.007$ & $50.355 \pm 16.256$ & $0.072 \pm 0.036$ \\
\midrule
FMCPE-LSR & 940 & 1,204 & $0.270 \pm 0.012$ & $203.950 \pm 20.409$ & $0.083 \pm 0.017$ \\
FMCPE-LSR & 940 & 2,408 & $0.210 \pm 0.025$ & $188.117 \pm 73.589$ & $0.117 \pm 0.029$ \\
FMCPE-LSR & 940 & 3,612 & $0.186 \pm 0.023$ & $62.011 \pm 26.766$ & $0.129 \pm 0.010$ \\
FMCPE-LSR & 940 & 4,816 & $\mathbf{0.175 \pm 0.005}$ & $31.051 \pm 8.460$ & $0.086 \pm 0.017$ \\
\midrule
PT BMP Simformer & 940 & 0 & $0.252 \pm 0.003$ & $11.111 \pm 0.116$ & $0.157 \pm 0.010$ \\
FT BMP Simformer & 940 & 1 & $0.245 \pm 0.003$ & $\mathbf{10.376 \pm 0.058}$ & $0.189 \pm 0.019$ \\
\bottomrule
\end{tabular*}
\end{table}

\paragraph{Protocol.} FMCPE first trains a neural posterior estimator (NPE) on simulator-generated pairs and then uses flow matching to correct its predictions with a smaller calibration set \citep{ruhlmann2026fmcpe}. For BMP, we used the authors' Lampe/Zuko NPE implementation and treated the $4{,}816$ retained LSR fits as pseudo-calibration parameters. They are estimates from the observed BMP dataset, not ground-truth parameters from a high-fidelity data-generating process. The complete sweep contains $48$ runs: four nested observation budgets ($235$, $470$, $705$, and $940$ measurements), four nested LSR budgets ($1{,}204$, $2{,}408$, $3{,}612$, and $4{,}816$ fits), and three seeds ($33$, $43$, and $53$). Observation subsets were balanced across cell lines. For each condition, we trained a dedicated NPE and correction flows, drew $1{,}024$ posterior samples, and evaluated the first $250$ through the original one-step BMP simulator over all $940$ measurements.

For the simformer comparison, we evaluated the FT model and the PT model from which it was initialized, each with sampling seeds $42$, $43$ and $44$. Every run used $250$ joint draws of the $60$ biophysical parameters and $20$ cell-line receptor concentrations, $20$ predictor steps, five corrector steps, three EM steps, hierarchical receptor sampling, and Tweedie conditioning in both parameter blocks. We computed the predictive metrics using the original BMP simulator over all $940$ measurements.

\paragraph{Shared predictive metrics.} Across our BMP posterior-inference experiments, we evaluate inferred parameters by running them through the original simulator and comparing its predictions with the observed responses. RMSE and median distance summarize this agreement, capturing the best fit among posterior draws and the typical fit across draws, respectively. For the full-data comparison here, let $\hat{\mathbf y}^{(s)}$ be the simulator's response vector for posterior draw $s$, let $\mathbf y^o$ contain the $940$ observed responses, and let $r_y=\max_i y_i^o-\min_i y_i^o$. We report
\begin{align}
    \operatorname{RMSE}_{\min}
    &=\min_s\frac{1}{r_y}\sqrt{\frac{1}{940}\sum_{i=1}^{940}
      \left(\hat y_i^{(s)}-y_i^o\right)^2},
\label{eq:bmp-predictive-metrics}\\
    d_{\mathrm{med}}&=\operatorname{median}_{s}
      \left\|\hat{\mathbf y}^{(s)}-\mathbf y^o\right\|_2.
\label{eq:bmp-median-predictive-distance}
\end{align}
This same RMSE definition is used for every method, with the FMCPE entries taken from its saved $250$-draw prediction banks. Both metrics are in response space.

\paragraph{Results.} FMCPE-LSR can achieve lower best-draw RMSE than either simformer, particularly with larger LSR budgets (\cref{tab:fmcpe-lsr-budget}). With all $940$ observations and $4{,}816$ LSR fits, its RMSE is $0.175$, compared with $0.252$ for PT and $0.245$ for FT. This advantage does not extend to median distance, which is $31.05$ for full-budget FMCPE-LSR versus $11.11$ for PT and $10.38$ for FT. Neither simformer is surpassed on mean median distance by any FMCPE-LSR configuration in the sweep. For each observation subset, our FMCPE-LSR implementation required training a dedicated NPE and two correction flows, whereas the PT and FT joint models reuse their weights across subsets through masking and compositional sampling. The FMCPE networks contained $0.87$--$1.28$ million parameters in total, including the NPE frozen during correction, compared with approximately $0.48$ million for each PT/FT simformer. Thus, the lower best-draw error came with subset-specific training and a larger combined model, without a corresponding improvement in typical predictive fit.

\subsection{Design-dependent misspecification toy}
\label{sec:design-misspecification-toy}

\paragraph{Analytic model.} The simulator and observation model in \cref{sec:imp_of_fine_tune} share the prior $\theta\sim\mathcal{N}(0,I_2)$ and noise standard deviation $\sigma=0.2$, but differ by the additive discrepancy
\[
    h(\xi)=\exp\!\left(-\frac{\|\xi-c\|_2^2}{2w^2}\right),
    \qquad c=(1,-0.5)^\top,\quad w=0.4,\quad \xi\in[-2,2]^2.
\]
Pretraining drew designs uniformly over this square. Each observation panel contains $n=10$ responses generated at $\theta^\star=(0.8,-1.2)^\top$; the localization run used a separate, manually specified design panel. Fine-tuning conditioned on an LSR estimate fitted under the simulator model, not on $\theta^\star$. We distinguish the misspecified region $h(\xi)\geq0.1$ from its complement.

The true posterior is available analytically. Let $X$ have rows $\xi_i^\top$, and let $\mathbf{y}^o$ and $\mathbf{h}$ collect the observations and discrepancies $h(\xi_i)$. Then
\[
    p^\star(\theta\mid\mathbf{y}^o,X)=\mathcal{N}(\mu_\star,\Sigma),
    \qquad \Sigma=(I_2+\sigma^{-2}X^\top X)^{-1},
    \qquad \mu_\star=\sigma^{-2}\Sigma X^\top(\mathbf{y}^o-\mathbf{h}).
\]
The simulator-implied posterior has the same covariance but omits $\mathbf{h}$ from its mean. This gives an exact reference for evaluating posterior correction.

\begin{table}[!b]
\centering
\caption{Fine-tuning on the analytic design-misspecification toy. Posterior errors are measured against the exact true posterior; likelihood RMSE is evaluated in correctly specified and misspecified regions of the design space. The improvement factor is the PT error divided by the FT error, so values above $1\times$ favor fine-tuning and values below $1\times$ indicate degradation.}
\label{tab:design-misspecification-toy}
\small
\begin{tabular}{lccc}
\toprule
\textbf{Metric} & \textbf{PT} & \textbf{FT} & \textbf{Improvement (PT/FT)} \\
\midrule
Posterior $W_2$ $\downarrow$ & $0.09$ & \textbf{$0.07$} & $1.35\times$ \\
Posterior mean error $\downarrow$ & $0.08$ & \textbf{$0.03$} & $2.88\times$ \\
Correct-region likelihood RMSE $\downarrow$ & \textbf{$0.05$} & $0.26$ & $0.18\times$ \\
Misspecified-region likelihood RMSE $\downarrow$ & $0.45$ & \textbf{$0.30$} & $1.49\times$ \\
\bottomrule
\end{tabular}
\end{table}

The toy exposes a trade-off hidden by posterior metrics alone: adaptation can move the posterior toward the truth while degrading a likelihood that was already accurate away from the discrepancy. More aggressive full-network fine-tuning reduces posterior $W_2$ further to $0.03$, but raises correct-region likelihood RMSE to $0.32$.

\paragraph{Design-localized error and adaptation.} \Cref{fig:design-misspecification-adjustment} uses a separate full-network fine-tuning run with joint two-dimensional RBF design embeddings and $\lambda=0.1$. We estimated the PT and FT likelihood means using 16 draws per design at the same fixed LSR reference and evaluated their absolute errors against the known noise-free response at $\theta^\star$. The displayed $e_{\mathrm{PT}}$ and $e_{\mathrm{FT}}$ maps share a color scale, and their annotations average these pointwise errors over all 1,681 design-grid points. This run is distinct from the posterior benchmark above and supplies the $\lambda=0.1$ setting in the regularization sweep below.

MAE decreases from $0.31$ to $0.20$ inside the misspecified region ($h\geq0.1$), but increases from $0.11$ to $0.15$ outside it. Path divergence has a Pearson correlation of $0.75$ with $h$, and its mean inside the misspecified region is $4.73$ times that outside. Thus, concentrated adaptation can coexist with both local correction and deterioration elsewhere; path divergence alone does not determine whether predictions improve. This analysis used one training seed, one dataset of 10 observations, and one generating parameter vector, so it is a controlled diagnostic rather than a multi-seed benchmark.

\paragraph{Regularization trade-off.}
\Cref{fig:toy-lambda-sweep} varies $\lambda$ for the joint-RBF model while holding the PT model, observation panel, LSR reference, and remaining training settings fixed. Strong regularization reduces response adjustment and path divergence, retaining more of the PT model, but also weakens correction near the omitted bump. Increasing $\lambda$ from $1$ to $10$ reduces RMSE outside the misspecified region by about $3\%$, while increasing it inside from $0.25$ to $0.29$. None of the tested FT models improves full-grid RMSE over PT. Here prediction error compares the FT mean at $\theta^{\mathrm{LSR}}$ with the noiseless true response at $\theta^\star$, so it includes reference-parameter error. This trade-off motivates a design-dependent penalty $\lambda(\xi)$ that preserves well-modeled regions while permitting correction where needed; testing that extension remains future work.

\begin{figure}[!b]
\centering
\includegraphics[width=\linewidth]{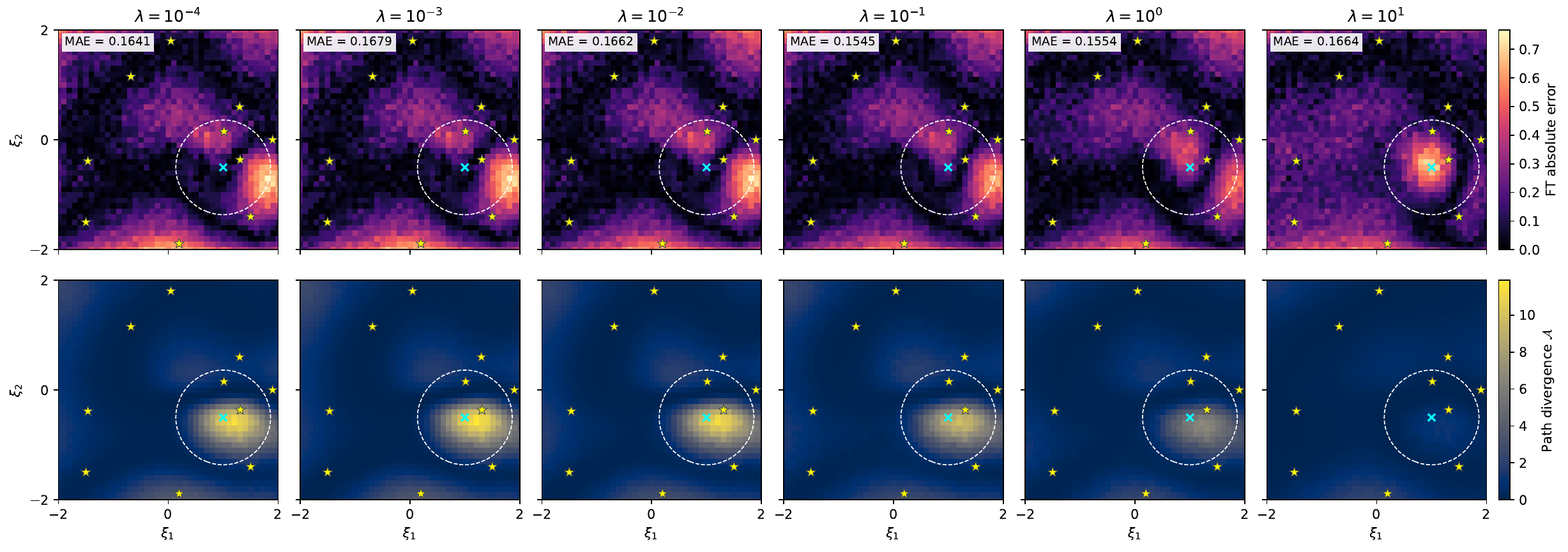}
\caption{Regularization sweep on the joint-RBF misspecification toy. Columns increase $\lambda$ from $10^{-4}$ to $10$. (\emph{Top}) FT absolute prediction error $e_{\mathrm{FT}}(\xi)$, with full-grid MAE annotated. (\emph{Bottom}) Path divergence $\mathcal{A}(\xi)$. Stars mark observed designs; dashed contours enclose $h(\xi)\geq0.1$. Color scales are shared within rows. The sweep used one training seed and 16 Monte Carlo draws per design to estimate response means.}
\label{fig:toy-lambda-sweep}
\end{figure}

\paragraph{Sampling-time autoguidance.} Inspired by autoguidance \citep{karras2024autoguidance}, we tested score blending as a sampling-time alternative to changing the training penalty. Using the frozen PT and FT models from \cref{fig:design-misspecification-adjustment}, we sampled with $s_w^y=s_{\mathrm{PT}}^y+w(s_{\mathrm{FT}}^y-s_{\mathrm{PT}}^y)$ at the same LSR reference, where $w=0$ recovers pretraining and $w=1$ recovers fine-tuning. Here the endpoints differ through misspecification correction, rather than training quality alone. \Cref{fig:toy-autoguidance} shows that intermediate weights can retain correction near the discrepancy while attenuating some errors elsewhere. With 64 response draws per design, 100 reverse steps, and matched sampling noise across weights, $w=0.5$ gives full-grid MAE $0.12$, compared with $0.14$ for PT and $0.15$ for FT. The full sweep also tested $w\in\{1.25,1.5,2\}$, which increase full-grid error beyond FT. The lower row estimates PT-relative path divergence along each weight's own reverse trajectories, using 64 paths per design rather than rescaling the FT map; this discretized estimate excludes the final Tweedie step. The best weight was selected retrospectively in this single-seed comparison, and improvement is not uniform across designs. These results suggest a sampling-time control for over-adaptation. 

\begin{figure}[htbp]
\centering
\includegraphics[width=\linewidth]{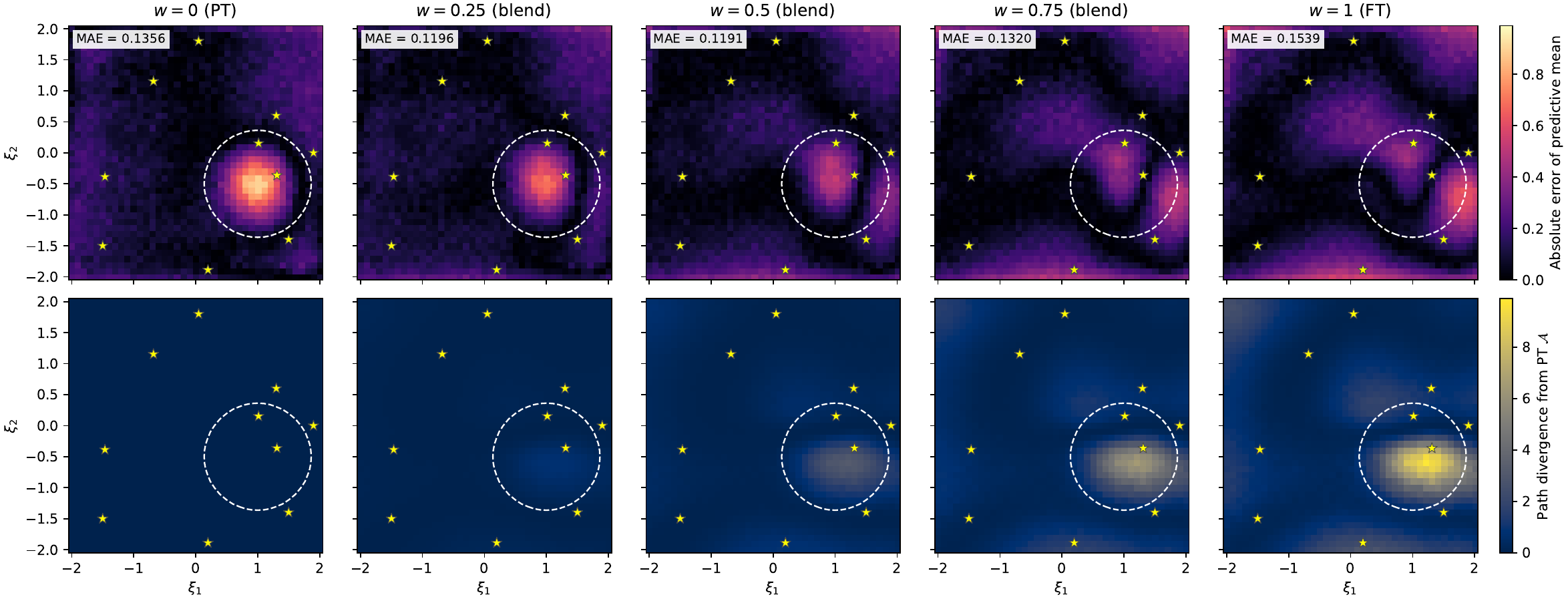}
\caption{PT--FT score blending on the misspecification toy. Columns vary the sampling weight $w$ from PT ($0$) to FT ($1$), without retraining. (\emph{Top}) Absolute error of the predictive mean against the noise-free response, with full-grid MAE annotated. (\emph{Bottom}) Estimated path divergence from PT. Stars mark observed designs; color scales are shared within rows. Intermediate weights reduce average error while limiting departure from PT.}
\label{fig:toy-autoguidance}
\end{figure}

\begin{table}[H]
\centering
\caption{Interpretation of design-localized fine-tuning. Path divergence $\mathcal{A}(\theta,\xi)$ measures departure from the PT model at fixed $\theta$, while predictive error change is evaluated against observations at the same design $\xi$.}
\label{tab:fine-tuning-diagnostic}
\small
\begin{tabular}{@{}L{0.18\textwidth}L{0.36\textwidth}L{0.36\textwidth}@{}}
\toprule
$\mathcal{A}(\theta,\xi)$ &
\textbf{Error improves} &
\textbf{Error does not improve} \\
\midrule
Small & Adequate simulator or mild correction & Under-correction \\
Large & Corrected misspecification & Over-correction or forgetting \\
\bottomrule
\end{tabular}
\end{table}

\Cref{tab:fine-tuning-diagnostic} relates the magnitude of adaptation to its predictive benefit at each design. Large path divergence with reduced error is consistent with misspecification correction, whereas large divergence without improvement suggests unnecessary adaptation or forgetting. Small divergence with persistent error indicates under-correction, while small divergence where predictions are already accurate suggests that little correction is needed.

\subsection{BMP misspecification evaluation}
\label{sec:bmp-simformer-misspecification}

\setcounter{topnumber}{1}
\suppressfloats[t]
\begin{figure}[t]
\centering
\includegraphics[width=0.49\linewidth]{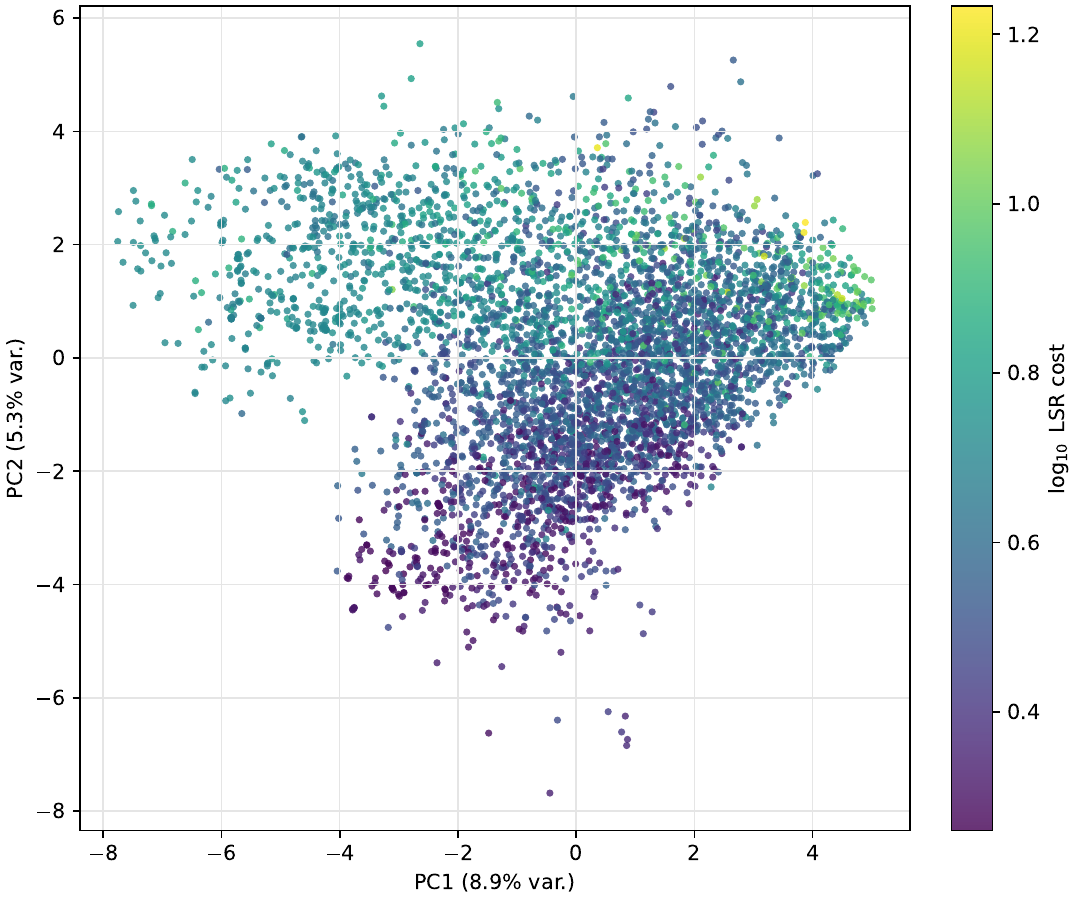}\hfill
\includegraphics[width=0.49\linewidth]{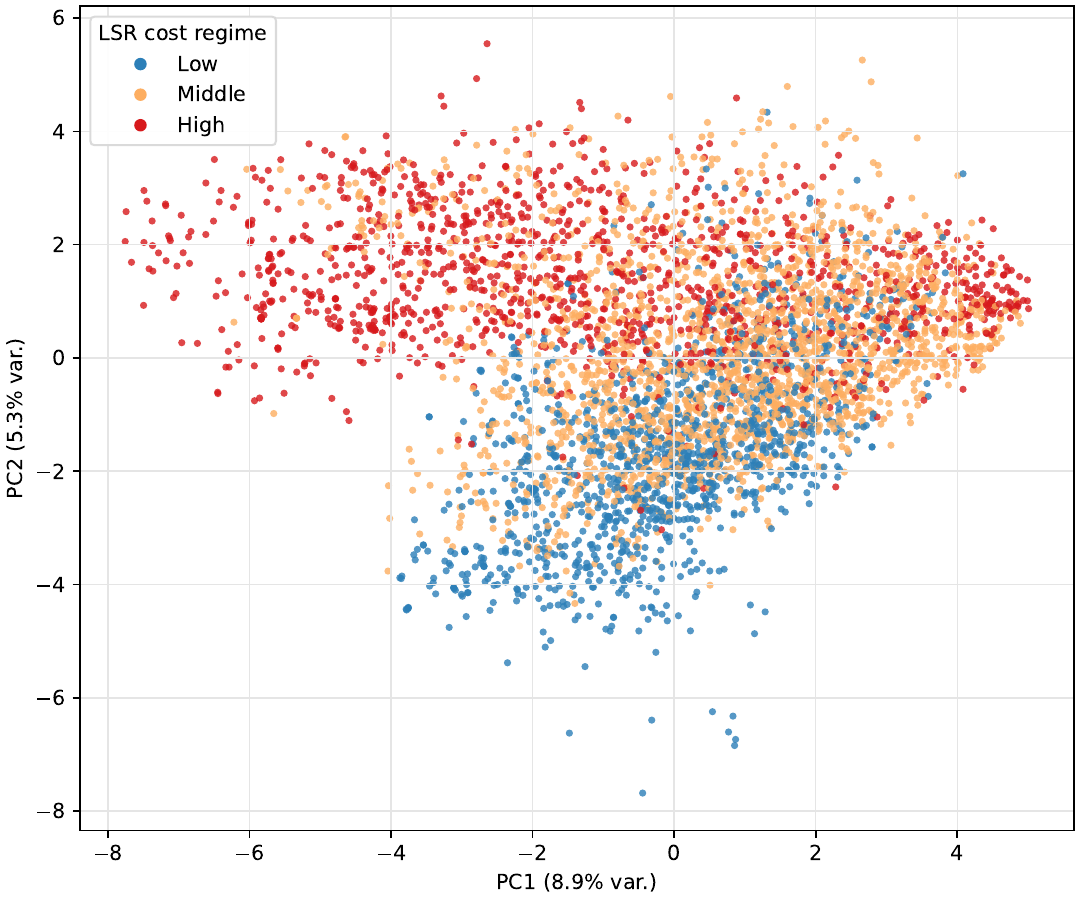}
\captionsetup{skip=4pt}
\caption{LSR solution structure and cost groups across $4{,}816$ retained fits. (\emph{Left}) PCA coordinates colored by $\log_{10}$ total LSR cost. (\emph{Right}) The same projection colored by low, middle, and high cost groups, containing $1{,}296$, $2{,}245$, and $1{,}275$ fits. PC1 and PC2 explain $8.9\%$ and $5.3\%$ of the standardized feature variance. Grouping also uses fitting cost, so overlap in this two-dimensional projection is expected.}
\label{fig:lsr-pca-cost-groups}
\vspace{-14pt}
\end{figure}

\paragraph{LSR cost groups.} The LSR ensemble was generated by fitting the expert simulator to the same BMP observations across four cell lines in $6{,}000$ bounded least-squares runs with different random initializations. Each run minimized squared prediction error and yielded a candidate parameter set, providing alternative fits to the same data. We used the $4{,}816$ fits retained in the original filtered bank provided by \cite{klumpe2022context}. Each fit contains 60 biophysical parameters and five receptor multiplicative factors, giving 65 features. We transformed these features to model-normalized coordinates, z-scored each feature across fits, and computed PCA (\cref{fig:lsr-pca-cost-groups}). To group solutions by parameter structure and fit quality, we applied $k$-means with nine clusters to standardized PC1, PC2, and $\log_{10}C$, where $C$ is the total LSR fitting cost, multiplying the standardized cost coordinate by eight. We ordered clusters by their mean log-cost and merged clusters 1--3, 4--6, and 7--9 into low (L), middle (M), and high (H) cost groups, respectively. These labels summarize global fit quality, not performance at every design. For the local group assignments in \cref{fig:rims_comparison,fig:bmp-misspecification-regimes}, we computed MSE across member fits at each design and selected the group with the lowest value. Thus, a globally higher-cost group can still fit a particular experimental condition best, as explicitly shown in \cref{fig:bmp-misspecification-regimes}A.

\paragraph{Ligand-competition and titration conditions.}
\label{sec:comparing_bmps}
The three response curves in \cref{fig:rims_comparison} probe competition for receptors and dose-dependent signaling. The ligand-competition axis transitions from high BMP4 and low BMP10 to high BMP10 and low BMP4. The NMuMG response changes along this axis, whereas the BMPR2-knockdown line remains close to zero. The BMP4 titration in NMuMG provides a dose--response comparison that also includes a decrease in observed signal at the highest concentration. We held the mechanistic parameters at $\theta^{\mathrm{LSR}}$ and sampled the FT likelihood, isolating response correction from changes in inferred parameters. The FT model improves over the simulator at this reference throughout the displayed BMPR2-knockdown series and at lower BMP4 doses, while the NMuMG competition series and higher-dose titration retain regions where LSR predicts better. These conditional-likelihood comparisons are distinct from the aggregate HBDS posterior-predictive benchmark in \cref{tab:rmse_three_methods_transposed}.

\paragraph{Design-level adaptation.}

We relate the design-conditional path divergence in \cref{eq:score-control-energy} to the LSR cost groups defined above. Panel A of \cref{fig:bmp-misspecification-regimes} shows which group has the lowest MSE at each experimental condition. Although the low-cost group dominates, other groups fit better at particular ligand conditions and cell lines. These preferences describe relative fit among groups of parameter estimates, rather than establishing simulator misspecification on their own.

Comparing panels A and B reveals a qualitative correspondence between some departures from the dominant LSR group and elevated path divergence $\mathcal{A}(\theta,\xi)$. Panels C and D then distinguish adaptation from predictive improvement by comparing FT responses with the selected LSR anchor and the PT bootstrap median, respectively. Fine-tuning improves over the PT model across many conditions, but its comparison with LSR remains mixed. Nevertheless, improvements over LSR occur both at conditions favoring alternative parameter groups and at some conditions where the dominant group remains preferred. The green-boxed examples connect these design-level diagnostics to the response curves in \cref{fig:rims_comparison}.

\begin{figure}[t]
\centering
\includegraphics[width=\linewidth]{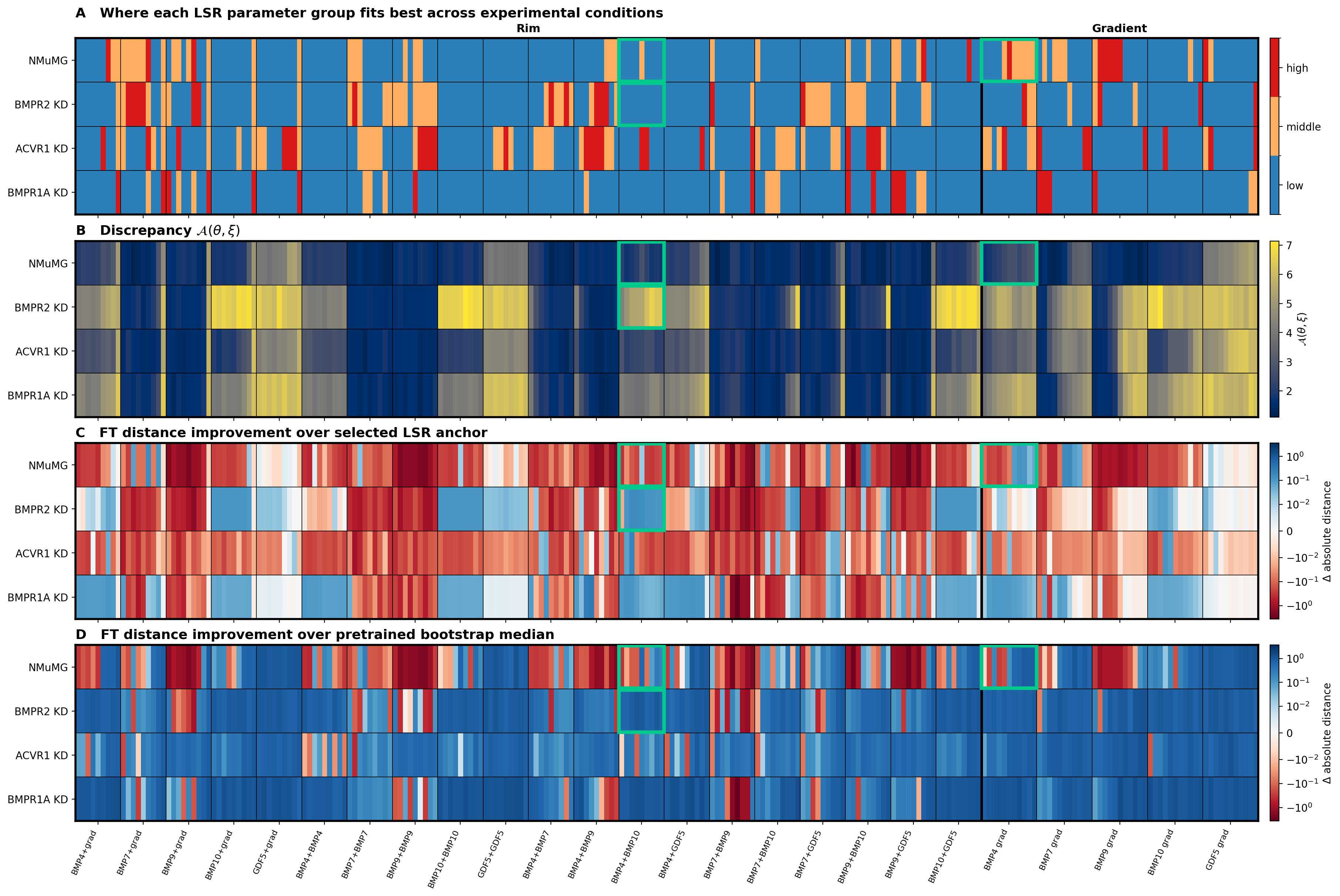}
\caption{Design-localized adaptation in the BMP model across four cell lines (rows) and experimental conditions (columns). Thin black lines separate titration series; the thick black line separates ligand-competition (``Rim'') from gradient series (``Gradient''). Green boxes mark the conditions illustrated in \cref{fig:rims_comparison}. (\emph{A}) Best-fitting LSR parameter group; colors identify groups, not error magnitudes. (\emph{B}) Path divergence $\mathcal{A}(\theta,\xi)$, with yellow indicating greater adaptation. (\emph{C}) Fine-tuning distance improvement relative to the selected LSR anchor. (\emph{D}) Fine-tuning distance improvement relative to the PT bootstrap median. In C and D, blue indicates reduced distance to observations and red indicates increased distance relative to the comparator.}
\label{fig:bmp-misspecification-regimes}
\end{figure}

Together, these results provide a proof of concept for design-localized likelihood correction where fine-tuning can improve predictions beyond the selected mechanistic fit on particular conditions, while path divergence identifies where the likelihood changes. The remaining regions where LSR performs better also show that the current fine-tuning procedure does not fully exploit this potential, or, can be improved. This motivates the design-adaptive curricula and divergence-aware regularization discussed in \cref{sec:soc-connection,sec:fisher-rao-analysis}, targeting residual error while preserving successful corrections. Establishing whether these extensions yield broader gains requires further evaluation, which we leave to future work.

\paragraph{Posterior transfer without collapse.} Although likelihood fine-tuning is anchored at a fixed reference $\theta^{\mathrm{LSR}}$, the resulting posterior does not collapse onto this point estimate. \Cref{fig:bmp-posterior-marginals} shows that FT posterior marginals generally shift toward the reference while retaining substantial spread, indicating that the likelihood-side correction transfers to posterior queries without reducing inference to the fine-tuning anchor.

\begin{figure}[p]
\centering
\includegraphics[width=\textwidth]{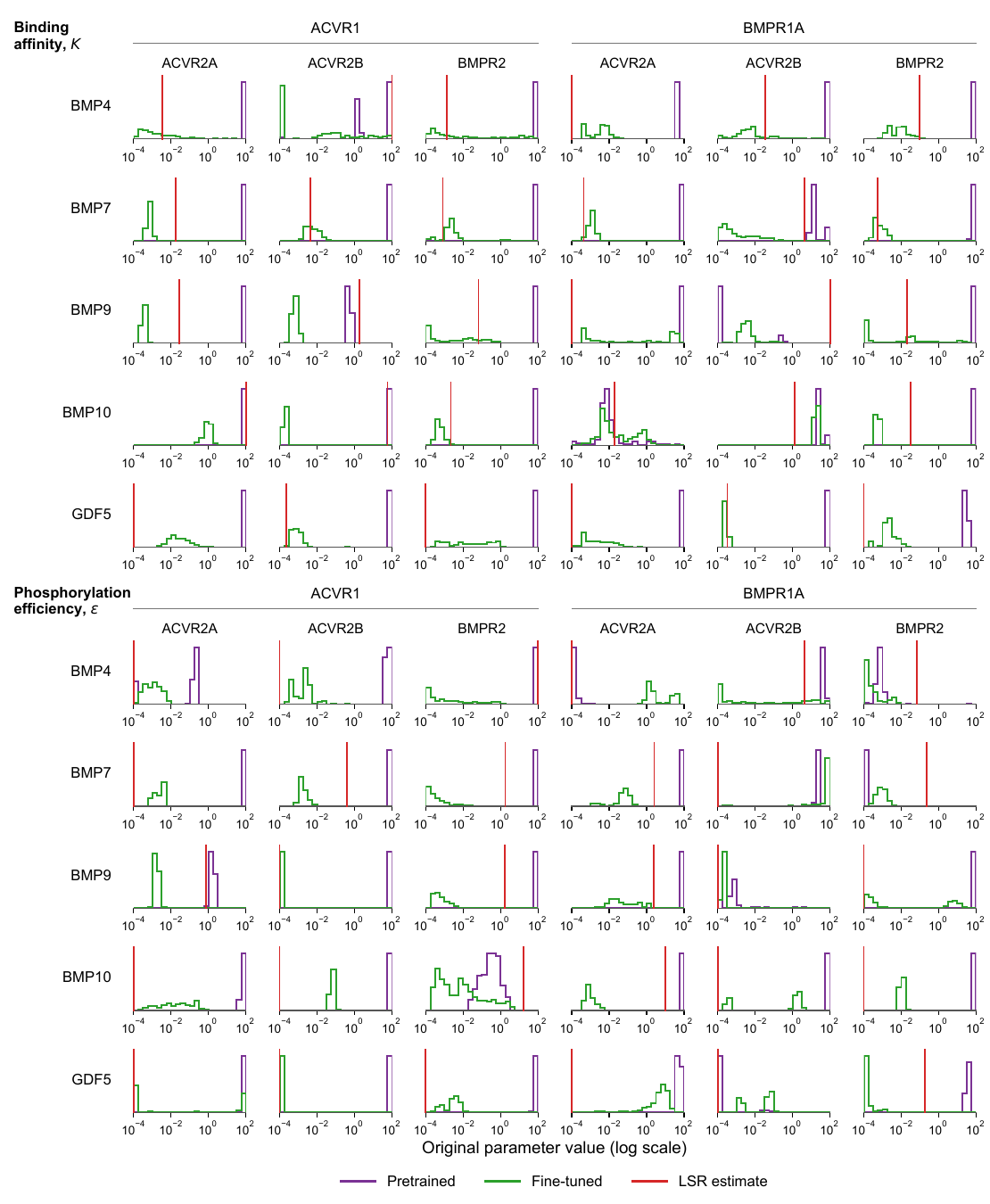}
\caption{BMP posterior alignment with the LSR reference. Purple and green show PT and FT marginals; red lines mark the fitted LSR reference. The upper and lower blocks show 30 binding affinities $K$ and 30 phosphorylation efficiencies $\varepsilon$, with ligand rows and receptor-complex columns. Type-I receptors label the column groups above; type-II receptors label individual columns beneath them, identifying each receptor pair. FT marginals generally move closer to the reference while retaining spread rather than collapsing to it.}
\label{fig:bmp-posterior-marginals}
\end{figure}

\clearpage
\section{AI use statement}
\label{sec:ai-use}

AI tools assisted with grammar, organization, TikZ figures, and experiment management. The authors take responsibility for the final content.


\end{document}